\documentclass{article}

\usepackage{PRIMEarxiv}
\usepackage{bm}
\usepackage{algorithm}
\usepackage{amsmath}
\usepackage{amsfonts}
\usepackage{multirow}
\usepackage{algpseudocode}
\usepackage{amsthm}
\usepackage{array}
\usepackage{epstopdf}
\usepackage{subfig}

\usepackage[utf8]{inputenc} 
\usepackage[T1]{fontenc}    
\usepackage{hyperref}       
\usepackage{url}            
\usepackage{booktabs}       
\usepackage{amsfonts}       
\usepackage{nicefrac}       
\usepackage{microtype}      
\usepackage{lipsum}
\usepackage{pifont}
\usepackage{fancyhdr}       
\usepackage{graphicx}       
\graphicspath{{media/}}     

\newcommand{\x}{{\bf x}}

\usepackage{stackengine}
\newcommand\xrowht[2][0]{\addstackgap[.5\dimexpr#2\relax]{\vphantom{#1}}}

\newtheorem{theorem}{Theorem}

\newtheorem{lemma}{Lemma}

\title{Late Fusion Multi-view Clustering via Global and Local Alignment Maximization
}

\author{
  Siwei Wang, Xinwang Liu, En Zhu \\
  School of Computer\\
  National University of Defense Technology\\
  Changsha, China, 410073 \\
  \texttt{\{wangsiwei13,\,xinwangliu,\,enzhu\}@nudt.edu.cn} \\
}

\begin{document}
\maketitle

\begin{abstract}
Multi-view clustering (MVC) optimally integrates complementary information from different views to improve clustering performance. Although demonstrating promising performance in various applications, most of existing approaches directly fuse multiple pre-specified similarities to learn an optimal similarity matrix for clustering, which could cause over-complicated optimization and intensive computational cost. In this paper, we propose late fusion MVC via alignment maximization to address these issues. To do so, we first reveal the theoretical connection of existing $k$-means clustering and the alignment between base partitions and the consensus one. Based on this observation, we propose a simple but effective multi-view algorithm termed {Late Fusion Multi-view Clustering via Global Alignment Maximization (LF-MVC-GAM)}. It optimally fuses multiple source information in partition level from each individual view, and maximally aligns the consensus partition with these weighted base ones. Such an alignment is beneficial to integrate partition level information and significantly reduce the computational complexity by sufficiently simplifying the optimization procedure. We then design another variant, {Late Fusion Multi-view Clustering via Local Alignment Maximization (LF-MVC-LAM)} to further improve the clustering performance by preserving the local intrinsic structure among multiple partition spaces. After that, we develop two three-step iterative algorithms to solve the resultant optimization problems with theoretically guaranteed convergence. Further, we provide the generalization error bound analysis of the proposed algorithms. Extensive experiments on eighteen multi-view benchmark datasets demonstrate the effectiveness and efficiency of the proposed LF-MVC-GAM and LF-MVC-LAM, ranging from small to large-scale data items. The codes of the proposed algorithms are publicly available at \url{https://github.com/wangsiwei2010/latefusionalignment}.
\end{abstract}

\keywords{Multiple kernel clustering \and multiple view clustering}

\section{Introduction}
Multi-view clustering (MVC) optimally integrates multiple view information to categorize data with similar structures or patterns into the same cluster \cite{zhan2018multiview,nie2017auto,liang2020multi,liu2020optimal,peng2018structured,gonen2008localized,chaudhuri2009multi,deng2020multi,han2020multi,huang2012multiple,liu2013efficient,gonen2014localized,Liu2016Multiple,Li2016Multiple,li2020multi,liu2020multi,liu2017optimal,liu2018late,zhang2018binary,wang2015multi,zhang2020consensus,wu2014k}. Many multi-view clustering algorithms have been proposed in literature, which can roughly be summarized with three categories: i) Co-training style; ii) Subspace clustering; iii) Multiple kernel clustering (MKC). The co-training approach for MVC iteratively learns multiple clustering results that can provide predicted clustering indices for the unlabeled data for other views \cite{blum1998combining,kumar2011co,yang2014information,tao2017ensemble,yao2017revisiting,zhang2017robust}. By this way, besides extracting the specific cluster information from the corresponding view, the clustering results are forced to be consistent across views. Subspace clustering approach for MVC aims to optimize a unified subspace from multi-view data representation \cite{liu2013multi,cao2015diversity,cai2013multi,zhao2017multi,xu2015multi,wang2017exclusivity,zhu2018one,wang2019gmc,zhang2016joint,sun2021projective}. \cite{DBLP:conf/ijcai/NieLL17} proposes a self-weighted strategy to adaptively update the view coefficients and output the ideal low-rank graph similarity matrix for spectral clustering (SWMC). Moreover, \cite{2021Fast} proposes a unified framework to jointly conduct prototype graph fusion and directly output the clustering labels. By following the multiple kernel learning framework, multiple kernel clustering (MKC) approach learns a group of optimal kernel coefficient to improve clustering performance, which has been intensively studied \cite{yu2012optimized,liu2013efficient,gonen2014localized,Liu2016Multiple,Li2016Multiple,
liu2017optimal,liu2019multiple,liu2019robust,yin2018subspace,xu2017re,yao2017revisiting}.The work in \cite{chen2007nonlinear} proposes a three-step alternate algorithm to jointly optimize clustering, kernel coefficients and dimension reduction. A multiple kernel $k$-means clustering algorithm with a matrix-induced regularization term to reduce the redundancy of the selected kernels is presented in \cite{Liu2016Multiple}. In addition, local kernel alignment criterion has been applied to multiple kernel clustering to enhance the clustering performance in \cite{Li2016Multiple}. In \cite{liu2017optimal}, a multiple-kernel clustering algorithm is developed to allow the optimal kernel to reside in the neighborhood of the linearly combined kernels, which is considered to enlarge the representation ability of the learned optimal kernel. Our work in this paper belongs to the third category.

Although the aforementioned multiple kernel clustering algorithms have improved clustering performances from different aspects, we observe that most of them suffer from the following drawbacks: i) The majority of existing multiple kernel clustering methods learn a shared optimal kernel matrix and then apply kernel $k$-means to obtain the final clustering result. In contrast, few of them fuse partition level information for further improving clustering quality. ii) The intensive computational complexity prevents them from being applied into large-scale clustering tasks, i.e., usually $\mathcal{O}(mn^3)$ per iteration, where $n$ and $m$ are the number of samples and views respectively. iii) According to \cite{liu2018late}, the resultant optimization processes of existing multiple kernel clustering algorithms are usually over-complicated, which could increase the risk of being trapped into over-fitting situation, leading to unsatisfying clustering performance.

In this paper, we propose to fuse multi-view information from partition level by maximizing the alignment between base partitions and the consensus one. Before doing so, we first theoretically show that this criterion is conceptually equivalent to minimizing the clustering loss of existing $k$-means algorithm. Based on this observation, we propose a novel algorithm, termed \emph{Late Fusion Multi-view Clustering via Global Alignment Maximization} (LF-MVC-GAM), to significantly improve the computational complexity as well as clustering performances of existing multiple kernel clustering. In specific, it maximally aligns the consensus partition and the weighted base ones with optimal permutations in a global way. We further improve LF-MVC-GAM and develop another variant, termed \emph{Late Fusion Multi-view Clustering via Local Alignment Maximization} (LF-MVC-LAM), to preserve the local partition structure among data and utilize more reliable base partitions to guide the learning of consensus partition. After that, two efficient algorithms are developed to solve the resultant optimization problems and their computational complexity and convergence are theoretically analyzed. Moreover, the generalization error bound is provided to theoretically justify the performance of the proposed algorithms. In addition, extensive experiments on twelve multiple-view benchmark datasets are conducted to evaluate the effectiveness and efficiency of our proposed method, including the clustering performance, the running time, the evolution of the learned consensus partition and the objective value with iterations. As demonstrated, the proposed algorithms enjoy superior clustering performance with significant reduction in computational cost, in comparison with several state-of-the-art multi-view kernel-based clustering methods.

This paper is substantially extended from our conference version \cite{wang2019mvclfa}. Its significant improvement over the conference version can be summarized as follows: 1) We design a new algorithm variant termed LF-MVC-LAM, by incorporating multiple views information in a local late fusion manner, and develop an iterative algorithm to efficiently solve the resultant optimization problem. By well preserving the local cluster structures and alleviating unreliable similarity evaluation in partition spaces, the newly proposed LF-MVC-LAM significantly outperforms LF-MVC-GAM proposed in the previous paper \cite{wang2019mvclfa}.
2) We theoretically study the generalization bound of the proposed LF-MVC-GAM and LF-MVC-LAM on test data.
3) Besides more detailed discussion and extension, we also conduct more comprehensive experiments to validate the effectiveness and efficiency of the proposed algorithms.

\section{Related Work}\label{Related Work}

In this section, we introduce existing work most related to our study in this paper, including $\emph{k}$-means algorithm, multi-view clustering and multiple-kernel clustering methods. Table \ref{notations} lists main notations used throughout the paper.

\begin{table}[!t]
\begin{center}
{\caption{{Main notations used throughout the paper.}}\label{notations}
\resizebox{0.99\textwidth}{!}{
\begin{tabular}{c||c}
\toprule
Notation        & Meaning   \\
\midrule
\hline
$n$         & The number of samples     \\
\hline
$m$             &  The number of views   \\
\hline
$k$       &  The number of clusters   \\
\hline
$\boldsymbol{\beta} \in \mathbb{R}^{m \times 1}$       &  The view coefficient    \\
\hline
\xrowht{10pt}
$\mathbf{X}_i^{(p)}$           & The $p$-th view of the $i$-th sample   \\
\hline
\xrowht{10pt}
$\mathbf{F}\in \mathbb{R}^{n \times k}$           & The consensus clustering indicator matrix    \\
\hline
\xrowht{10pt}
$\mathbf{K}_p \in \mathbb{R}^{n \times n}$            &  The kernel matrix for the $p$-th view    \\
\hline
\xrowht{10pt}
$\mathbf{H}_p \in \mathbb{R}^{n \times k}$            & The base partition matrix for the $p$-th view    \\
\hline
\xrowht{10pt}
$\mathbf{W}_p \in \mathbb{R}^{k \times k}$            & The permutation matrix for the $p$-th base partition    \\
\hline
\xrowht{10pt}
$\mathbf{B}\in \mathbb{R}^{n \times k}$    &  The combined partition representation   \\
\hline
\xrowht{10pt}
${\mathbf{A}^{(i)}_{p}} \in {\left\{0,1\right\}}^{n\times n}$  &  The indicator matrix for the neighbors of $i$-th sample   \\
\hline
\xrowht{10pt}
$\mathbf{K}_{\boldsymbol{\beta}}$  &  The combined kernel matrix $\mathbf{K}_{\boldsymbol{\beta}} = \sum\nolimits_{p=1}^m {\boldsymbol{\beta}}_p^2 \mathbf{K}_p$   \\
\hline
\xrowht{10pt}
${\{\kappa_{p}(\cdot,\cdot)\}}_p$  &  The kernel function for the $p$-th view  \\
\hline
\xrowht{10pt}
$\mathbf{M}$  &  The partition matrix obtained by performing kernel $k$-means with average kernel  \\
\bottomrule
\end{tabular}}
}
\end{center}
\end{table}

\subsection{$\emph{k}$-means clustering for single view}

As one of the classical and widely-used algorithms, $\mathit{k}$-means algorithm provides an intuitive and effective way to perform clustering. Defining $\{\mathbf{x}_{i}\}_{i=1}^{n}\subseteq\mathcal{X}$ as a collection of $n$ samples in $k$ clusters and the data matrix $\mathbf{X} \in \mathbb{R}^{\mathit{n}\times \mathit{d}}$ whose rows correspond to data points with $d$ features.

We represent the $k$-means clustering of $\mathbf{X}$ by its cluster indicator matrix $\mathbf{F} \in  \mathbb{R}^{\mathit{n}\times \mathit{k}}$. Each column $j = 1, \cdots, k$ of $\mathbf{F}$ corresponds to a cluster. Each row $i = 1, \cdots , n$ indicates the cluster membership of the point $\mathbf{x}_{i}$. As a result, $F_{ij} = \frac {1} {\sqrt{|\mathbf{C}_j}|}$ if and only if data point $\mathbf{x}_i$ is in the $j$-th cluster $\mathbf{C}_j$. The traditional $k$-means problem can be equivalently rewritten as \cite{boutsidis2015randomized}, 
	\begin{equation}
	\label{2}
\begin{split}
&\; \max\limits_{\mathbf{F}}\; \mathrm{Tr}(\mathbf{F}^{\top}\mathbf{X} \mathbf{X}^{\top}\mathbf{F}),\\
s.t.   \mathbf{F} \in  \mathbb{R}^{\mathit{n}\times \mathit{k}}, &\;	
	\mathbf{F}_{ij}=\left\{
	\begin{array}{lr}
		 \frac {1} {\sqrt{|\mathbf{C}_j}|},\text{if } x_i \text{ is in the $j$-th cluster}.&\\
		0,\text{otherwise}.&
	\end{array}
	\right.
\end{split}
\end{equation}
	where $\mathbf{X}\in\mathbb{R}^{n\times d}$ denotes the data matrix and $|\mathbf{C}_j|$ denotes the number of samples in $j$-th cluster. According to the definition of $\mathbf{F}$, it naturally induces an orthogonal constraint on ${\mathbf{F}}$ such that ${\mathbf{F}}^{\top}\mathbf{F} = \mathbf{I}_{k}$.

\begin{figure}[t]
\centering
\subfloat{\includegraphics[width=0.99\textwidth]{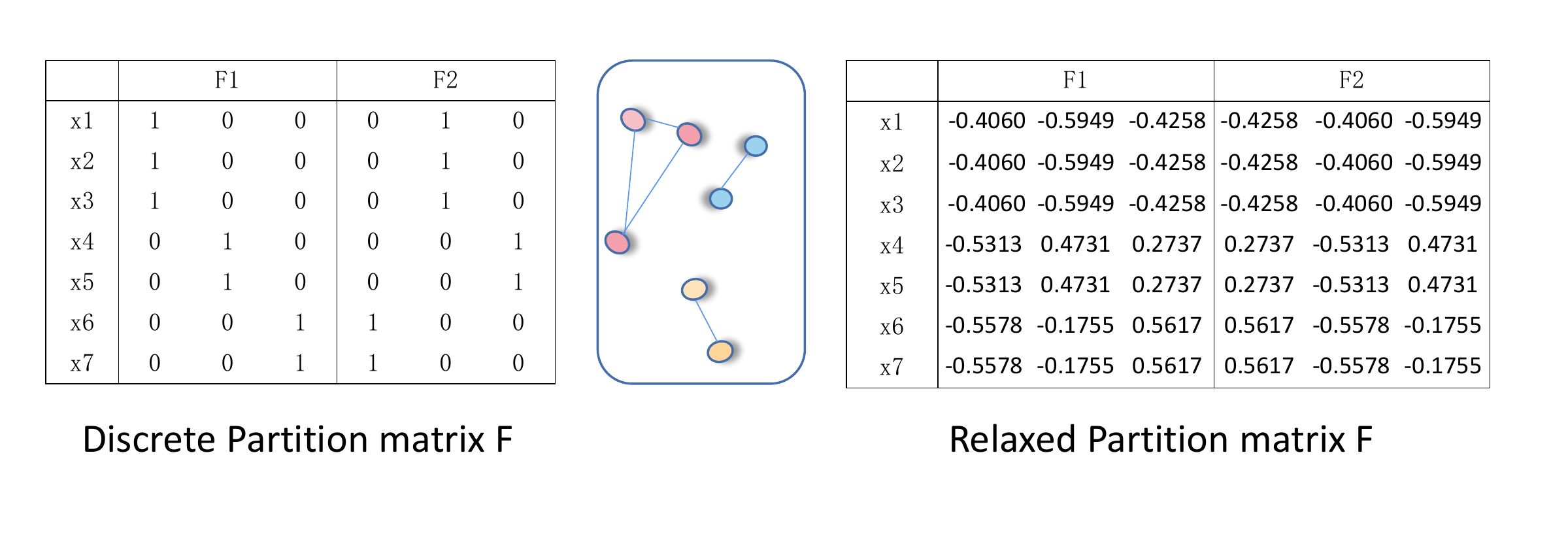}}
\caption{The illustration of partition matrix $\mathbf{F}$. An example of hard partition matrix $\mathbf{F}$ is shown on the left and the right is the relaxed orthogonal version. (left)In each line of $\mathbf{F}$, the index of non-zero element (others 0) is output as the label for the respective sample. (Right) $\mathbf{F}$ is relaxed into orthogonal constraints and as the input for Floyd algorithm \cite{DBLP:journals/tit/Lloyd82} to obtain the clustering labels. Both of them represent the cluster label of 7 samples belonging to 3 clusters. }
\label{partition}
\end{figure}

Directly solve the hard partition assignment problem is proven to be NP-hard. By relaxing $\mathbf{F}$ into orthogonal constraints, the clustering loss is transformed as follows,
\begin{equation}\label{LFA 1}
\begin{split}
\min\limits_{\mathbf{F}}\;&\;\mathrm{Tr}(\mathbf{X} \mathbf{X}^{\top}) - \mathrm{Tr}(\mathbf{F}^{\top}\mathbf{X} \mathbf{X}^{\top}\mathbf{F}),\\
s.t. &\;  \mathbf{F} \in \mathbb{R}^{\mathit{n}\times \mathit{k}},\, {\mathbf{F}}^{\top}\mathbf{F} = \mathbf{I}_{k},
\end{split}
\end{equation}
where $\mathbf{F}$ is the partition indicator matrix. After obtaining the relaxed representation $\mathbf{F}$, $\mathbf{F}$ is taken as the input to the Floyd algorithm \cite{DBLP:journals/tit/Lloyd82} to get final clustering labels: $\Large{\textcircled{\small{1}}}$initial $k$ centers and update assignment;$\Large{\textcircled{\small{2}}}$update new cluster centers by assignment matrix. The two processes are repeated until converged and the discrete clustering labels are output. Since the first term of Eq. (\ref{LFA 1}) is constant regrading of optimizing $\mathbf{F}$, we can obtain the optimal $\mathbf{F}$ for Eq. (\ref{LFA 1}) by taking the $k$ eigenvectors that correspond to the $k$ largest eigenvalues of $\mathbf{X}\mathbf{X}^{\top}$. 

Eq. (\ref{LFA 1}) could not be directly applied into multiple view data since $\mathbf{X}$ in Eq. (\ref{LFA 1}) represents the single-view data. However, it is not unusual in many real-world applications that data are collected from diverse sources or represented by heterogeneous features from different views. To overcome this limitation, many novel algorithms have been proposed to handle multi-view clustering, as introduced in the following section.

\subsection{Multi-view clustering}

Multi-view Clustering methods perform effectively on multiple source data by utilizing the diversity and
complementarity of different views. Existing literature can be categorized into three ways: co-training, subspace clustering and multi-kernel clustering. The simplest approach is to concatenate multiple source features directly and conduct traditional clustering in the whole data matrix . Most of the existing methods for multi-view subspace clustering integrate multi-view information in similarity or representation level by merging multiple graphs or representation matrices into a shared one \cite{li2015large,wang2015robust,wang2018multiview}. For example, \cite{Guo2013} learns a shared sparse subspace representation by performing matrix factorization. Different from obtaining a shared representation or graph directly, \cite{xia2014robust} and \cite{cao2015diversity} induce HSIC criterion and Markov chain to learn distinct subspace representations and then add them together directly or adaptively to get unified representation. Our work in this paper belongs to the third category, namely, multiple kernel clustering \cite{chen2007nonlinear,yu2012optimized,liu2013efficient,gonen2014localized,Liu2016Multiple,Li2016Multiple,liu2017optimal,liu2019multiple,huang2016ensemble,liu2020efficient}. In the following section, we will introduce multi-kernel $k$-means (MKKM) to handle multi-view clustering.

\subsection{Multi-kernel $k$-means (MKKM)}\label{MKKM}

Let $\phi_{p}: \mathcal{X} \rightarrow \mathcal{H}_{p}$ be the $p$-th feature mapping that maps $\mathbf{x}$ onto a reproducing kernel Hilbert space $\mathcal{H}_{p}~(1\leq p\leq m)$.
In the multiple kernel setting, each sample is represented as $\phi_{\boldsymbol{\beta}}(\mathbf{x} )=[{\boldsymbol{\beta}}_1\phi_{1}(\mathbf{x})^{\top},\cdots,{\boldsymbol{\beta}}_m\phi_{m}(\mathbf{x})^{\top}]^{\top}$, where $\boldsymbol{\beta}=[{\boldsymbol{\beta}}_1,\cdots,{\boldsymbol{\beta}}_m]^{\top}$ consists of the coefficients of the $m$ base kernel functions $\{\kappa_{p}(\cdot,\cdot)\}_{p=1}^{m}$. These coefficients will be optimized during learning procedure. Table \ref{kernel functions} lists several commonly-used kernel functions.

\begin{table}[htbp]
\centering
\caption{Commonly-used kernel functions for calculating $\mathbf{x}_{i}^{(p)}$ and $\mathbf{x}_{j}^{(p)}$}\label{kernel functions}
\resizebox{0.999\textwidth}{!}{
\begin{tabular}{c|c|c}
\toprule
Name & Expression&Parameter \tabularnewline
\midrule[1pt]
Linear kernel  &$\kappa \left (  {\mathbf{x}}_{i}^{(p)}, \x_{j}^{(p)}\right )={\x_{i}^{(p)}}^{\top}\x_{j}^{(p)}$ 
\\ \hline
Polynomial kernel & $\kappa \left ( \x_{i}^{(p)},\x_{j}^{(p)}\right )=\left ( {\x_{i}^{(p)}}^{\top}\x_{j}^{(p)} \right )^{d}$& $d\geq 1$ is the degree of the polynomial \\ \hline
Gaussian kernel & $\kappa \left ( \x_{i}^{(p)},\x_{j}^{(p)}\right )=\textup{exp}\left ( -\frac{\left \| \x_{i}^{(p)}-\x_{j}^{(p)} \right \|^{2}}{2\sigma ^{2}} \right )$&$\sigma > 0$ is the bandwidth of the Gaussian kernel \\ \hline
Laplace kernel & $\kappa \left ( \x_{i}^{(p)},\x_{j}^{(p)}\right )=\textup{exp}\left ( -\frac{\left \| \x_{i}^{(p)}-\x_{j}^{(p)} \right \|}{\sigma} \right )$&$\sigma > 0$ \\ \hline \xrowht{20pt}
Sigmoid kernel & $\kappa \left ( \x_{i}^{(p)},\x_{j}^{(p)}\right )=\textup{tanh}\left ( \gamma {\x_{i}^{(p)}}^{\top}\x_{j}^{(p)}+\theta \right ) $&tanh is the hyperbolic tangent function,$\gamma > 0$,$\theta < 0$\\
\hline
\end{tabular}}
\end{table}

Based on the definition of $\phi_{\boldsymbol{\beta}}(\mathbf{x} )$, a kernel function can be expressed as
\begin{equation}\label{eq:Kmatrix}
{\kappa_{\boldsymbol{\beta}}(\mathbf{x}_{i},\mathbf{x}_{j}) = \phi_{\boldsymbol{\beta}}(\mathbf{x}_{i})^{\top}\phi_{\boldsymbol{\beta}}(\mathbf{x}_{j}) = \sum\nolimits_{p=1}^{m}{\boldsymbol{\beta}}_p^2\kappa_{p}(\mathbf{x}_{i},\mathbf{x}_{j}).}
\end{equation}
A kernel matrix $\mathbf{K}_{\boldsymbol{\beta}}$ is then calculated by applying the kernel function $\kappa_{\boldsymbol{\beta}}(\cdot,\cdot)$ into $\{\mathbf{x}_{i}\}_{i=1}^{n}$.
Based on the kernel matrix $\mathbf{K}_{\boldsymbol{\beta}}$, the objective of MKKM can be written as
\begin{equation}\label{eq:MKKM:v1}
{\begin{split}
\min\limits_{\;\mathbf{H},\boldsymbol{\beta}}\;&\;\mathrm{Tr}(\mathbf{K}_{\boldsymbol{\beta}}
(\mathbf{I}_{n}-\mathbf{H}\mathbf{H}^{\top}))\\
s.t.\;&\;\mathbf{H}\in\mathbb{R}^{n\times k},\;\mathbf{H}^{\top}\mathbf{H}= \mathbf{I}_{k},\\
\;&\;\boldsymbol{\beta}^{\top}\mathbf{1}_{m}=1,\;{\boldsymbol{\beta}}_{p}\geq 0,\;\forall p,
\end{split}}
\end{equation}
where $\mathbf{I}_{k}$ is an identity matrix with size $k\times k$.

The optimization problem in Eq. (\ref{eq:MKKM:v1}) can be solved by alternately updating $\mathbf{H}$ and $\boldsymbol{\beta}$:

i) \textbf{Optimizing $\mathbf{H}$ given $\boldsymbol{\beta}$}. With the kernel coefficients $\boldsymbol{\beta}$ fixed, $\mathbf{H}$ can be obtained by solving a kernel $k$-means clustering optimization problem shown in Eq. (\ref{eq:KkM:v3}),
\begin{equation}\label{eq:KkM:v3}
{\begin{split}
\max\limits_{\;\mathbf{H}}\;&\;\mathrm{Tr}(\mathbf{H}^{\top}\mathbf{K}_{\boldsymbol{\beta}}\mathbf{H}),
s.t.\;\;\mathbf{H}\in\mathbb{R}^{n\times k},\mathbf{H}^{\top}\mathbf{H}= \mathbf{I}_{k},
\end{split}}
\end{equation}
The optimal $\mathbf{H}$ for Eq. (\ref{eq:KkM:v3}) can be obtained by taking the $k$ eigenvectors of the respective largest eigenvalues of $\mathbf{K}_{\boldsymbol{\beta}}$ \cite{JegelkaGSSL09}.

ii) \textbf{Optimizing $\boldsymbol{\beta}$ given $\mathbf{H}$}. With $\mathbf{H}$ fixed, $\boldsymbol{\beta}$ can be optimized via solving the following quadratic programming with linear constraints,
\begin{equation}\label{eq:MKKM:kw}
{\begin{split}
\min\limits_{\boldsymbol{\beta}}\;&\;\sum\limits_{p=1}^{m}{\boldsymbol{\beta}}_p^{2}\mathrm{Tr}(\mathbf{K}_{p}
(\mathbf{I}_{n}-\mathbf{H}\mathbf{H}^{\top})), 
s.t.\boldsymbol{\beta}^{\top}\mathbf{1}_{m}=1,\;{\boldsymbol{\beta}}_{p}\geq 0,
\end{split}}
\end{equation}
where $\mathbf{K}_{p}$ is the given $p$-th view's kernel matrix.

The clustering solution is obtained by first normalizing all rows of $\mathbf{H}$ to be on the unit sphere and then performing classical $k$-means clustering on this normalized matrix \cite{scholkopf1998nonlinear}. Along with this line, many variants of MKKM have been proposed in the literature. The work in \cite{Liu2016Multiple} proposes a multiple kernel $\mathit{k}$-means clustering algorithm with matrix-induced regularization to reduce the redundancy and enhance the diversity of the pre-defined kernels. Furthermore, local kernel alignment criterion has been applied to multiple kernel learning to enhance the clustering performance in \cite{Li2016Multiple}.

Although demonstrating promising performance, we observe that most of the existing multiple kernel clustering methods extend the Eq. (\ref{LFA 1}) into multi-view by assuming that $\mathbf{K}_{\boldsymbol{\beta}} = \sum\nolimits_{p=1}^m {\boldsymbol{\beta}}_p^2 \mathbf{K}_p$ and learn an optimal kernel to perform kernel $k$-means clustering. In contrast, few of them utilize multiple partition information to improve clustering quality. Moreover, According to \cite{liu2018late}, this category causes intensive computational complexity and over-complicated optimization increasing the risk of trapped into low-quality local minimum. In the next section, we propose \emph{Late Fusion Multi-view Clustering via Global and Local Alignment Maximization} to address these issues.

\section{Late Fusion Multi-view Clustering via Global and Local Alignment Maximization}\label{proposed algorithm}

\subsection{LF-MVC-GAM}\label{MVC-LFA}

 According to \cite{tao2017ensemble}, multi-view clustering holds the assumption that various clustering structures among different views should be consistent between each other. In the following, we manage to fuse multiple kernel information in partition level through finding consensus partition from multiple clustering results of views.

Firstly, through each single-view kernel $k$-means, the partition matrix $\left\{\mathbf{H}_p\right\}_{p=1}^m$ can be individually obtained. Then we need to define similarity measures between different clustering partitions. Orthogonal to other machine learning tasks, the cluster partition for each view is not unique. In general, for each partition matrix $\mathbf{H}_i \in \mathbb{R}^{n \times k}$, there are in fact $k!$ ($k$ factorial) equivalent partitions through permutations. To circumvent this obstacle, we can use a permutation matrix $\mathbf{W}_i$ for each individual partition $\mathbf{H}_i$ to represent the similar partition. It is not difficult to show that the partition is invariant with respect to permutations.

We then can conduct multiple kernel clustering by performing $k$-means clustering on the combined partition matrix $\sum\nolimits_{p=1}^m {\boldsymbol{\beta}}_p \mathbf{H}_p\mathbf{W}_p$ and get the final consensus clustering result $\mathbf{F}$, which we refer as late fusion. This idea can be mathematically formulated by replacing $\mathbf{X}$ with $ \sum\nolimits_{p=1}^m {\boldsymbol{\beta}}_p \mathbf{H}_p \mathbf{W}_p$ in Eq. (\ref{LFA 1}),
\begin{equation}\label{LFA new}
\begin{split}
\min\limits_{\mathbf{F},\left\{\mathbf{W}_p \right\}_{p=1}^{m}}\;&\;\mathrm{Tr}(\mathbf{B} \mathbf{B}^{\top}) - \mathrm{Tr} ( {\mathbf{F}}^{\top} \mathbf{B} \mathbf{B}^{\top} \mathbf{F}),\\
s.t.   &\; \mathbf{F}\in\mathbb{R}^{\mathit{n} \times \mathit{k}}, {\mathbf{F}}^{\top}\mathbf{F} = \mathbf{I}_{k}, \mathbf{B} = \sum\nolimits_{p=1}^m {\boldsymbol{\beta}}_p \mathbf{H}_p \mathbf{W}_p,
\end{split}
\end{equation}
where $\left\{\mathbf{H}_p \right\}_{p=1}^{m}$ are a set of base partitions obtained by single-view kernel $k$-means on the each kernel matrix $\mathbf{K}_p$,$\left\{\mathbf{W}_p \right\}_{p=1}^{m}$ are relaxed into orthogonal  matrices. ${\boldsymbol{\beta}} \in \mathbb{R}^{m}$ is the view-specific coefficient with non-negative and $\ell_2$ norm constraint.

Although Eq. (\ref{LFA new}) fulfills the multiple kernel $k$-means clustering with partition level fusion, it is still troublesome and costive to directly solve it with large-scale datasets for the reason that the optimal $\mathbf{F}$ for Eq. (\ref{LFA new}) should be taken by the $k$ eigenvectors that correspond to the $k$ largest eigenvalues of $\mathbf{B} \mathbf{B}^{\top} = ( {\sum\nolimits_{p=1}^m {\boldsymbol{\beta}}_p \mathbf{H}_p \mathbf{W}_p)}^{\top} ( \sum\nolimits_{p=1}^m {\boldsymbol{\beta}}_p \mathbf{H}_p \mathbf{W}_p) \in \mathbb{R}^{n \times n}$. Hence the time complexity to solve Eq. (\ref{LFA new}) still needs $\mathcal{O}(n^3)$, which makes it difficult for real large-scale data.

In the following, we firstly show our theoretical result on the connection between the Eq. (\ref{LFA new}) and the proposed late fusion alignment, as shown in Theorem \ref{theorem1}. Before that, the following Lemma \ref{proposition1} and \ref{proposition2} are used.

\begin{lemma}\label{proposition1}
For any given matrix $\mathbf{P}\in \mathbb{R}^{k\times k}$, the inequality $\mathrm{Tr}^2 (\mathbf{P}) \leq k \cdot \mathrm{Tr}({\mathbf{P}}^{\top}\mathbf{P})$ always holds.
\end{lemma}
\begin{lemma}\label{proposition2}
$\mathrm{Tr}(\mathbf{B} \mathbf{B}^{\top}) \leq m^2 k$.
\end{lemma}
The detailed proof of Lemma \ref{proposition1} and \ref{proposition2} can be found in the appendix.

\begin{lemma}\label{lemma3}
The optimal solution of $\mathbf{F}$ to maximize $\mathrm{Tr} ({\mathbf{F}}^{\top} \mathbf{B})$. is also a solution of the global optimum  of Eq. (\ref{LFA new}). 
\end{lemma}
\begin{proof}
	For $\mathrm{Tr} ({\mathbf{F}}^{\top} \mathbf{B})$, we can obtain the optimum solution with SVD. Suppose that the matrix $\mathbf{B}$ has the  rank-$k$ truncated singular value decomposition form as $\mathbf{B} = {\mathbf{S} \mathbf{\Sigma} \mathbf{V}^{\top}}$, where ${\mathbf{S}_{k}} \in \mathbb{R}^{n \times n},{\mathbf{\Sigma}_{k}} \in \mathbb{R}^{n \times k},{\mathbf{V}_{k}} \in \mathbb{R}^{k \times k}$. Maximizing $\mathrm{Tr} ({\mathbf{F}}^{\top} \mathbf{B})$ has a closed-form solution that $\mathbf{F}_{3} = \mathbf{S}_{k} \mathbf{V}^{\top}$, where $\mathbf{S}_{k}$ denotes the $k$ left singular vector corresponding to $k$-largest singular values. The detailed deviation can found in our Theorem 2 in the manuscript. Further, we can easily obtain that the solution of $\Large{\textcircled{\small{1}}}$ is the product of the $k$-largest eigenvectors of $\mathbf{B} \mathbf{B}^{\top}$ and arbitrary rank-$k$ orthogonal square matrix. The the $k$-largest eigenvectors of $\mathbf{B} \mathbf{B}^{\top}$ is also the rank-$k$ left singular vector of $\mathbf{B}$. Therefore, it is not difficult to verify that the solution of $\Large{\textcircled{\small{3}}}$ can also reach the global optimum  of $\Large{\textcircled{\small{1}}}$. This completes the proof.
\end{proof}

\begin{theorem}\label{theorem1}
Based on Lemma \ref{proposition1} and \ref{lemma3}, minimizing Eq. (\ref{LFA new}) is conceptually equivalent to maximizing $\mathrm{Tr} ({\mathbf{F}}^{\top} \mathbf{B})$.
\end{theorem}
\begin{proof}
By taking $\mathbf{P} = {\mathbf{F}}^{ \top} \mathbf{B}$ in Lemma \ref{proposition1}, we can see that $\mathrm{Tr}^2 ({\mathbf{F}}^{\top} \mathbf{B}) \leq k\mathrm{Tr}({\mathbf{F}}^{\top} \mathbf{B} \mathbf{B}^{\top}{\mathbf{F}} )$.
As a result, $\mathrm{Tr}(\mathbf{B} \mathbf{B}^{\top}) - \mathrm{Tr} ( {\mathbf{F}}^{\top} \mathbf{B} \mathbf{B}^{\top} \mathbf{F}) \leq m^2k - \mathrm{Tr} ( {\mathbf{F}}^{\top} \mathbf{B} \mathbf{B}^{\top} \mathbf{F}) \leq  m^2k - \frac{1}{k} \mathrm{Tr}^2 ({\mathbf{F}}^{\top}\mathbf{B})$. Based on this, $m^2k - \frac{1}{k} \mathrm{Tr}^2 ({\mathbf{F}}^{\top} \mathbf{B})$ is the upper bound for Eq. (\ref{LFA new}). Therefore, if we want to minimize Eq. (\ref{LFA new}), we need to maximize $\mathrm{Tr} ({\mathbf{F}}^{\top} \mathbf{B})$ as possible. Further, in our case, we theoretically prove that the solution of maximizing $\mathrm{Tr} ({\mathbf{F}}^{\top} \mathbf{B})$ is a global optimum for problem Eq. (\ref{LFA new}), as stated in Lemma \ref{lemma3}. Also, it is not difficult to verify that the solutions share the same clustering structure, leading to the same clustering results. In addition, solving $\mathrm{Tr} ({\mathbf{F}}^{\top} \mathbf{B})$ is proved with $\mathcal{O}(n)$ time complexity while Eq. (\ref{LFA new}) is $\mathcal{O}(n^3)$, making it more efficient in large-scale data clustering tasks.
\end{proof}

Based on Theorem \ref{theorem1}, we derive a simple but effective clustering loss function for clustering maximizing $\mathrm{Tr} ({\mathbf{F}}^{\top} \mathbf{B})$ for replacement in Eq. (\ref{LFA new}) where $\mathbf{B} = \sum\nolimits_{p=1}^m \boldsymbol{{\beta}}_p \mathbf{H}_p \mathbf{W}_p$. Compared with minimizing Eq. (\ref{LFA new}), this new formulation is much easier to optimize since it only requires the singular value decomposition of $\mathbf{B} \in \mathcal{R}^{n \times k}$ in contrast to $\mathbf{B}\mathbf{B}^{\top}\in \mathcal{R}^{n \times n}$ so that significantly reduces the time-complexity and simplifies the optimization procedure.

Based on the above discussion, we obtain the objective function of our LF-MVC-GAM as follows,
\begin{equation}\label{LFA 3}
\begin{split}
\; \max\limits_{\mathbf{F},\left\{\mathbf{W}_p\right\}_{p=1}^m,\boldsymbol{\beta}} &\; \mathrm{Tr} ({\mathbf{F}}^{ \top} \sum\nolimits_{p=1}^m {{\beta}}_p \mathbf{H}_p \mathbf{W}_p) +  \lambda \mathrm{Tr} ({\mathbf{F}}^{\top} \mathbf{M}),\\
\;&\; s.t.\;{\mathbf{F}}^{\top}\mathbf{F} = \mathbf{I}_k, \mathbf{{W}}_p^{\top}\mathbf{W}_p = \mathbf{I}_k,\\
\;&\; {\Vert {\boldsymbol{\beta}} \Vert}_2 = 1, {\boldsymbol{\beta}}_p \geq 0,\;\forall p,
\end{split}
\end{equation}
where $\mathrm{Tr}({\mathbf{F}}^{\top}\mathbf{M})$ is a regularization on the consensus partition. It is required that the consensus clustering matrix $\mathbf{F}$ shall reside in the neighborhood of a pre-specified $\mathbf{M}$.$\lambda$ is a hyper-parameter to trade-off the loss of partition alignment and the regularization term.This can be fulfilled by minimizing $\|\mathbf{F}-\mathbf{M}\|_{\mathrm{F}}^2$ to guide the learning of $\mathbf{F}$, where $\mathbf{M}$ could be some prior knowledge about the clusters. In our paper, we adopt $\mathbf{M}$ as the partition matrix obtained form average combined kernel $k$-means. Note that minimizing $\|\mathbf{F}-\mathbf{M}\|_{\mathrm{F}}^2$ is equivalent to maximizing $\mathrm{Tr}(\mathbf{F}^{\top}\mathbf{M})$.

Compared with existing multi-view clustering approaches, Eq. (\ref{LFA 3}) enjoys several strengths. (1) \textit{more natural to integrate partition level information.} The early similarity stage might fail to obtain the optimal clustering result since some information might get occupied  during optimization process. For instance, many real-world data are often complicated due to noise or outliers, which leads to a poor quality kernel. Hence it makes more natural and reasonable to fuse multiple views' information in partition level as each partition will reach an agreement on consensus partition. (2) \textit{lower computation complexity.} Contrary to the previous methods combing the optimal kernel $\mathbf{K}_{\boldsymbol{\beta}}$ with the size of $n \times n$, we unifies a set of transformed partitions into consensus partition $\mathbf{F}$ with the size of $n \times k$ which significantly decrease the computation complexity from $\mathcal{O}(n^3)$ to $\mathcal{O}(nk^2)$ per iteration. (3) \textit{more robust to local minimums and improve clustering performance.} By adding the third regularizing term to guide the consensus clustering indicator matrix $\mathbf{F}$, the resultant algorithms can effectively avoid low-quality local optimums and achieve more satisfactory clustering performance. A three-step iterative algorithm with proved convergence is designed to solve the optimization problem in Eq. (\ref{LFA 3}). Interested readers can refer to \cite{wang2019mvclfa} for the detail.

\subsection{LF-MVC-LAM}

\subsubsection{The Formulation of LF-MVC-LAM}

The proposed LF-MVC-GAM in section \ref{MVC-LFA} globally aligns the consensus partition $\mathbf{F}$ with the individual transformed partition and achieves considerably improved clustering. As can be seen from Eq. (\ref{LFA 3}), it is crucial to select high-quality base partitions $\left\{\mathbf{H}_p\right\}_{p=1}^m$ and further guide the learning of consensus partition $\mathbf{F}$. However, according to the analysis in \cite{gonen2008localized,gonen2014localized,Li2016Multiple}, globally aligning the consensus partition with the individual transformed partition would ignore the local intrinsic structures in high-dimensional partition space and therefore deteriorate the clustering performance. To overcome this issue and further improve clustering performance, we introduce a localized late fusion alignment variant to focus on local sample pairs that shall be stayed closer, and alleviate unreliable similarity evaluation for farther sample pairs by following \cite{Li2016Multiple}. By integrating the local alignment of $\left\{\mathbf{H}_p\right\}_{p=1}^m$ for each sample $\mathbf{x}_i$ into Eq. (\ref{LFA 3}), we obtain the following objective,
\begin{equation}\label{local-LFA}
\begin{split}
\; \max\limits_{\mathbf{F},\left\{\mathbf{W}_p\right\}_{p=1}^m,\boldsymbol{\beta}}  &\; \sum_{i=1}^n \mathrm{Tr} ({\mathbf{F}}^{\top} \sum\nolimits_{p=1}^m {\boldsymbol{\beta}}_p \tilde{\mathbf{H}}_p^{(i)}\mathbf{W}_p) +  \lambda \mathrm{Tr} ({\mathbf{F}}^{\top} \tilde{\mathbf{M}}),\\
\;&\; s.t.\;{\mathbf{F}}^{\top}\mathbf{F} = \mathbf{I}_k, \mathbf{{W}}_p^{\top}\mathbf{W}_p = \mathbf{I}_k,\\
\;&\; {\Vert {\boldsymbol{\beta}} \Vert}_2 = 1, {\boldsymbol{\beta}}_p \geq 0,\;\forall p,\\
\;&\;\tilde{\mathbf{H}}_p^{(i)} = {\mathbf{A}^{(i)}_{p}}^{\top}\mathbf{H}_p, \tilde{\mathbf{M}} = {\mathbf{A}^{(i)}}^{\top} \mathbf{M},
\end{split}
\end{equation}
where $\lambda$ is a hyerparameter to trade-off the loss of local partition alignment and the regularization term on ideal consensus partition $\mathbf{F}$.  $\tilde{\mathbf{H}}_p^{(i)} = {\mathbf{A}^{(i)}_{p}}^{\top}\mathbf{H}_p$ where ${\mathbf{A}^{(i)}_{p}} \in {\left\{0,1\right\}}^{n\times n}$ is a matrix indicating the $\tau$-nearest neighbors of $i$-th sample.

Compared with LF-MVC-GAM, Eq. (\ref{local-LFA}) improves multi-view clustering by localized alignment and inherits the nice properties of LF-MVC-GAM. Apart from that, the local alignment brought in Eq. (\ref{local-LFA}) well preserves the local intrinsic clustering structure in multiple partition spaces and abandons inaccurate similar measures for distant pair samples. Further, by keeping local similarity structure and incorporating partition level information, high-quality local partitions are helpful for well utilizing the consensus partition with local information preserving, which are desired to improve the clustering performance.

\subsubsection{Alternate Optimization of LF-MVC-LAM}

In this section, we design an effective three-step alternate algorithm with proved convergence, where each step could be easily solved by the existing off-the-shelf packages.

\paragraph{\textbf{Optimizing $\mathbf{F}$ with fixed $\left\{\mathbf{W}_p\right\}_{p=1}^m$ and $\boldsymbol{\beta}$}} With $\left\{\mathbf{W}_p\right\}_{p=1}^m$ and $\boldsymbol{\beta}$ being fixed, the optimization Eq. ({\ref{local-LFA}}) could be rewritten as follows,
\begin{equation}
\label{optimization H}
\begin{split}
\max\limits_{\mathbf{F}} \mathrm{Tr}({\mathbf{F}}^{\top}\mathbf{U})  ~~
s.t. ~~\mathbf{F}^{ \top}\mathbf{F} = \mathbf{I}_k,
\end{split}
\end{equation}
where $\mathbf{U} = \sum_{i=1}^n \sum_{p=1}^m{\boldsymbol{\beta}}_p \tilde{\mathbf{H}}_p^{(i)}\mathbf{W}_p + \lambda \tilde{\mathbf{M}}$. In machine learning and computer vision community, Eq. (\ref{optimization H}) is called Orthogonal Procrustes Analysis which has been well studied in  literature \cite{DBLP:books/sp/HastieFT01,DBLP:conf/kdd/NieTL18}. Its optimum in stefiel manifold has also been provided in \cite{DBLP:books/sp/HastieFT01,DBLP:conf/kdd/NieTL18}. In Theorem \ref{optimization-H-theorem}, we also offer an alternative proof for the optimum solving the respective subproblem.

The problem in Eq. (\ref{optimization H}) could be easily solved by taking the singular value decomposition (SVD) of the given matrix $\mathbf{U}$. Here the following Theorem gives a closed-form solution for the problem in Eq. (\ref{optimization H}).

\begin{theorem}\label{optimization-H-theorem}
Suppose that the matrix $\mathbf{U}$ in Eq. (\ref{optimization H}) has the  rank-$k$ truncated singular value decomposition form as $\mathbf{U} = {\mathbf{S}_{k} \mathbf{\Sigma}_{k} \mathbf{V}_{k}^{\top}}$, where ${\mathbf{S}_{k}} \in \mathbb{R}^{n \times k},{\mathbf{\Sigma}_{k}} \in \mathbb{R}^{k \times k},{\mathbf{V}_{k}} \in \mathbb{R}^{k \times k}$. The optimization problem in Eq. (\ref{optimization H}) has a closed-form solution as follows,
\begin{equation}
\label{equation H}
\mathbf{F} = \mathbf{S}_{k} \mathbf{V}_{k}^{\top},
\end{equation}
\end{theorem}

\begin{proof}
By taking the the singular value decomposition that $\mathbf{U} ={\mathbf{S}\boldsymbol{\Sigma} \mathbf{V}^{ \top}}$, the Eq. (\ref{optimization H}) could be rewritten as,
\begin{equation}
\label{stefield}
\mathrm{Tr}(\mathbf{F}^{\top}\mathbf{S}\boldsymbol{\Sigma}\mathbf{V}^{\top}) = \mathrm{Tr} (\mathbf{V}^{\top}\mathbf{F}^{\top} \mathbf{S} \boldsymbol{\Sigma}) = \mathrm{Tr}(\mathbf{Q} \boldsymbol{\Sigma}),
\end{equation}
where $\mathbf{Q} = \mathbf{V}^{\top} \mathbf{{F}^{\top}} \mathbf{S}$, we have $\mathbf{Q} \mathbf{Q}^{\top} = \mathbf{V}^{ \top} \mathbf{{F}^{\top}} \mathbf{S} \mathbf{S}^{\top} \mathbf{F} \mathbf{V} = \mathbf{I}_k$. Together with Theorem \ref{Q-singular}, we can obtain that $\mathrm{Tr} (\mathbf{V}^{ \top} \mathbf{{F}^{\top}} {\mathbf{S} \boldsymbol{\Sigma}}) = \mathrm{Tr}(\mathbf{Q} \boldsymbol{\Sigma}) \leq \sum_{i=1}^{k} \sigma_{i}$. Hence in order to maximize the value of Eq. (\ref{optimization H}), the solution should be given as Eq. (\ref{equation H}).
\end{proof}

\begin{theorem}\label{Q-singular}
For a given matrix $\mathbf{U}$ in Eq. (\ref{optimization H}), all the singular values ${\left\{\sigma_i\right\}}_{i=1}^k$ of $\mathbf{Q}$ in Eq. (\ref{stefield}) are nonnegative.
\end{theorem}

\begin{proof}
From Eq. (\ref{optimization H}), we can obtain that $\mathbf{U} = \sum_{i=1}^n \sum_{p=1}^m {\boldsymbol{\beta}}_p \tilde{\mathbf{H}}_p^{(i)} \mathbf{W}_p + \lambda \tilde{\mathbf{M}}$. Further it is obvious that the singular values of each $ \tilde{\mathbf{H}}_p^{(i)} $ are non-negative, with the rotation matrix $\mathbf{W}_p$, the singular values of $\tilde{\mathbf{H}}_p^{(i)}\mathbf{W}_p$ are still kept non-negative. Moreover, the singular values of $\tilde{\mathbf{M}}$ are also non-negative due to the $\mathbf{M}$ is formed by a PSD kernel matrix. Therefore the Theorem \ref{Q-singular} is proved.
\end{proof}

\paragraph{\textbf{Optimizing $\left\{\mathbf{W}_p\right\}_{p=1}^m$ with fixed $\mathbf{F}$ and $\boldsymbol{\beta}$}}

With $\mathbf{F}$ and $\boldsymbol{\beta}$ being fixed, for each single $\mathbf{W}_p$, the optimization problem in Eq. ({\ref{local-LFA}}) is equivalent to Eq. ({\ref{optimization Wp}}) as follows,
\begin{equation}\label{optimization Wp}
\begin{split}
\max\limits_{\mathbf{W}_p} \;&\;\mathrm{Tr}(\mathbf{W}_p^{\top}\mathbf{L})\;\;s.t. \; \mathbf{W}_p^{ \top}\mathbf{W}_p = \mathbf{I}_k,
\end{split}
\end{equation}
where $\mathbf{L} = \sum_{i=1}^n {\boldsymbol{\beta}}_p \left(\tilde{\mathbf{H}}_p^{(i)}\right)^{\top} \mathbf{F}$. And this problem in Eq. (\ref{optimization Wp}) could be easily solved by taking the singular value decomposition (SVD) of the given matrix $\mathbf{L}$. Like the closed-form expressed in Theorem \ref{optimization-H-theorem}, if the matrix $\mathbf{V}$ has the singular value decomposition form as $\mathbf{L} = {\mathbf{S} \boldsymbol{\Sigma} \mathbf{G}^{\top}}$, the optimization in Eq. (\ref{optimization Wp}) has a closed-form solution as $\mathbf{W}_{p} = \mathbf{S} \mathbf{G}^{\top}$. Hence we optimize one $\mathbf{W}_p$ with other $\mathbf{W}_{i \not = p}$ fixed at each iteration. As a result, we can obtain a set of optimized $\left\{{\mathbf{W}_p}\right\}_{p=1}^m$.

\paragraph{\textbf{Optimizing $\boldsymbol{\beta}$ with fixed $\mathbf{F}$ and $\left\{\mathbf{W}_p\right\}_{p=1}^m$}}

With $\mathbf{F}$ and $\left\{\mathbf{W}_p\right\}_{p=1}^m$ being fixed, the optimization problem in Eq. ({\ref{local-LFA}}) is equivalent to the optimization problem as follows,
\begin{equation}\label{optimization gamma1}
\begin{split}
\max_{\boldsymbol{\beta}} \sum\nolimits_{p=1}^{m} {\boldsymbol{\beta}}_p \boldsymbol{{\delta}}_p,  ~~
s.t. {\Vert {\boldsymbol{\beta}} \Vert}_2 = 1, {\boldsymbol{\beta}}_p \geq 0,\;\forall p,
\end{split}
\end{equation}
where $\boldsymbol{{\delta}}_p = \sum_{i=1}^n \mathrm{Tr}({\mathbf{F}}^{\top}\tilde{\mathbf{H}}_p^{(i)}\mathbf{W}_p)$.
According to \cite{wang2019mvclfa}, the optimization in Eq. (\ref{optimization gamma1}) admits an closed-form solution as follows,
\begin{equation}
\label{optimization beta}
{\boldsymbol{\beta}}_p =  \boldsymbol{{\delta}}_p/\sqrt{\sum\nolimits_{p=1}^m {{\boldsymbol{\delta}}_p^2}}.
\end{equation}

In sum, our algorithm for solving Eq. (\ref{local-LFA}) is outlined in Algorithm \ref{Algorithm-proposed2}, where the  $\text{obj}^{(t)}$ denotes the objective value at the $t$-th iteration. The following Theorem \ref{convergence} shows our algorithm is guaranteed to converge. Also, as shown in the experimental study, LF-MVC-LAM usually converge in less than $10$ iterations.

\alglanguage{pseudocode}
\begin{algorithm}[!htbp]
{\caption{ Proposed LF-MVC-LAM}\label{Algorithm-proposed2}
\begin{algorithmic}[1]
\State \textbf{Input}: $\left\{\mathbf{K}_p\right\}_{p=1}^m,\,k,\,\lambda $ and $\epsilon_{0}$.
\State \textbf{Output}: Consensus partition $\mathbf{F} $ and $\boldsymbol{\beta}$.
\State Initialize $\left\{\mathbf{W}_p\right\}_{p=1}^m = \mathbf{I}_k,\boldsymbol{\beta} = \frac {1}{\sqrt m}$ and $t=1$.
\State Calculate the neighborhood matrix $\mathbf{A}^{(i)}_{p}$ for the $i$-th sample in $p$-th view by $\mathbf{K}_p$ and  $\mathbf{\tilde{M}}$ by local kernel $k$-means with average kernel.
\State Calculate the local partition matrices $\tilde{\mathbf{H}_p^{(i)}}$ for each sample $x_i$.
\Repeat
   \State Update $\mathbf{F}$ by solving Eq. (\ref{optimization H}) with fixed $\left\{\mathbf{W}_p\right\}_{p=1}^m$ and $\boldsymbol{\beta}$.
   \State Update $\left\{\mathbf{W}_p\right\}_{p=1}^m$ with fixed $\mathbf{F}$ and $\boldsymbol{\beta}$ by solving Eq. (\ref{optimization Wp}).
   \State Update $\boldsymbol{\beta}$ by solving Eq. (\ref{optimization beta}) with fixed $\mathbf{F}$ and $\left\{\mathbf{W}_p\right\}_{p=1}^m$.
   \State $t = t+1$.
\Until{$\Big(\text{obj}^{(t-1)}-\text{obj}^{(t)}\Big)/\text{obj}^{(t)}\leq\epsilon_{0}$}
\State Perform Floyd algorithm on $\mathbf{F}$ to get the final clustering result.
\end{algorithmic}}
\end{algorithm}

\begin{theorem}
\label{convergence}
The proposed optimization function for  Algorithm \ref{Algorithm-proposed2} is proved to be upper-bounded.
\end{theorem}

\begin{proof}
Note that for $ \forall p,q,i,\,\mathrm{Tr}[{({\boldsymbol{\beta}}_{p}\tilde{\mathbf{H}}_p^{(i)}\mathbf{W}_p)}^{\top} ({\boldsymbol{\beta}}_{q}\tilde{\mathbf{H}}_q^{(i)}\mathbf{W}_q)]  \leq  \mathrm{Tr}[{(\tilde{\mathbf{H}}_p^{(i)}\mathbf{W}_p)}^{\top} (\tilde{\mathbf{H}}_q^{(i)} \mathbf{W}_q)] \leq \frac{1}{2} (\mathrm{Tr} [{(\mathbf{H}_p\mathbf{W}_p)}^{\top} {(\mathbf{H}_p\mathbf{W}_p)}]  + \mathrm{Tr} [{(\mathbf{H}_q\mathbf{W}_q)}^{\top} {(\mathbf{H}_q\mathbf{W}_q)}]) = k$. As a result, we could derive the upper bound of the optimization goal in Eq. ({\ref{local-LFA}}). We can obtain that $\forall i, \mathrm{Tr} ({\mathbf{F}}^{\top}\sum_{p=1}^m {\boldsymbol{\beta}}_p \tilde{\mathbf{H}}_p^{(i)} \mathbf{W}_p)  \leq \frac{1}{2} (\mathrm{Tr} [{\mathbf{F}}^{\top} {\mathbf{F}}]  + \mathrm{Tr} [{(\sum_{p=1}^m {\boldsymbol{\beta}}_p \tilde{\mathbf{H}}_p^{(i)} \mathbf{W}_p)}^{\top} {(\sum_{p=1}^m {\boldsymbol{\beta}}_p \tilde{\mathbf{H}}_p^{(i)} \mathbf{W}_p)}]) \! \leq \! \frac{1}{2} (\mathrm{Tr} [{\mathbf{F}}^{\top} {\mathbf{F}}]  + \mathrm{Tr} (\sum_{p,q=1}^m {({\boldsymbol{\beta}}_p\mathbf{H}_p\mathbf{W}_p)}^{\top} ({\boldsymbol{\beta}}_p\mathbf{H}_q\mathbf{W}_q))) \leq \frac{k}{2} (m^2+1)$. Meanwhile, the $({\mathbf{F}}^{\top} \tilde{\mathbf{M}}) \leq  \frac{1}{2} (\mathrm{Tr} [{\mathbf{F}}^{\top} {\mathbf{F}}]  + \mathrm{Tr} [{\tilde{\mathbf{M}}}^{\top} {\tilde{\mathbf{M}}}]) = k$. Consequently, the whole optimization function is upper bounded. This completes the proof. 
\end{proof}

\subsection{Discussion and Extensions}

We end up this section with analyzing the convergence, computational complexities and potential extensions of our proposed LF-MVC-GAM and LF-MVC-LAM.

\textit{Convergence}: As can be seen form Eq. (\ref{local-LFA}),  the whole function is not jointly convex when all variables are considered simultaneously.
Instead, we propose an alternate optimization algorithm to optimize each variable with the other three variables been fixed. Let we define $\mathbf{F}^{(t)}, {\left\{\mathbf{W}_{p}^{(t)}\right\}_{p=1}^{m}}, \boldsymbol{\beta}^{(t)}$ be the solution at the $t$-iteration.\\
i)Optimizing $\mathbf{F}$ with fixed $\left\{\mathbf{W}_{p}\right\}_{p=1}^{m}$ and $\boldsymbol{\beta}$. Given $\boldsymbol{\beta}^{(t)} and \left\{\mathbf{W}_{p}^{(t)}\right\}_{p=1}^{m}$, the optimum of $\mathbf{F}$ can be analytically obtained by Theorem 2. The detailed derivation can be found in the manuscript. Suppose the obtained optimal solution be $\mathbf{F}^{(t+1)}$. We have\\ 
\begin{equation}
\label{1}
\resizebox{0.5\textwidth}{!}{$
\begin{split}
\mathcal{J}\left(\mathbf{F}^{(t+1)}, {\left\{\mathbf{W}_{p}^{(t)}\right\}_{p=1}^{m}}, \boldsymbol{\beta}^{(t)} \right) \geq \mathcal{J}\left(\mathbf{F}^{(t)}, {\left\{\mathbf{W}_{p}^{(t)}\right\}_{p=1}^{m}}, \boldsymbol{\beta}^{(t)} \right).
\end{split}$}       
\end{equation}\\
ii)Optimizing $\left\{\mathbf{W}_{p}\right\}_{p=1}^{m}$ with fixed $\mathbf{F}$ and $\boldsymbol{\beta}$. Given $\mathbf{F}^{(t+1)}, \boldsymbol{\beta}^{(t)}$ , the optimization respecting to  $\left\{\mathbf{W}_{p}\right\}_{p=1}^{m}$ can be analytically obtained. Suppose the obtained optimal solution be $\left\{\mathbf{W}_{p}^{(t+1)}\right\}_{p=1}^{m}$. We have 
\begin{equation}
\label{2}
\resizebox{0.5\textwidth}{!}{$
\begin{split}
\mathcal{J}\left(\mathbf{F}^{(t+1)}, {\left\{\mathbf{W}_{p}^{(t+1)}\right\}_{p=1}^{m}}, \boldsymbol{\beta}^{(t)} \right) \geq \mathcal{J}\left(\mathbf{F}^{(t+1)}, {\left\{\mathbf{W}_{p}^{(t)}\right\}_{p=1}^{m}}, \boldsymbol{\beta}^{(t)} \right).
\end{split}$}
\end{equation}\\
iii)Optimizing $\boldsymbol{\beta}$ with fixed $\left\{\mathbf{W}_{p}\right\}_{p=1}^{m}$ and $\mathbf{F}$. Given ${\left\{\mathbf{W}_{p}^{(t+1)}\right\}_{p=1}^{m}}, \mathbf{F}^{(t+1)}$ , the optimization respecting to  $\boldsymbol{\beta}$ can be optimally solved with close-formed solutions. Suppose the obtained optimal solution be $\boldsymbol{\beta}^{(t+1)}$. We have 
\begin{equation}
\label{3}
\resizebox{0.5\textwidth}{!}{$
\begin{aligned}
\mathcal{J}\left(\mathbf{F}^{(t+1)}, {\left\{\mathbf{W}_{p}^{(t+1)}\right\}_{p=1}^{m}}, \boldsymbol{\beta}^{(t+1)} \right) \geq \mathcal{J}\left(\mathbf{F}^{(t+1)}, {\left\{\mathbf{W}_{p}^{(t+1)}\right\}_{p=1}^{m}}, \boldsymbol{\beta}^{(t)} \right).
\end{aligned}$}     
\end{equation}
\\
Together with Eq. (\ref{1}), (\ref{2}) and (\ref{3}), we have that
\begin{equation}
\label{5}
\resizebox{0.5\textwidth}{!}{$
\begin{split}
\mathcal{J}\left(\mathbf{F}^{(t+1)}, {\left\{\mathbf{W}_{p}^{(t+1)}\right\}_{p=1}^{m}}, \boldsymbol{\beta}^{(t+1)} \right) \geq \mathcal{J}\left(\mathbf{F}^{(t)}, {\left\{\mathbf{W}_{p}^{(t)}\right\}_{p=1}^{m}}, \boldsymbol{\beta}^{(t)} \right).
\end{split}$}       
\end{equation}
which indicates that the objective function of our algorithm in Eq. (\ref{local-LFA}) monotonically increases with the increase of iterations. Also, the objective function in Eq. (\ref{local-LFA}) is upper-bounded by the provided Theorem \ref{convergence}. As a result, the proposed algorithm can be verified to converge to a local minimum according to \cite{DBLP:journals/npsc/BezdekH03}.

\textit{Computational Complexity}: With the optimization process outlined in Algorithm \ref{Algorithm-proposed2}, the total time complexity consists of three parts referring to the alternate steps. The first step of algorithm \ref{Algorithm-proposed2}, mentioned in Eq. (\ref{optimization H}), actually needs an singular value decomposition(SVD) of a matrix with the size of $n \times k$ and therefore needs $\mathcal{O}(nk^2)$ (since $k \ll n$ ). In the second step, similar with the first step, we need to solve $m$ subproblems with SVD on $\mathbf{L}$ in Eq. (\ref{optimization Wp}). Hence the time complexity is $\mathcal{O} (mnk^2)$. As for the third step, the time complexity for calculating $\boldsymbol{\delta}$ needs $\mathcal{O}(mnk^2)$. The whole time complexity of our proposed algorithms are $\mathcal{O} (mnk^2 + nk^2)$ per iteration. This implies that our algorithm has a linearly growing complexity with the number of samples, making it efficiently to handle large-scale tasks comparing to the state-of-the-art multiple-kernel clustering algorithms.

\textit{Regularization on $\mathbf{F}$}: The regularization term on consensus partition $\mathbf{F}$ is essential to improve the late fusion clustering performance. In this work, we regularize $\mathbf{F}$ by assuming that it lies in the neighborhood of average partition space $\mathbf{M}$ or local variant $\tilde{\mathbf{M}}$.  Other approaches to generate regularization term can also be designed to further improve the clustering performance. Apart from that, many task related prior properties such as low-rank and sparse can be jointed into the optimization objective.

\textit{Extensions}: LF-MVC-GAM and LF-MVC-LAM can be easily extended with the following directions. Firstly, late fusion alignment could be further improved by capturing the noises or low-quality partitions existing in basic partitions. For example, we could integrate the basic partitions  $\left\{\mathbf{H}_p\right\}_{p=1}^m $ into the optimization procedure to capture more advanced base partitions. By doing so, the high-quality basic partitions are further used to guide the generation of consensus partition. Secondly, we could apply more similarity-based clustering methods to generate basic partitions. Exploring other generating methods and evaluating their clustering performance will be an interesting future work.

\section{Generalization Error Bound Analysis of LF-MVC-GAM and LF-MVC-LAM}\label{Generalization}

In this section, we derive the generalization bounds of the proposed algorithms via exploiting the reconstruction error. Given an input space $\mathcal{X}$, $n$ samples $\left\{\mathbf{x}_i\right\}_{i=1}^n$ drawn i.i.d. from an unknown sampling distribution $\mu$. Then the label space is $\mathcal{Y}$ and $\mathcal{G}$ is a group of mapping functions which each function $g: \mathcal{X} \rightarrow \mathcal{Y}$. Recalling that our proposed formulation for LF-MVC-GAM is given in Eq. (\ref{LFA 3}).

Let $\mathbf{C} = [\mathbf{c}_1,\cdots,\mathbf{c}_k]$ be the learned center matrix composed of the $k$ centroids, and $\boldsymbol{\beta},\,\{\mathbf{W}_{p}\}_{p=1}^{m}$ the learned kernel weights and permutation matrices by the proposed LF-MVC-GAM and LF-MVC-LAM. The effective LF-MVC-GAM and LF-MVC-LAM should make the following empirical error small \cite{maurer2008generalization,liu2016dimensionality},
\begin{equation}\label{eq:Generalization1}
{
1-\frac{1}{n} \sum_{i=1}^{n}\left[\max\limits_{ \mathbf{y}\in\{\mathbf{e}_1,\cdots,\mathbf{e}_k\}}
\langle h(\mathbf{x}_i),\mathbf{C}\mathbf{y}\rangle\right],
}
\end{equation}
where $i$ denotes the index of the sample, $h(\mathbf{x}_i)=\sum_{p=1}^{m}{\boldsymbol{\beta}}_p\mathbf{W}_{p}^{\top}h_p(\mathbf{x}_{i}^{(p)}) +  \lambda \mathbf{M}_{i,:}^{\top}$, and $\mathbf{e}_1,\ldots,\mathbf{e}_k$ form the orthogonal bases of $\mathbb{R}^k$. $h_p(\mathbf{x}_{i})$ denotes the $i$-th row of $\mathbf{H}_p$ representing the $i$-th sample's representation in $p$-th view and $\mathbf{M}_{i,:}$ denotes the $i$-th row of the regularization matrix $\mathbf{M}$. The operator $\langle \mathbf{A},\mathbf{B}\rangle$ denotes  $\mathrm{Tr} (\mathbf{A}^{\top} \mathbf{B})$.

We define the function class $\mathcal{F}$ for LF-MVC-GAM first:
\begin{equation}\label{eq:Generalization2}
{
\begin{split}
& \mathcal{F}= \Big\{g:\;\mathbf{x}\mapsto 1 - \max\limits_{\mathbf{y}\in\{\mathbf{e}_1,\cdots,\mathbf{e}_k\}}
\left\langle \sum\limits_{p=1}^{m}{\boldsymbol{\beta}}_p\mathbf{W}_{p}^{\top}h_p(\mathbf{x}^{(p)}),\, \mathbf{C}\mathbf{y}\right\rangle \\ &{\Big|}{\mathbf{F}}^{\top}\mathbf{F} = \mathbf{I}_k,\,\mathbf{{W}}_p^{\top}\mathbf{W}_p = \mathbf{I}_k,\,{\Vert {\boldsymbol{\beta}} \Vert}_2 = 1,\,{\boldsymbol{\beta}}_p \geq 0,  \forall p,\forall \mathbf{x}\in\mathcal{X}\Big\},
\end{split}
}
\end{equation}
where $\mathbf{x} $ is a multi-view data sample feature containing $m$ views $\left[\mathbf{x}^{(1)},\mathbf{x}^{(2)}, \cdots, \mathbf{x}^{(m)} \right]$ and $\mathbf{C}$ is the obtained cluster centroids by performing Floyd algorithm on $\mathbf{F}$\cite{DBLP:journals/prl/Jain10}.

We have the following theorem on the generalization error bound of our proposed algorithms according to \cite{mohri2018foundations}.
\begin{theorem}\label{1}
Let $\mathcal{F}$ be a family of functions class learned by LF-MVC-GAM and LF-MVC-LAM mapping on $\mathcal{X}$.
For any $\delta>0$, with probability at least $1-\delta$, the following holds for all $g\in \mathcal{F}$:
\begin{equation}
\begin{split}
 \mathbb{E}\left[g({\mathbf{x}})\right]-\frac{1}{n}\sum\nolimits_{i=1}^{n}g({\mathbf{x}}_i)\leq \frac{\sqrt{\pi/2}k}{\sqrt{n}} + \left( 8m \right) \sqrt{\frac{\log{1/\delta}}{2n}}.
\end{split}
\end{equation}
\end{theorem}
Since $n \gg k$ and $n \gg m$ in most cases, the Theorem \ref{1} implies the generalization bounds of LF-MVC-GAM and LF-MVC-LAM are $\mathcal{O} (\frac{1} {\sqrt{n}})$. The upper bound of the empirical error goes into $0$ when $n$ goes into infinity. This clearly verifies the good generalization ability of the proposed algorithms. The detailed proof is provided in the supplemental material due to space limit.

\section{Experiments}\label{experiments}

In this section, we evaluate the effectiveness and efficiency of the proposed LF-MVC-GAM and LF-MVC-LAM for eighteen widely used multiple kernel datasets from the aspects of clustering performance, computational efficiency and convergence.

\subsection{Benchmark Datasets}

The proposed algorithms are experimentally evaluated on twelve widely used MKL benchmark data sets shown in Table \ref{datasets}. They are {AR10P,\footnote{\footnotesize{\texttt{http://featureselection.asu.edu/old/datasets.php}}}} {Oxford Flower17\footnote{\footnotesize{\texttt{http://www.robots.ox.ac.uk/\~{}vgg/data/flowers/}}}}, {Caltech102\footnote{\footnotesize{\texttt{http://files.is.tue.mpg.de/pgehler/projects/iccv09/}}}}, {Columbia Consumer Video (CCV)\footnote{\footnotesize{\texttt{http://www.ee.columbia.edu/ln/dvmm/CCV/}}}}, {YALE Face\footnote{\footnotesize{\texttt{www.cs.yale.edu/cvc/projects/yalefaces/yalefaces.html}}}}, {Plant, Mfeat\footnote{\footnotesize{\texttt{http://mkl.ucsd.edu/dataset/}}}} and {Caltech102\footnote{\footnotesize{\texttt{http://www.vision.caltech.edu/archive.html}}}}. For these datasets, all kernel matrices are pre-computed and can be publicly downloaded from the above websites. Further, followed by \cite{liu2018late}, Caltech102-$5$ means the number of samples belonging to each cluster is $5$, and so on. All the datasets can be available at our repository.

Moreover, we have also conducted experiments on several multi-view datasets including MSRC-v1, Caltech102-7, Caltech101-20 and ALOI. MSRC-v1 \footnote{\footnotesize{\texttt{https://www.microsoft.com/en-us/research/project/\\image-understanding}}} contains 210 images of 7 classes which are tree, building, airplane, cow, face, car, and bicycle. Four kinds of features are adopted in the test. They are namely 24-dimension CM feature, 576-dimension HOG feature, 256-dimension LBP feature and 254-dimension CENT feature. Then, Caltech101 \footnote{\footnotesize{\texttt{http://www.vision.caltech.edu/Image\_Datasets/\\ Caltech101/}}} contains 101 categories of images, where most categories have about 50 images. Following the previous work [65], two subsets, namely Caltech101-7 and Caltech101-20, are selected for the experiment. Caltech101-7 consists of a total of 1474 images belonging to 7 different object categories while Caltech101-20 consists of a total of 2386 images belonging to 20 different object categories. Three features are adopted to generate 3 views. They are 512-dimension GIST feature, 1984-dimension HOG feature and 928-dimension LBP feature. Another dataset is Amsterdam Library of Object Images (ALOI) \footnote{\footnotesize{\texttt{http://elki.dbs.ifi.lmu.de/wiki/DataSets/MultiView}}}. It is an image dataset consisting of 110250 images of 1000 small objects, taken under various light conditions and rotation angles. We select a subset containing 8640 images of 80 small objects for the experiments,including 125-dimension RGB feature, 77-dimension colorsim feature and 13-dimension haralick feature to generate 3 views. Additionally, we also conduct comprehensive experiments on large-scale multiple kernel datasets. For the MNIST dataset, it is a classic handwritten digits dataset with 60,000 samples. To construct multi-view description for the samples,
we adopt 3 classic ImageNet pre-trained deep neural networks, i.e. VGG, DenseNet121 and ResNet101 to extract features. With the extracted features, we finally construct 3 linear kernels for the dataset.

\begin{table}[!t]
\begin{center}
{
\caption{{Mutiple Kernel Datasets used in our experiments.}}\label{datasets}
\begin{tabular}{c||c|c|c}
\toprule
Dataset         & \#Samples   & \#Views & \#Classes\\
\midrule
\hline
AR10P           & $130$     & $6$          & $10$\\
\hline
YALE            &  $165$    & $5$           & $15$\\
\hline
Plant           & $940$     & $69$          & $4$\\
\hline
Mfeat          & $2000$     & $12$          & $10$\\
\hline
CCV             &  $6773$    & $3$           & $20$\\
\hline
Flower17        &  $1360$    & $7$           & $17$\\
\hline
Caltech102-5    &  $510$     & $48$          & $102$\\
\hline
Caltech102-10   &  $1020$    & $48$          & $102$\\
\hline
Caltech102-15   &  $1530$    & $48$          & $102$\\
\hline
Caltech102-20   &  $2040$    & $48$          & $102$\\
\hline
Caltech102-25   &  $2550$    & $48$          & $102$\\
\hline
Caltech102-30   &  $3060$    & $48$          & $102$\\
\hline
MSRC-V1         &  $210$     &  $4$     &  $7$ \\
\hline
Caltech-7        &  $1474$     &  $3$     &  $7$ \\
\hline
Caltech-20        &  $2386$     &  $3$     &  $20$ \\
\hline
ALOI         &  $8640$     &  $3$     &  $80$ \\
\bottomrule
\end{tabular}
}
\end{center}
\end{table}

\begin{table}[]
\begin{center}
{
\caption{Large-scale Mutiple Kernel Datasets used in our experiments.}\label{datasets2}
\begin{tabular}{c||c|c|c|c}
\midrule
Dataset & Samples & Kernels & Clusters & Storage Size \\
\hline
Reuters  & 18758   & 5       & 6        & 12GB         \\
\hline
Mnist    & 60000   & 3       & 10       & 40GB  \\
\hline       
\end{tabular}
}
\end{center}
\end{table}

\begin{table*}[!htbp]
\begin{center}
{
\centering
\caption{ACC comparison of different multiple kernel clustering algorithms on twelve benchmark data sets.The best result is highlighted with underlines. Boldface means no statistical difference from the best one and '-' means the out-of-memory failure.}
\label{ACC result}
\resizebox{\textwidth}{!}{
\begin{tabular}{|c|c|c|c|c|c|c|c|c|c|c|c|}
\hline
\multirow{2}{*}{Datasets} & \multirow{2}{*}{A-MKKM} & \multirow{2}{*}{SB-KKM} & {MKKM} & {CRSC} & {RMKKM} & {RMSC} & {LMKKM} & {MKKM-MR} & {LKAM}   & LF-MVC-GAM & LF-MVC-LAM  \\
\cline{4-10,11-12}
 &  &  &\cite{huang2012multiple}  &\cite{kumar2011co}  &\cite{du2015robust}  &\cite{xia2014robust}  &\cite{gonen2014localized}  &\cite{Liu2016Multiple}  &\cite{Li2016Multiple}  & \multicolumn{2}{c|}{Proposed} \\
\midrule
\multicolumn{12}{|c|}{ACC$(\%)$}    \\ \hline
AR10P & 38.46  & 43.08  & 40.00  & 32.31  & 30.77  & 30.77  & 40.77  & 39.23  & 27.69  & 43.85 & \underline{$\mathbf{53.08}$}   \\ \hline
YALE & 52.12  & 56.97  & 52.12  & 52.36  & 56.36  & 58.03  & 53.33  & 58.00  & 46.67  & $\mathbf{58.55}$  & \underline{$\mathbf{60.61}$}     \\ \hline
Plant & 60.21  & 51.91  & 56.38  & 60.21  & 55.00  & 53.62  & - & 52.55  & 50.32  & 62.66  & \underline{$\mathbf{64.79}$}    \\ \hline
Mfeat & 94.20  & 86.00  & 66.75  & 77.19  & 73.70  & 94.60  & 94.90  & 92.55  & \underline{$\mathbf{96.65}$}      & $\mathbf{95.80}$  & $\mathbf{95.90}$  \\ \hline
Flower17 & 51.03  & 42.06  & 45.37  & 46.02  & 53.38  & 51.10  & 48.97  & 58.82  & 57.87  & $\mathbf{60.29}$ & \underline{$\mathbf{62.35}$}  \\ \hline
CCV & 19.98  & 20.23  & 18.29  & 19.98  & 16.76  & 16.29  & 20.17  & 20.86  & 18.35  & $\mathbf{26.89}$  & \underline{$\mathbf{29.79}$}  \\ \hline
Caltech102-5 & 36.67  & 36.86  & 28.63  & 31.52  & 32.75  & 33.73  & 37.24  & 37.04  & 32.16  & $\mathbf{39.41}$  & \underline{$\mathbf{41.37}$} \\ \hline
Caltech102-10 & 30.88  & 32.71 & 22.75  & 29.00  & 26.67  & 29.80  & 31.96  & 31.73  & 28.33  & $\mathbf{35.00}$  & \underline{$\mathbf{35.59}$}   \\ \hline
Caltech102-15 & 29.11  & 30.44  & 20.39  & 27.73  & 24.90  & 25.49  & - & 32.29  & 27.32  & 33.86 & \underline{$\mathbf{36.61}$} \\ \hline
Caltech102-20 & 28.20  & 29.73  & 18.73  & 27.34  & 24.51  & 23.87  & - & 32.55 & 25.88  & $\mathbf{34.56}$  & \underline{$\mathbf{35.90}$}  \\ \hline
Caltech102-25 & 26.41  & 27.12  & 16.63  & 27.00  & 21.92  & 24.08  & - & 30.12  & 26.16  & $\mathbf{31.96}$  & \underline{$\mathbf{35.61}$}  \\ \hline
Caltech102-30 & 25.91  & 27.29  & 16.31  & 26.51  & 21.41  & 22.58  & - & 30.31  & 24.54  & $\mathbf{32.42}$  & \underline{$\mathbf{34.17}$}    \\ \hline

\end{tabular}}
}
\end{center}
\end{table*}

\begin{table*}[!htbp]
\begin{center}
{
\centering
\caption{NMI comparison of different multiple kernel clustering algorithms on twelve benchmark data sets.The best result is highlighted with underlines. Boldface means no statistical difference from the best one and '-' means the out-of-memory failure.}
\label{NMI result}
\resizebox{\textwidth}{!}{
\begin{tabular}{|c|c|c|c|c|c|c|c|c|c|c|c|}
\hline
\multirow{2}{*}{Datasets} & \multirow{2}{*}{A-MKKM} & \multirow{2}{*}{SB-KKM} & {MKKM} & {CRSC} & {RMKKM} & {RMSC} & {LMKKM} & {MKKM-MR} & {LKAM} & LF-MVC-GAM & LF-MVC-LAM \\
\cline{4-10,11-12}
 &  &  &\cite{huang2012multiple}  &\cite{kumar2011co}  &\cite{du2015robust}  &\cite{xia2014robust}  &\cite{gonen2014localized}  &\cite{Liu2016Multiple}  &\cite{Li2016Multiple}  & \multicolumn{2}{c|}{Proposed} \\
\midrule
\multicolumn{12}{|c|}{NMI$(\%)$}    \\ \hline
AR10P & 37.27  & 42.61  & 39.53  & 33.32  & 26.62  & 27.87  & 41.67  & 40.11  & 24.72  & 44.42  & \underline{$\mathbf{53.11}$}  \\ \hline
YALE & 57.72  & 58.42  & 54.16  & 54.65  & $\mathbf{59.32}$  & 57.58  & 56.60  & 58.87  & 53.51  & $\mathbf{59.86}$  & \underline{$\mathbf{60.50}$}  \\ \hline
Plant & 25.54  & 17.19  & 20.02  & 25.54  & 19.43  & 23.18  & -  & 21.65  & 21.46  & $\mathbf{27.91}$  & \underline{$\mathbf{30.94}$} \\ \hline
Mfeat & 89.83  & 75.79  & 60.84  & 70.16  & 73.05  & 90.64  & 89.68  & 85.90  & \underline{$\mathbf{92.70}$}  & 90.92  & $\mathbf{91.25}$  \\ \hline
Flower17 & 50.19  & 45.14  & 45.35  & 45.69  & 52.56  & 54.39  & 47.79  & 57.05  & 56.06  & \underline{$\mathbf{59.79}$}  & $\mathbf{59.39}$  \\ \hline
CCV & 17.06  & 17.84  & 15.04  & 17.06  & 12.42  & 13.77  & 16.86  & 18.71  & 16.52  & $\mathbf{20.48}$  & \underline{$\mathbf{22.10}$}  \\ \hline
Caltech102-5 & 68.64  & 70.55  & 65.97  & 66.34  & 66.76  & 68.93  & 71.28  & 71.08  & 67.18  & $\mathbf{72.10}$  & \underline{$\mathbf{72.85}$}  \\ \hline
Caltech102-10 & 59.77  & 63.22  & 55.80  & 59.03  & 57.28  & 59.86  & 61.86  & 62.76  & 58.51  & $\mathbf{63.47}$  & \underline{$\mathbf{64.34}$}  \\ \hline
Caltech102-15 & 53.66  & 56.62  & 49.27  & 54.48 & 52.04  & 54.57  & - & 58.25  & 55.20  & $\mathbf{59.73}$  & \underline{$\mathbf{61.05}$}  \\ \hline
Caltech102-20 & 53.19  & 53.55  & 45.61  & 52.01  & 48.66  & 50.34  & - & 55.06  & 51.42  & $\mathbf{57.12}$  & \underline{$\mathbf{57.47}$}  \\ \hline
Caltech102-25 & 49.92  & 50.84  & 41.86  & 49.99  & 45.53  & 48.35  & - & 52.44 & 50.12  & $\mathbf{53.98}$  & \underline{$\mathbf{54.95}$}  \\ \hline
Caltech102-30 & 49.31  & 50.85  & 39.92  & 48.25  & 43.72  & 46.04  & -  & 51.55  & 47.39  & $\mathbf{53.33}$  & \underline{$\mathbf{53.49}$}  \\ \hline

\end{tabular}}
}
\end{center}
\end{table*}

\subsection{Compared MKC Algorithms and Experimental Setting}

In the experiments, LF-MVC-GAM and LF-MVC-LAM are compared with the following state-of-the-art multi-view clustering methods.
\begin{enumerate}
\item {\bf Average Multiple Kernel $k$-means (A-MKKM)}: All kernels are averagely weighted to conduct the optimal kernel, which is used as the input of kernel $k$-means algorithm.
\item {\bf Single Best Kernel $k$-means (SB-KKM)}: Kernel $k$-means is performed on each single kernel and the best result is outputted.
\item {\bf Multiple Kernel $k$-means (MKKM)} \cite{huang2012multiple}: The algorithm alternatively performs kernel $k$-means and updates the kernel coefficients.
\item {\textbf{Co-regularized Spectral Clustering (CRSC)}} \cite{kumar2011co}: CRSC provides a co-regularization way to perform spectral clustering on multiple views.
\item {\bf Robust Multiple Kernel $k$-means using $\ell_{2,1}$ norm (RMKKM)} \cite{du2015robust}: RMKKM simultaneously finds the best clustering label, the cluster membership and the optimal combination of multiple kernels by adding $\ell_{2,1}$ norm.
\item {\bf Robust Multi-view Spectral Clustering (RMSC)} \cite{xia2014robust}: RMSC constructs a transition probability matrix from each single view, and then use recover a shared low-rank transition probability matrix as an input to the standard Markov chain for clustering.
\item {\bf Localized Multiple Kernel $k$-means (LMKKM)} \cite{gonen2014localized}: LMMKM combines the base kernels by sample-adaptive weights.
\item {\bf Multiple kernel $k$-means with matrix-induced regularization (MKKM-MR)} \cite{Liu2016Multiple}: The algorithm applies the multiple kernel $k$-means clustering with a matrix-induced regularization to reduce the redundancy and enhance the diversity of the kernels.
\item {\bf Multiple kernel clustering with local kernel alignment maximization (LKAM)} \cite{Li2016Multiple}: The algorithm maximizes the local kernel with multiple kernel clustering and focuses on closer sample pairs that they shall stay together.
\item {\bf Latent multi-view subspace clustering (LMSC)}\cite{zhang2017latent}:LMSC learns latent data representation and simultaneously explores the underlying complementary information among multiple views in a self-representation learning framework.
\item {\bf Graph learning for multiview clustering (MVGL)}\cite{zhan2017graph}:MVGL is a two-step multi-view clustering method which fuses the view-specific proximity matrix into the consensus proximity similarity matrix .
\item {\bf Graph-based multi-view clustering (GMC)}\cite{wang2019gmc}: GMC adaptively produces the consensus proximity similarity matrix by fusing multiple grpahs based on original graphs.
\item {\bf Multi-view spectral clustering via multi-view weighted consensus and matrix-decomposition based discretization(MvWCMD)}\cite{chen2019multi}:MvWCMD performs the consensus proximity matrix learning and
discrete cluster label learning simultaneously.
\end{enumerate}

\begin{table*}[]
\caption{ACC comparison of different Graph-based multi-view clustering algorithms on twelve benchmark data sets.The best result is highlighted with format mean(std).}
\label{mgc acc result}
\resizebox{\textwidth}{!}{
\begin{tabular}{c|cccccccc}
\midrule
Dataset       & SC            & Co-reg        & LMSC          & MVGL          & MvWCMD        & GMC          & LF-MVC-GAM & LF-MVC-LAM               \\
\hline
MSRC-v1       & 0.699 (0.007) & 0.543 (0.046) & 0.678 (0.062) & 0.719 (0.000) & 0.722 (0.035) & 0.748 (0.000) & 0.832 (0.000) & \textbf{0.843 (0.000)} \\
Caltech101-7  & 0.489 (0.001) & 0.586 (0.031) & 0.562 (0.038) & 0.569 (0.000) & 0.717 (0.063) & 0.804 (0.000) & 0.820 (0.000) & \textbf{0.851 (0.000)} \\
Caltech101-20 & 0.407 (0.006) & 0.497 (0.031) & 0.503 (0.008) & 0.634 (0.000) & 0.568 (0.048) & 0.673 (0.000) & 0.744 (0.000) & \textbf{0.774 (0.000)} \\
ALOI          & 0.627 (0.017) & 0.623 (0.017) & 0.691 (0.003) & 0.631 (0.000) & 0.520 (0.006) & 0.664 (0.000) & 0.659 (0.000) & \textbf{0.704 (0.000)} \\ 
\hline       
\end{tabular}}
\end{table*}

\begin{table*}[]
\caption{NMI comparison of different Graph-based multi-view clustering algorithms on twelve benchmark data sets.The best result is highlighted with format mean(std).}
\label{mgc nmi result}
\resizebox{\textwidth}{!}{
\begin{tabular}{c|cccccccc}
\midrule
Dataset       & SC            & Co-reg        & LMSC          & MVGL          & MvWCMD        & GMC           & LF-MVC-GAM & LF-MVC-LAM        \\
\hline
MSRC-v1       & 0.552 (0.010) & 0.445 (0.042) & 0.575 (0.047) & 0.635 (0.000) & 0.690 (0.026) & 0.693 (0.000) & 0.729 (0.000) & \textbf{0.773 (0.000)} \\
Caltech101-7  & 0.347 (0.000) & 0.489 (0.011) & 0.480 (0.023) & 0.489 (0.000) & 0.552 (0.017) & 0.647 (0.000) & 0.528 (0.000) & \textbf{0.781 (0.000)} \\
Caltech101-20 & 0.497 (0.006) & 0.531 (0.007) & 0.568 (0.004) & 0.595 (0.000) & 0.557 (0.043) & 0.567 (0.000) & 0.584 (0.000) & \textbf{0.621 (0.000)} \\
ALOI          & 0.779 (0.004) & 0.775 (0.010) & 0.808 (0.006) & 0.680 (0.000) & 0.626 (0.002) & 0.728 (0.000) & 0.771 (0.000) & \textbf{0.790 (0.000)} \\       
\hline
\end{tabular}}
\end{table*}

\begin{table*}[]
\caption{Time-consuming comparison (in seconds) of different Graph-based multi-view clustering algorithms on benchmarks.}
\resizebox{\textwidth}{!}{
\label{mgc time result}
\begin{tabular}{c|cccccccccccccc}
\midrule
              & \multicolumn{2}{c}{Co-reg} & \multicolumn{2}{c}{LMSC} & \multicolumn{2}{c}{MVGL} & \multicolumn{2}{c}{MvWCMD} & \multicolumn{2}{c}{GMC} & \multicolumn{2}{c}{LF-MVC-GAM}     & \multicolumn{2}{c}{LF-MVC-LAM}  \\
\hline
Dataset       & Time        & Speed Up     & Time       & Speed Up    & Time        & Speed Up   & Time        & Speed Up     & Time       & Speed Up   & Time             & Speed Up     & Time            & Speed Up    \\
\hline
MSRC-v1       & 0.572       & 4$\times$            & 4.513      & 0.49$\times$        & 2.206       & 1$\times$          & 2.194       & 0.99$\times$         & 2.655      & 0.83$\times$      & \textbf{0.208}   & \textbf{11$\times$}  & \textbf{0.266}  & \textbf{8$\times$}  \\
Caltech101-7  & 6.395       & 16$\times$           & 326.244    & 0.32$\times$        & 105.236     & 1$\times$          & 173.2       & 1.65$\times$         & 7.414      & 14.19$\times$      & \textbf{2.489}   & \textbf{42$\times$}  & \textbf{3.503}  & \textbf{30$\times$} \\
Caltech101-20 & 18.949      & 44$\times$           & 324.744    & 2.58$\times$        & 837.33      & 1$\times$ & 1375.2      & 1.64$\times$         & 32.031     & 26.14$\times$      & \textbf{5.877}   & \textbf{142$\times$} & \textbf{12.412} & \textbf{67$\times$} \\
ALOI          & 626.846     & 33$\times$           & 5144834    & 0.004$\times$        & 20885.28    & 1$\times$ & 6639.48     & 0.32$\times$         & 1027.62    & 20.32$\times$      & \textbf{350.362} & \textbf{60$\times$}  & \textbf{406.74} & \textbf{51$\times$}\\
\hline
\end{tabular}}
\end{table*}

For all the above mentioned algorithms, we have downloaded their public Matlab code implementations from original websites\footnote{\footnotesize{\texttt{https://github.com/xinwangliu}}}. Our Matlab codes for LF-MVC-GAM and LF-MVC-LAM are available at here \footnote{\footnotesize{\texttt{https://github.com/wangsiwei2010/\\ latefusionalignment}}}.

The hyper-parameters are set according to the suggestions of the corresponding literature. For the proposed algorithms  LF-MVC-GAM and LF-MVC-LAM, the trade-off parameter $\lambda$ is chosen from $\left[2^{-5}, 2^{-4}, \cdots, 2^{5}\right] $ by grid search. Especially, to all the compared spectral clustering algorithms, the optimal neighbor numbers are carefully searched in the range of $[0.1,0.2,\cdots,0.9]*n$ , where $n$ is the sample number in restive dataset. In all our experiments, all base kernels are first centered and then normalized so that for all sample $x_i$ and each view $p$, we have $K_p(x_i,x_i) = 1$ by following {\cite{cortes2012algorithms}}. For all data sets, it is assumed that the true number of clusters is known in advance and set as the true number of classes. The widely used clustering accuracy (ACC), normalized mutual information (NMI) and purity are applied to evaluate the clustering performance. For all algorithms, we repeat each experiment for $50$ times with random initialization to reduce the affect of randomness caused by $k$-means, and report the best result. All our experiments are conducted on a desktop computer with a 2.5GHz Intel Platinum 8269CY CPU and 48GB RAM, MATLAB 2019b (64bit).

\begin{table*}[!htbp]
\begin{center}
\caption{Clustering performance comparison between the state-of-the-art algorithms on large-scale datasets. In this table, ACC, NMI, purity of different clustering algorithms benchmark datasets are reported and '-' means the out-of-memory failure.  }\label{Table large result}
\resizebox{1\textwidth}{!}{
\begin{tabular}{c|c|c|c|c|c|c|c}
\hline
{Datasets}  & {A-MKKM} & {SB-KKM} & {CRSC} & {MKAM} & {MKKM-MR} & {LF-MVC-GAM } & {LF-MVC-LAM }  \\
\hline
\hline
\multicolumn{8}{c}{ACC ($\%$)}\\
\hline
\hline
Reuters  &45.00	&47.24	&-	&42.37	&45.70	&49.16	&\textbf{51.24}	 \\
\hline
MNIST	  &77.33	&77.89	&-	    &-	    &-	    &80.58	&\textbf{82.85}	 \\
\hline
\hline
\multicolumn{8}{c}{NMI(\%) }\\
\hline
\hline
Reuters  &27.32	&26.42	&-	&24.56	&27.64	&29.62	&\textbf{31.38}	\\
\hline
MNIST	  &74.28	&76.50  &-	    &-	    &-	    &78.47	&\textbf{80.87}	\\
\hline
\hline
\multicolumn{8}{c}{Purity ($\%$)}\\
\hline
\hline
Reuters   &65.48	&67.36	&-	&62.36	&65.75	&68.32	&\textbf{70.81}	\\
\hline
MNIST	   &76.53	&74.63	&-	    &-	    &-	    &82.65	&\textbf{85.86}	\\
\hline
\end{tabular}}
\end{center}
\end{table*}

\subsection{Experiments Results}

Table \ref{ACC result} presents the ACC comparison of the above algorithms on the twelve benchmark datasets. The best result is highlighted with underlines. Boldface means no statistical difference from the best one and '-' means the out-of-memory failure. Based on the results, we have the following observations:
\begin{itemize}
\item  LF-MVC-GAM and LF-MVC-LAM show clear advantages over other multi-view clustering baselines, with 10 best and 1 second best results out of the total 12 data sets; in particular, the margins for the three data sets: AR10P, Plant and CCV are very impressive.  These results verify the effectiveness of the proposed late fusion alignment.
\item Comparing with the LKAM(\cite{Li2016Multiple}), the proposed LF-MVC-LAM consistently further improves the clustering performance and achieves better results among the benchmark datasets. Both of them adopt the similarity measure in a local way. The clustering results clearly demonstrate the effectiveness of fusing multi-view information in partition level.
\item As can be seen, the regularization term constraints the consensus partition $\mathbf{F}$ in our algorithms to approximate the average partition. However, as the results show, our proposed algorithms consistently outperform the average kernel $k$-means. These verify that the late fusion alignment part can excellently  guide the learning of $\mathbf{F}$ and achieve much better performance than A-MKKM.
\end{itemize}

We also report the NMI in Table \ref{NMI result}. Again, we observe that the proposed algorithm significantly outperforms other early-fusion multi-kernel algorithms. These results are consistent with our observations in Table \ref{ACC result}. Results of purity are shown in the appendix due to the space limit.

In summary, the above experimental results have well demonstrated the effectiveness of our proposed LF-MVC-GAM and LF-MVC-LAM comparing to other state-of-the-art methods. We attribute the superiority of proposed algorithms as two aspects:
\romannumeral1) Both LF-MVC-GAM and LF-MVC-LAM employ joint fusion to update weighted basic partitions and the consensus one and get slightly higher performance than other methods. To be specific, when a higher quality consensus partition is obtained, we could further make full use of the high-quality partition to guide the weighted basics ones and improve the performance.
\romannumeral2) Compared with the existing early-fusion methods, the proposed LF-MVC-GAM and LF-MVC-LAM fuse multiple kernel information in the partition level, which demonstrates the benefits of fusing high-level information. These two factors contribute to significant improvements on clustering performance.

\begin{figure}[!htbp]
\centering
\subfloat{\includegraphics[width = 0.5\textwidth]{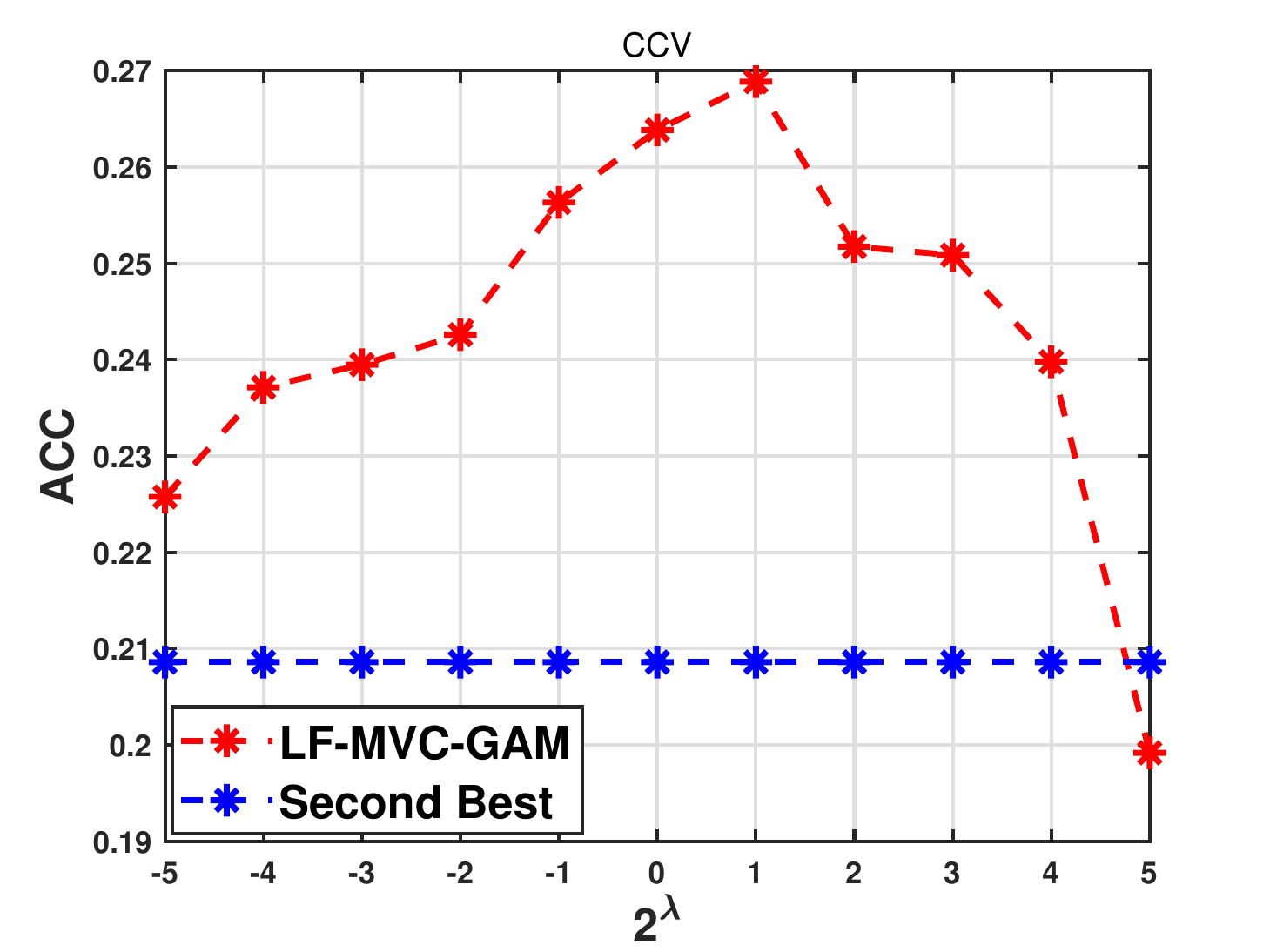}\label{ccv_mvc_lfa_sentivity}}%
\subfloat{\includegraphics[width = 0.5\textwidth]{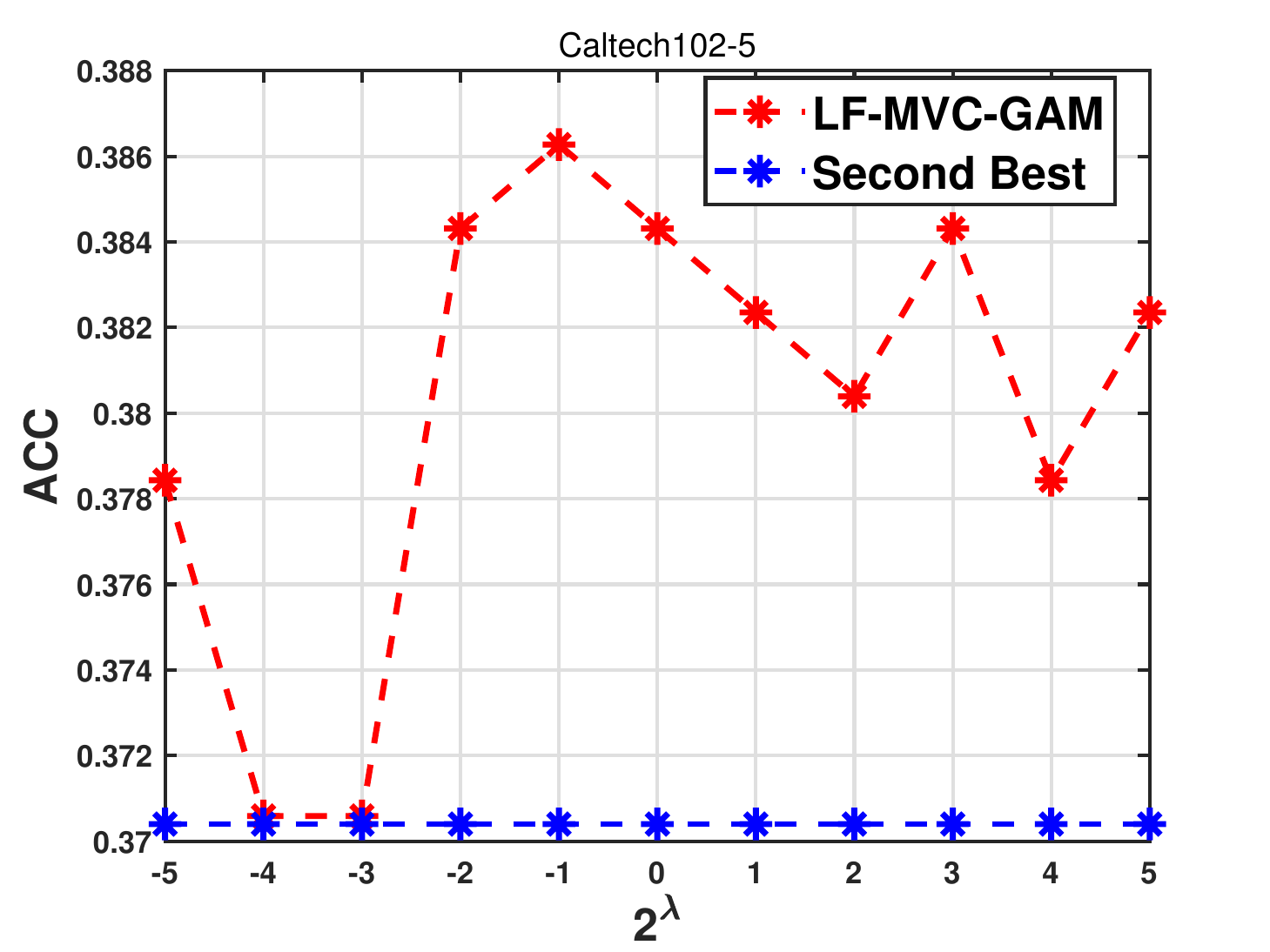}\label{caltech5_mvc_lfa_sentivity}}\\
\caption{The sensitivity of LF-MVC-GAM with the variation of $\lambda$ on CCV and Caltech102-5. More experimental results can be found in the appendix.}\label{MVC_LFA_SensitivityFig}
\end{figure}

\begin{figure*}[!htbp]
\centering
\subfloat[Mfeat $1^{st}$ iteration]{\includegraphics[width = 0.25\textwidth]{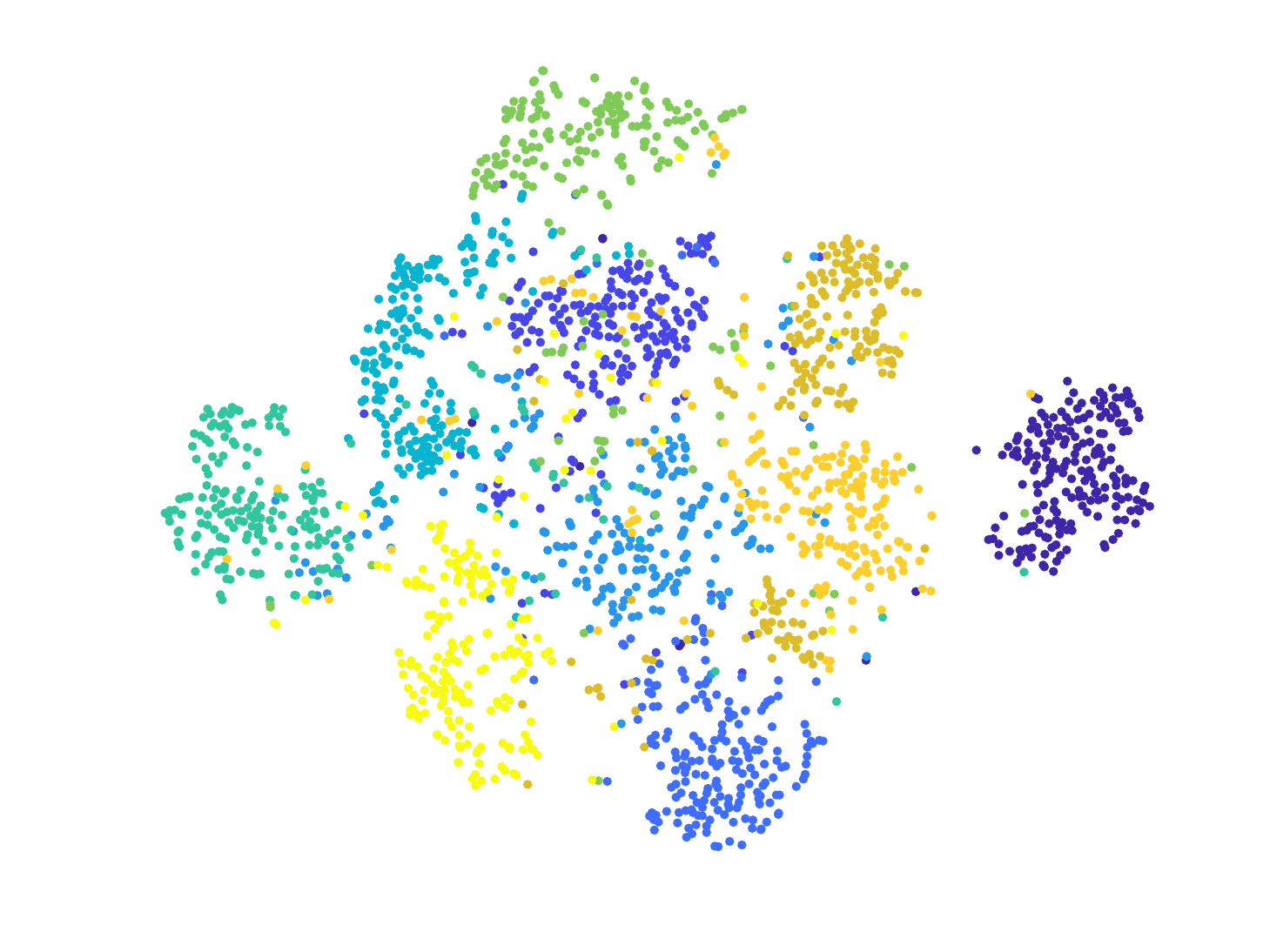}\label{evolution_1}}
\subfloat[Mfeat $5^{th}$ iteration]{\includegraphics[width = 0.25\textwidth]{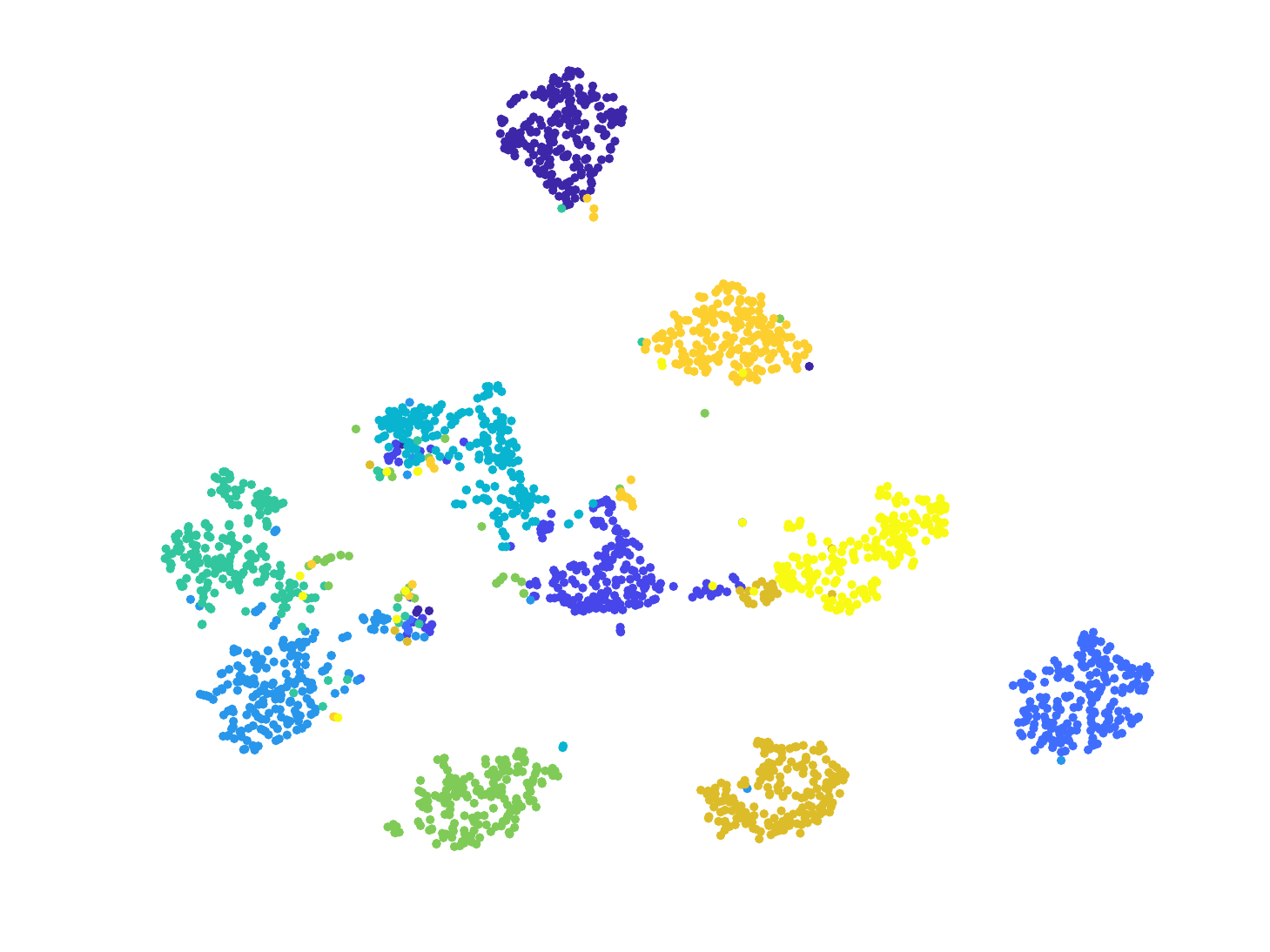}\label{evolution_5}}
\subfloat[Mfeat $10^{th}$ iteration]{\includegraphics[width = 0.25\textwidth]{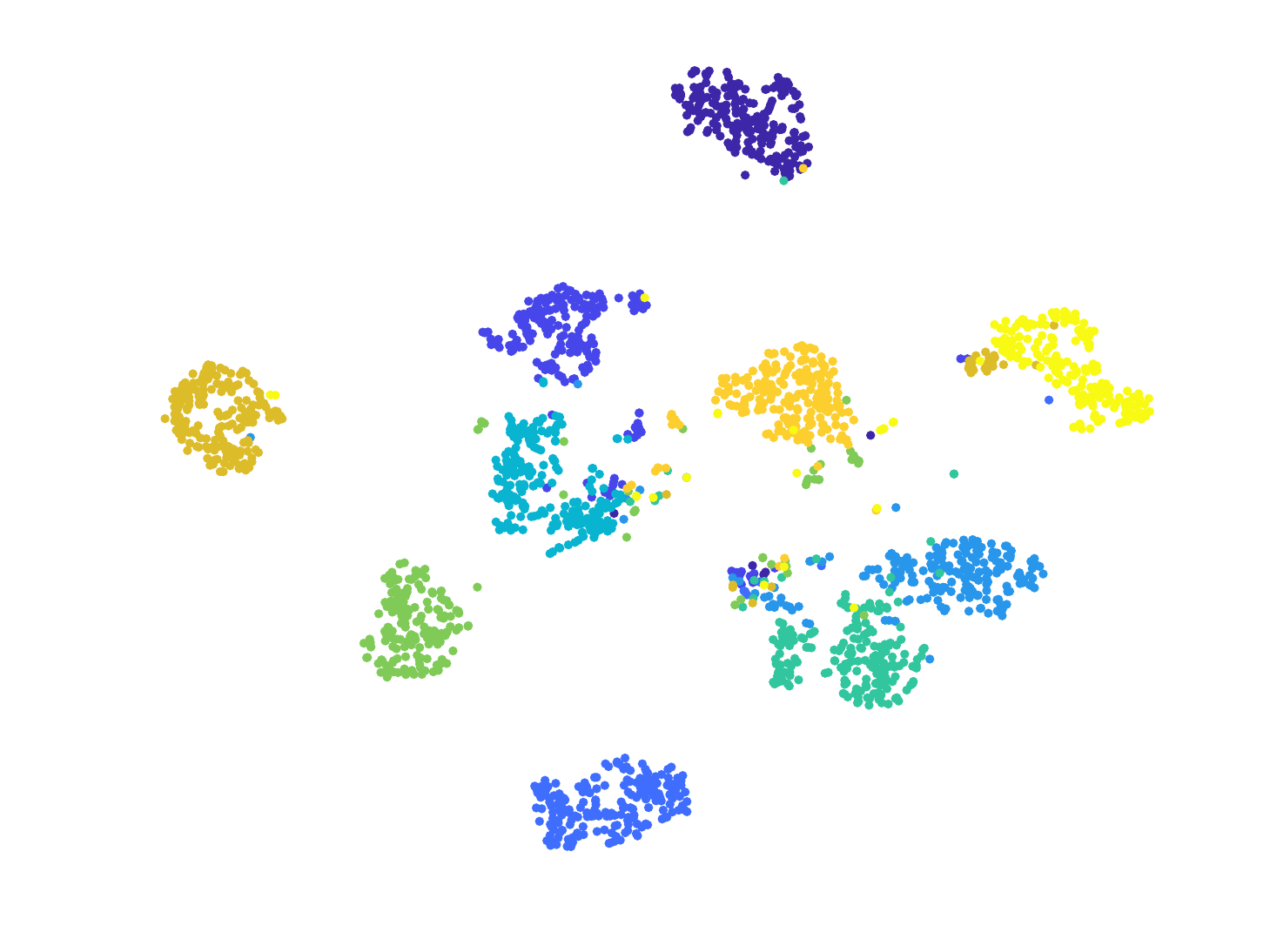}\label{evolution_10}}
\subfloat[Mfeat $20^{th}$ iteration]{\includegraphics[width = 0.25\textwidth]{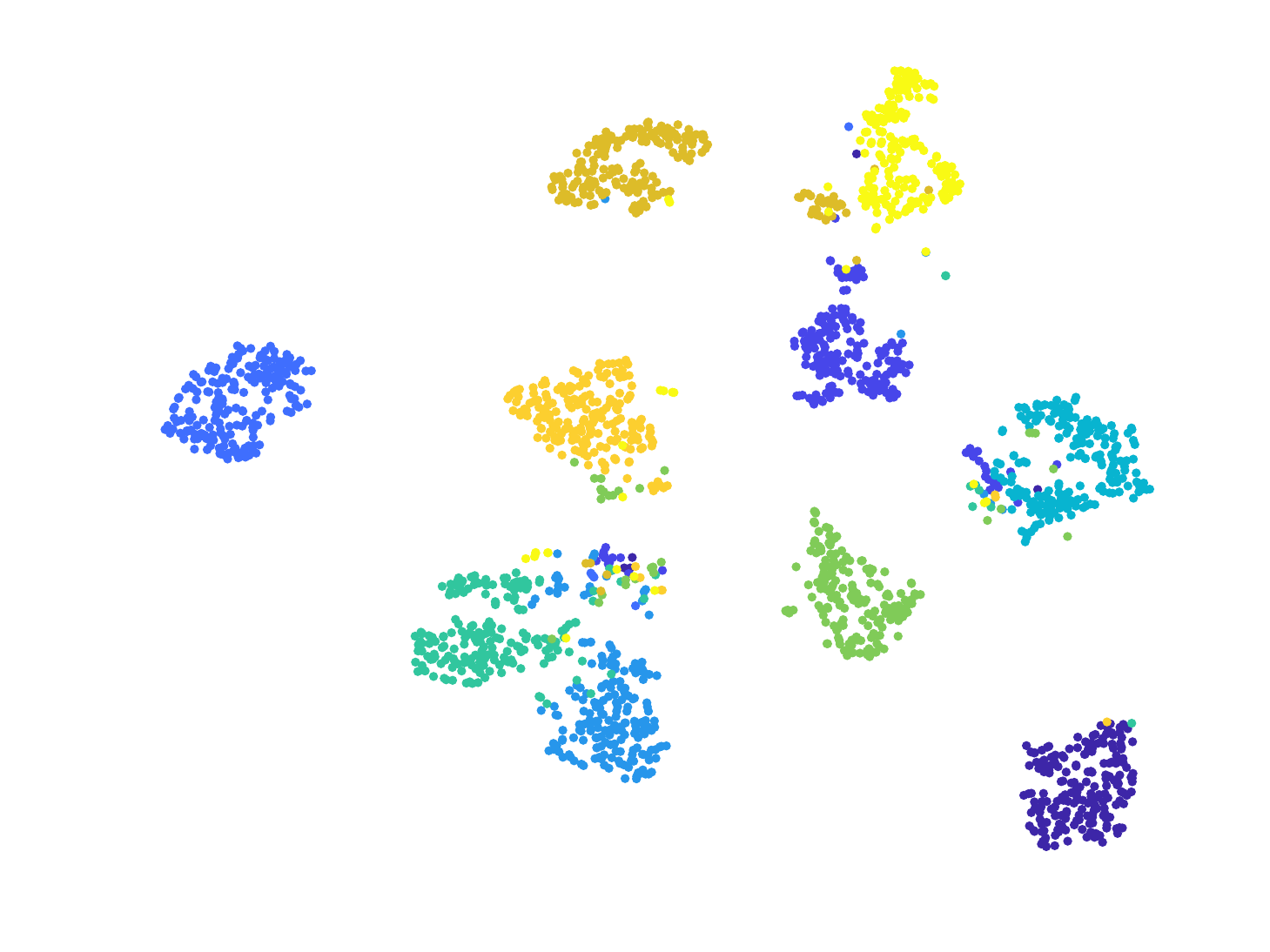}\label{evolution_20}}\\[-10pt]
\subfloat[Plant $1^{st}$ iteration]{\includegraphics[width = 0.25\textwidth]{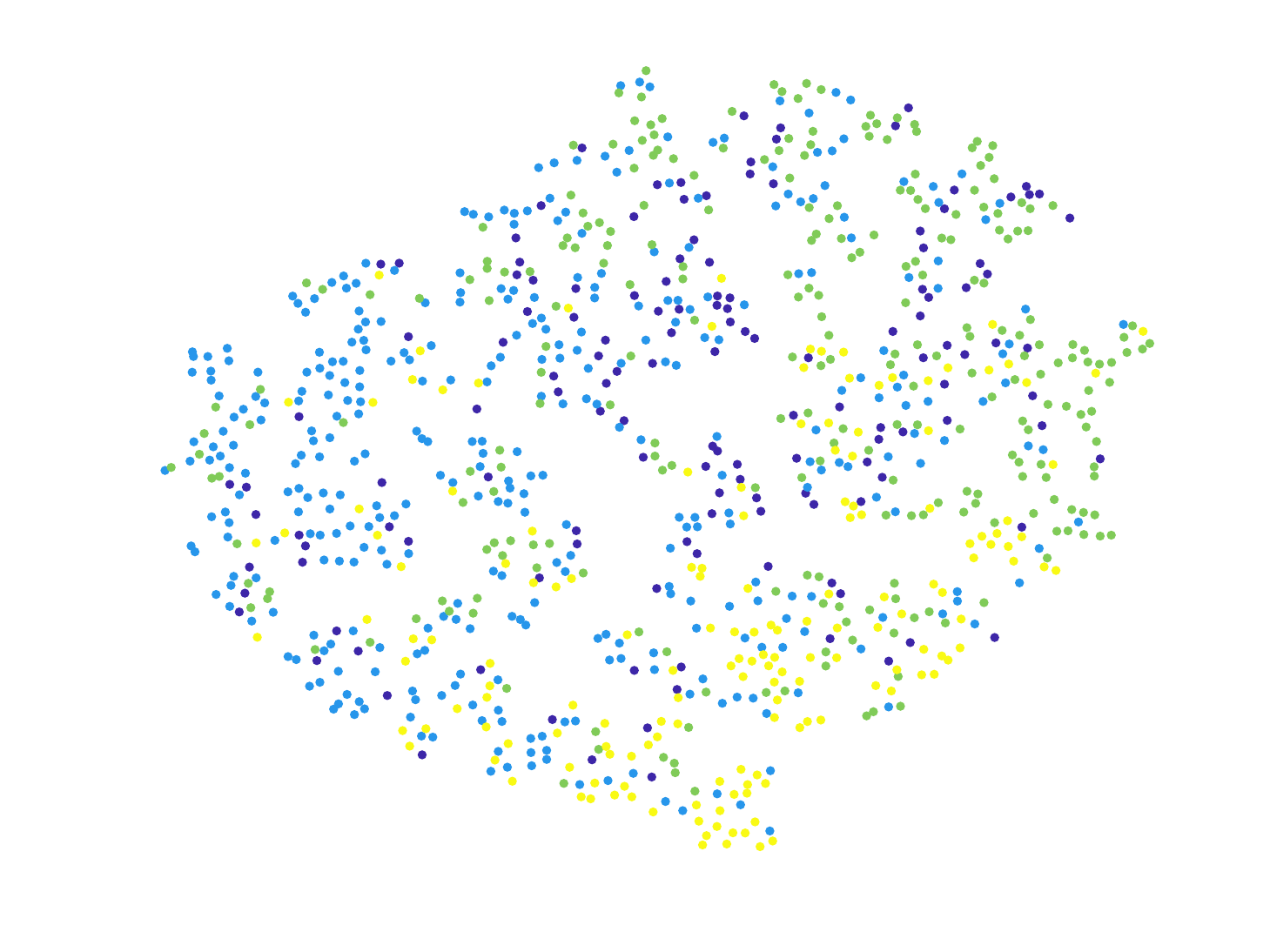}\label{evolution_20}}
\subfloat[Plant $5^{th}$ iteration]{\includegraphics[width = 0.25\textwidth]{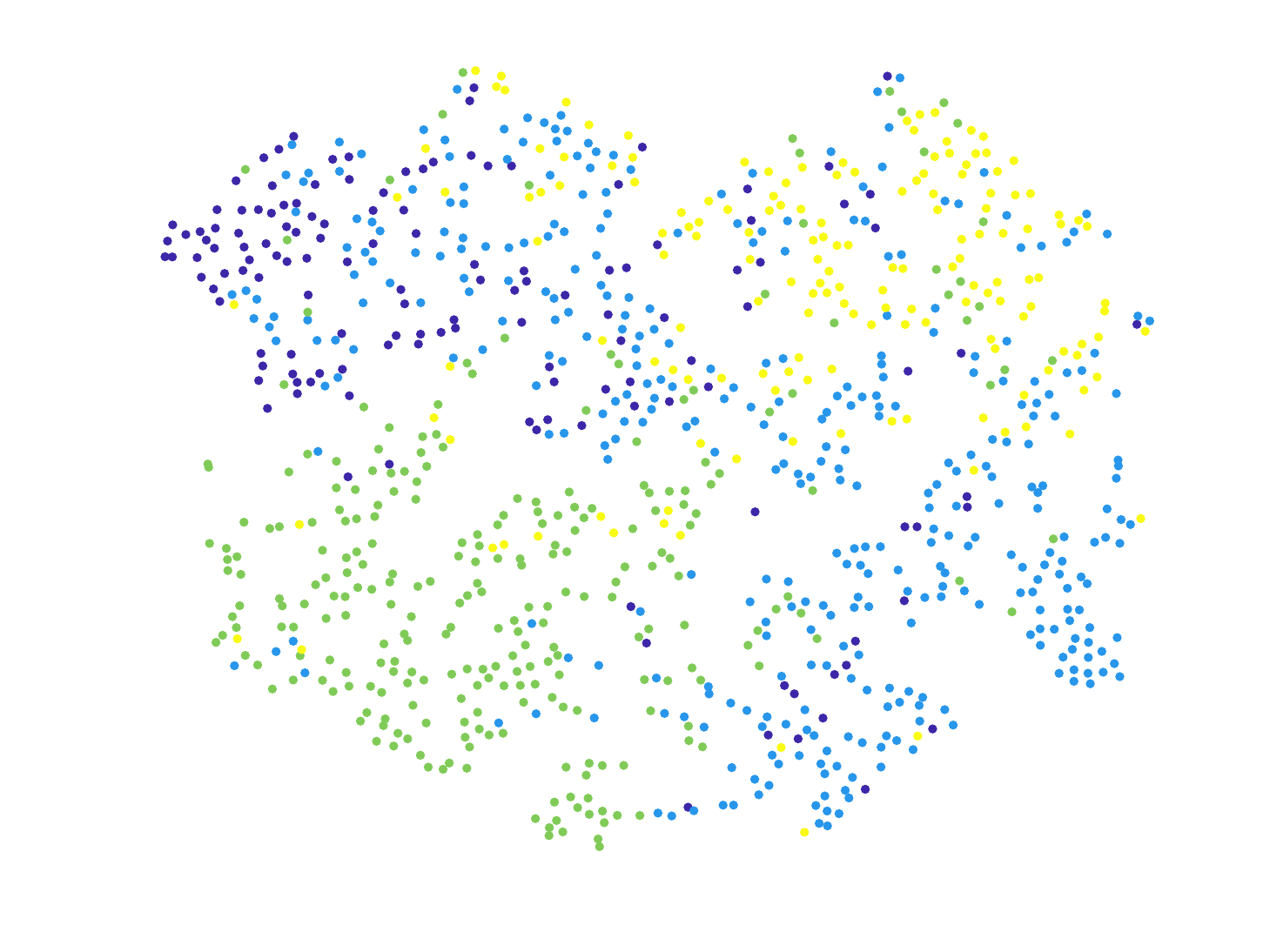}\label{evolution_20}}
\subfloat[Plant $10^{th}$ iteration]{\includegraphics[width = 0.25\textwidth]{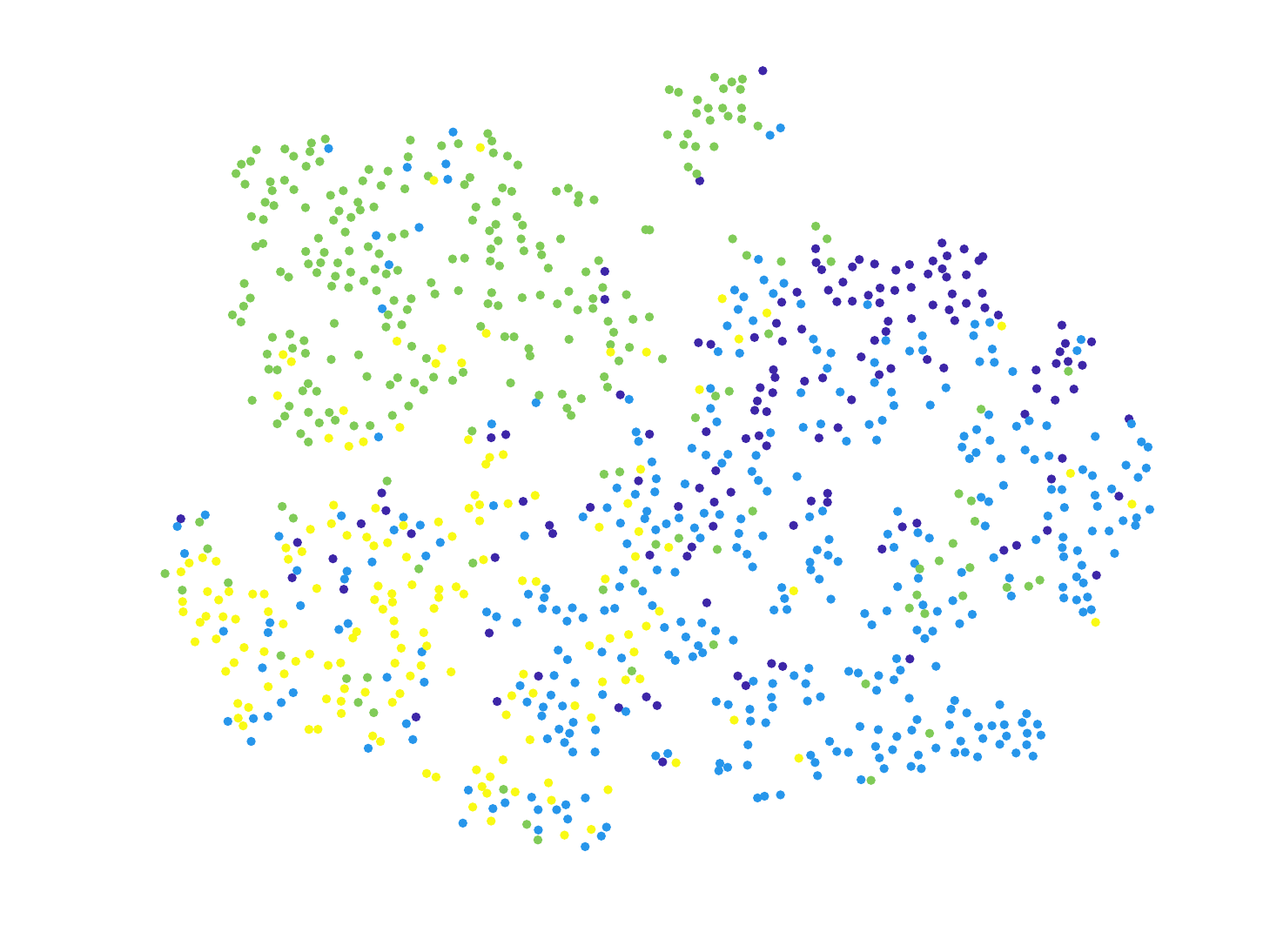}\label{evolution_20}}
\subfloat[Plant $20^{th}$ iteration]{\includegraphics[width = 0.25\textwidth]{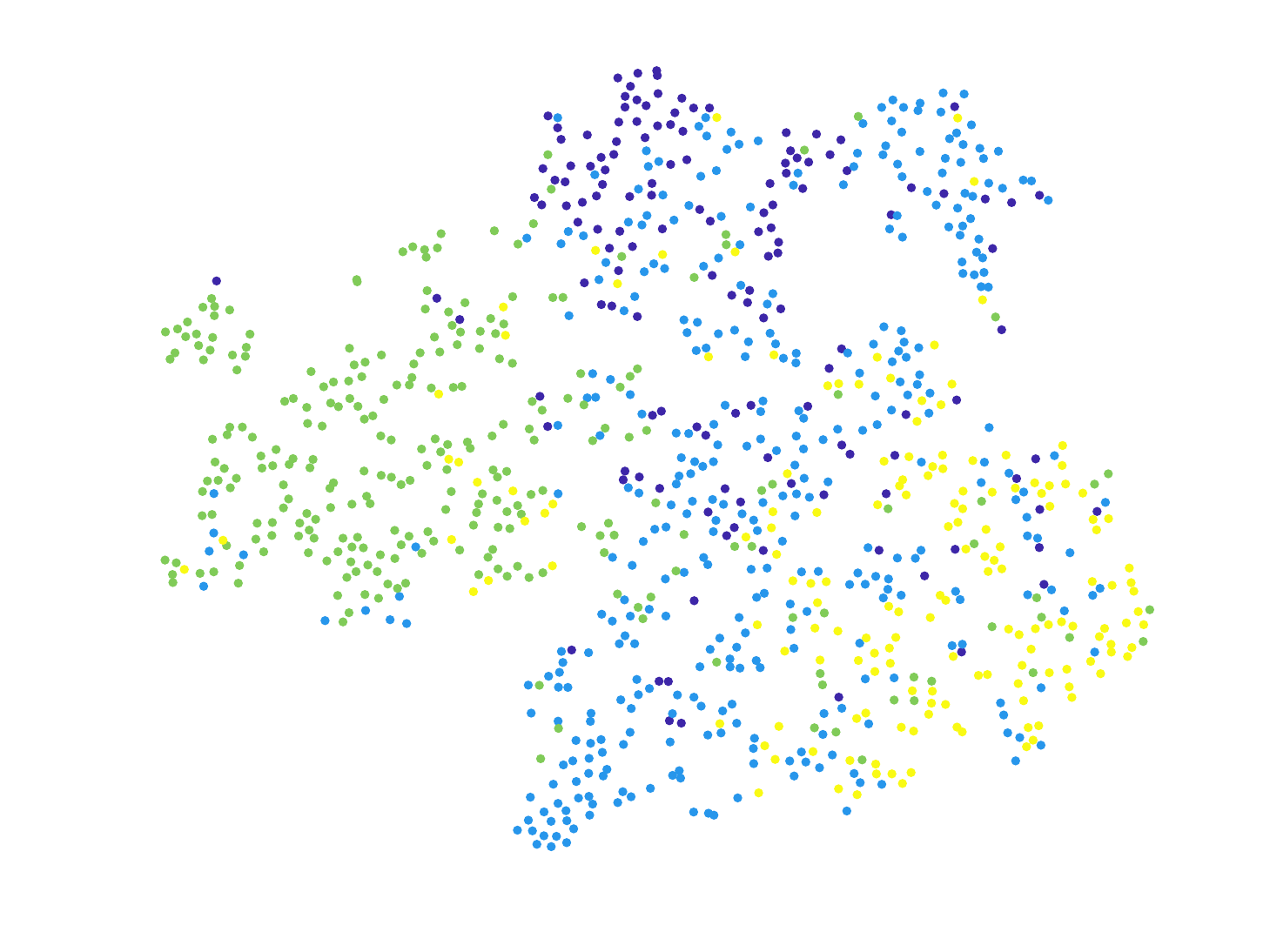}\label{evolution_20}}
\caption{Clustering structure visualization through the t-SNE algorithm \cite{maaten2008visualizing}.In these figures, the first and the second rows are corresponding to the performance evolution of  the $1^{st}$,$5^{th}$,$10^{th}$ and $20^{th}$ iteration on the Mfeat and Plant datasets respectively.}\label{evolution}
\end{figure*}

\subsection{Running Time Comparison}

\begin{figure}[!htbp]
\begin{center}
{
\centering
\subfloat{{\includegraphics[width=0.75\textwidth]{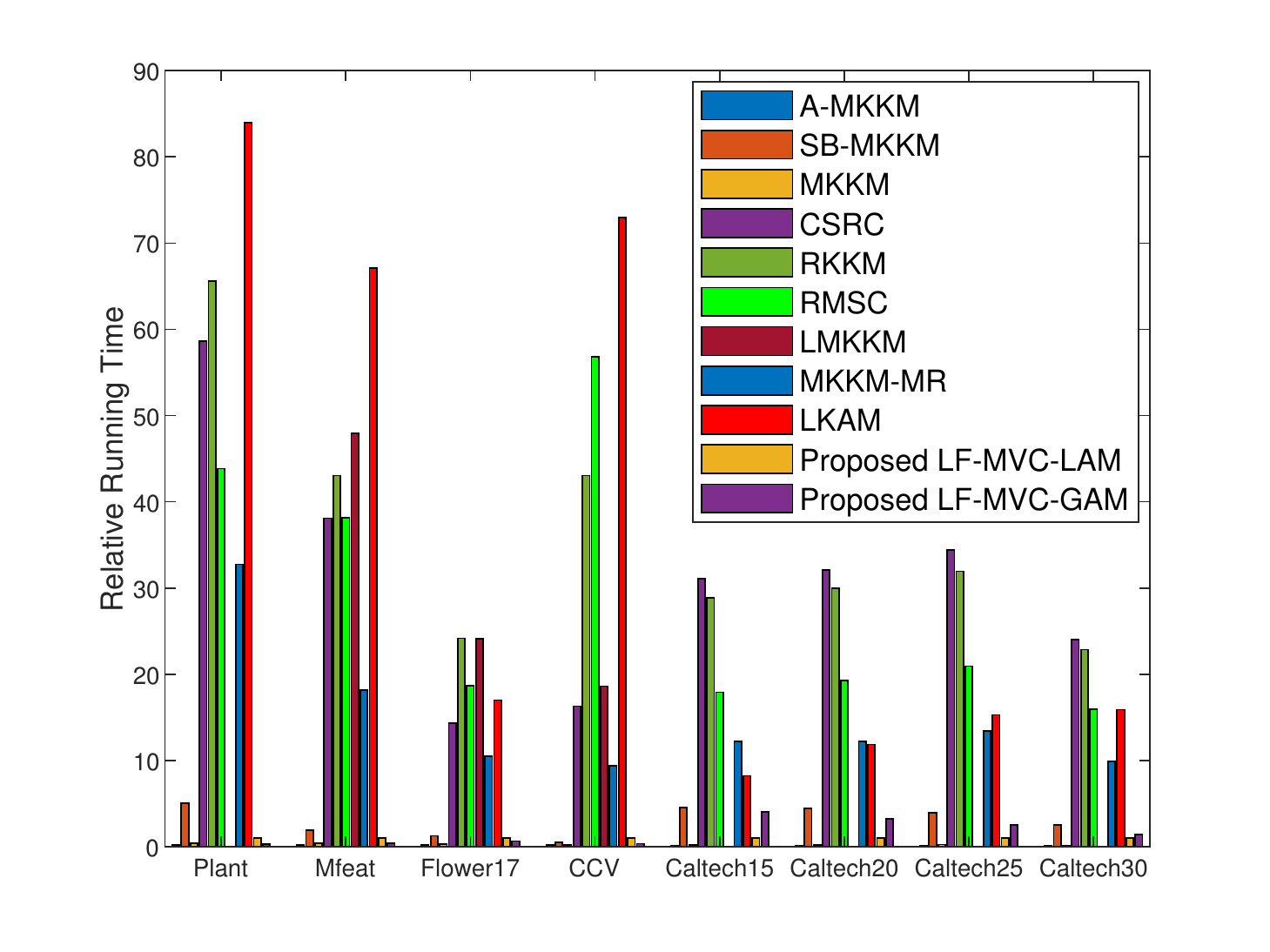}} }
\caption{ The relative running time of the compared algorithms on the benchmark datasets. More results are omitted due to space limit and are provided in supplementary materials.}\label{figure_time1}
}
\end{center}
\end{figure}

To compare the computational efficiency of the proposed algorithms, we record the running time of various algorithms on these benchmark datasets and report them in Figure \ref{figure_time1} and Table \ref{mgc time result}. To avoid the different sizes of magnitude for the running time, we set the the running time of proper algorithm as 1 in the Figure \ref{figure_time1}. The relative running time means the speed up times for various compared algorithms in the same dataset, the same as Table \ref{mgc time result}. As can be seen, LF-MVC-GAM and LF-MVC-LAM  have much shorter running time on all datasets comparing to the-state-of-art multi-view methods (MKKM, CSRC, RKKM, LMKKM, MKKM-MR, LKAM), demonstrating the computational efficiency of the proposed method. As theoretically demonstrated, LF-MVC-GAM and LF-MVC-LAM reduce the time complexity from $\mathcal{O}(n)^3$ to $\mathcal{O}(n)$ per iteration and avoid complicated optimization procedure.  In sum, both the theoretical and the experimental results have well demonstrated the computational advantage of proposed algorithms, making them efficient to handle with multi-view clustering. 

\subsection{Handling with large-scale datasets}
To further demonstrate the effectiveness and efficiency of our proposed method, we evaluate the clustering performance when facing with large-scale datasets. The experimental results are shown in Table \ref{Table large result} and appendix.

From these results, we have the following observations:
\begin{enumerate}
\item LF-MVC-GAM and LF-MVC-LAM consistently outperform other multiple-kernel clustering competitors by a large margin, in terms of ACC, NMI and Purity. These results clearly show the advantages over late fusion framework.
\item In terms of computational efficiency, our proposed methods enjoy significant acceleration in time consuming. The experimental results have well demonstrated the computational advantage of proposed algorithms, making them efficient to handle with large-scale datasets.
\end{enumerate}

\subsection{Visualization of the evolution of $\mathbf{F}$}

To demonstrate the effective of the consensus partition matrix $\mathbf{F}$, specifically, we evaluate the ACC of consensus partition $\mathbf{F}$ learned at each iteration, as shown in Figure \ref{evolution}. We conduct the t-SNE algorithm \cite{maaten2008visualizing} on the consensus partition matrix $\mathbf{F}$ with different iterations, namely, $1^{st}$,$5^{th}$,$10^{th}$ and $20^{th}$ iteration. As the experimental results on Figure \ref{LMVC-LFA convergence} shows, our algorithms quickly converge to a local minimum with less than 20 iterations.

Two examples of the evolution of consensus partition matrix $\mathbf{F}$ on Mfeat and Plant are demonstrated in Figure \ref{evolution}. As Figure \ref{evolution} shows, with the increasing number of iterations, the clustering structures of data become more significant and clearer than the old ones. These results clearly demonstrate the effectiveness of the learned consensus matrix $\mathbf{F}$ for clustering.

\subsection{Parameter Sensitivity Analysis}

\subsubsection{Parameter Sensitivity of LF-MVC-GAM}

In this section, we analyze the impact of the regularization hyper-meter $\lambda$ in LF-MVC-GAM to the clustering performance. We chose the parameter $\lambda$  ranging from $\left[2^{-5}, 2^{-4}, \cdots, 2^{5}\right] $ by grid search, which indicates the trade-off between late fusion alignment and the regularization term. Figure \ref{MVC_LFA_SensitivityFig} plots the clustering performance by varying $\lambda$ in a large range on CCV and Caltech102-5. The clustering performance of the second best algorithm on the respectively dataset is also provided for reference.

We have the following observations from Figure \ref{MVC_LFA_SensitivityFig}: i) The clustering performance of LF-MVC-GAM firstly increases to a high value and generally maintains it up to slight variation with the increasing value of $\lambda$. ii) Comparing to the second best algorithm, our proposed LF-MVC-GAM demonstrates stable performance across a wide range of $\lambda$.

\subsubsection{Parameter Sensitivity of LF-MVC-LAM}

Due to the experimental results of LF-MVC-LAM, we observe that: i) all the two hyper-parameters are effective in improving the clustering performance; ii) LF-MVC-LAM is practically stable against the these parameters that it achieves competitive performance in a wide range of parameter settings; iii) the acc first increases to a high value and generally maintains it up to slight variation with the increasing value of $\lambda$. LF-MVC-LAM demonstrates stable performance across a wide range of $\lambda$. iiii) the proposed algorithm is relatively sensitive to the neighbor numbers $\tau$ which relatively reflects the inner data structure of data themselves. However, it still outperforms the second-best algorithm in most of the benchmark datasets. The results can be found in the appendix due to limitation of space.

\begin{figure*}[!htbp]
\centering
\subfloat[Caltech102-5]{\includegraphics[width = 0.225\textwidth]{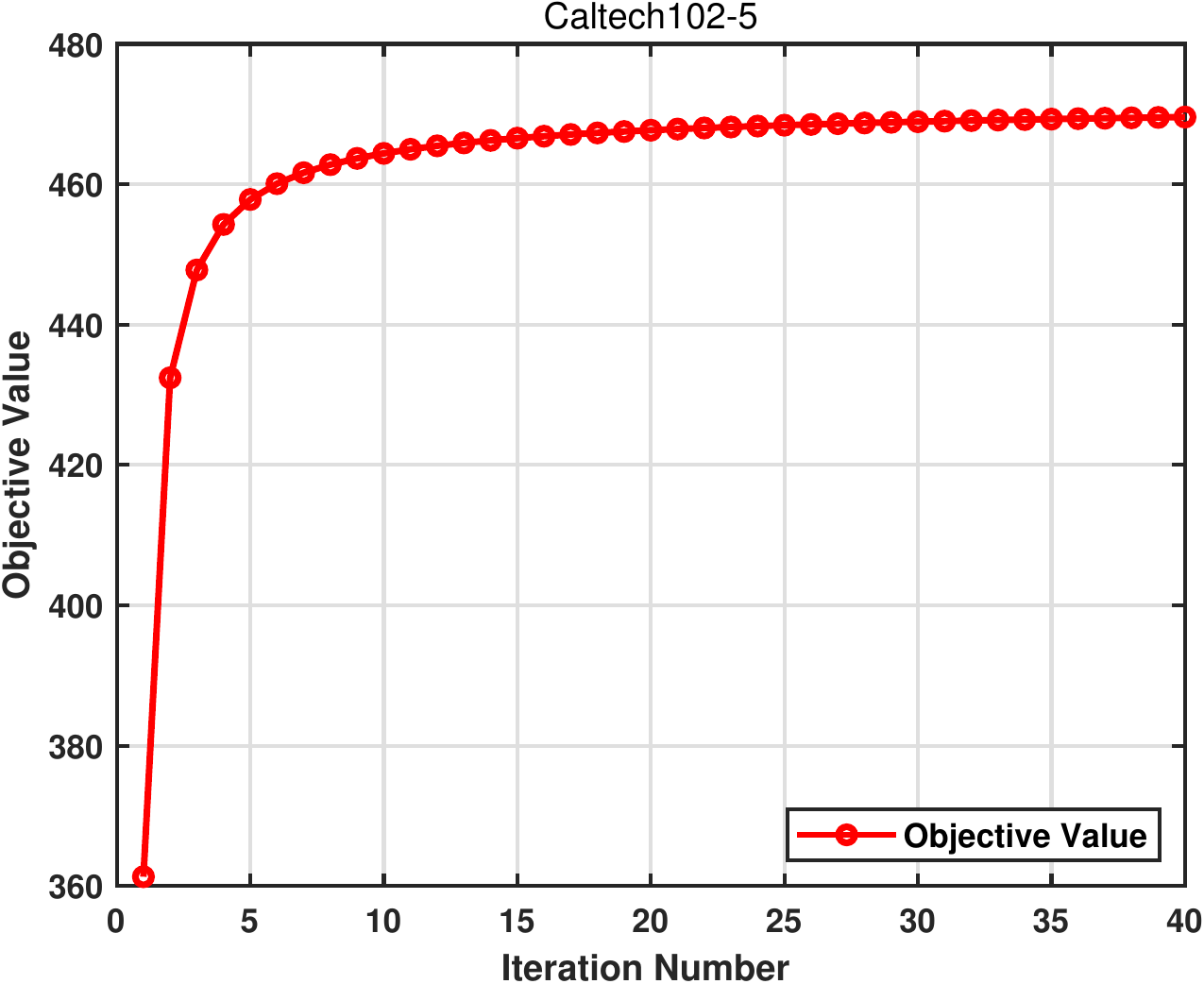}\label{flower17_paraMRLambdaACC}}
\subfloat[Caltech102-10]{\includegraphics[width = 0.225\textwidth]{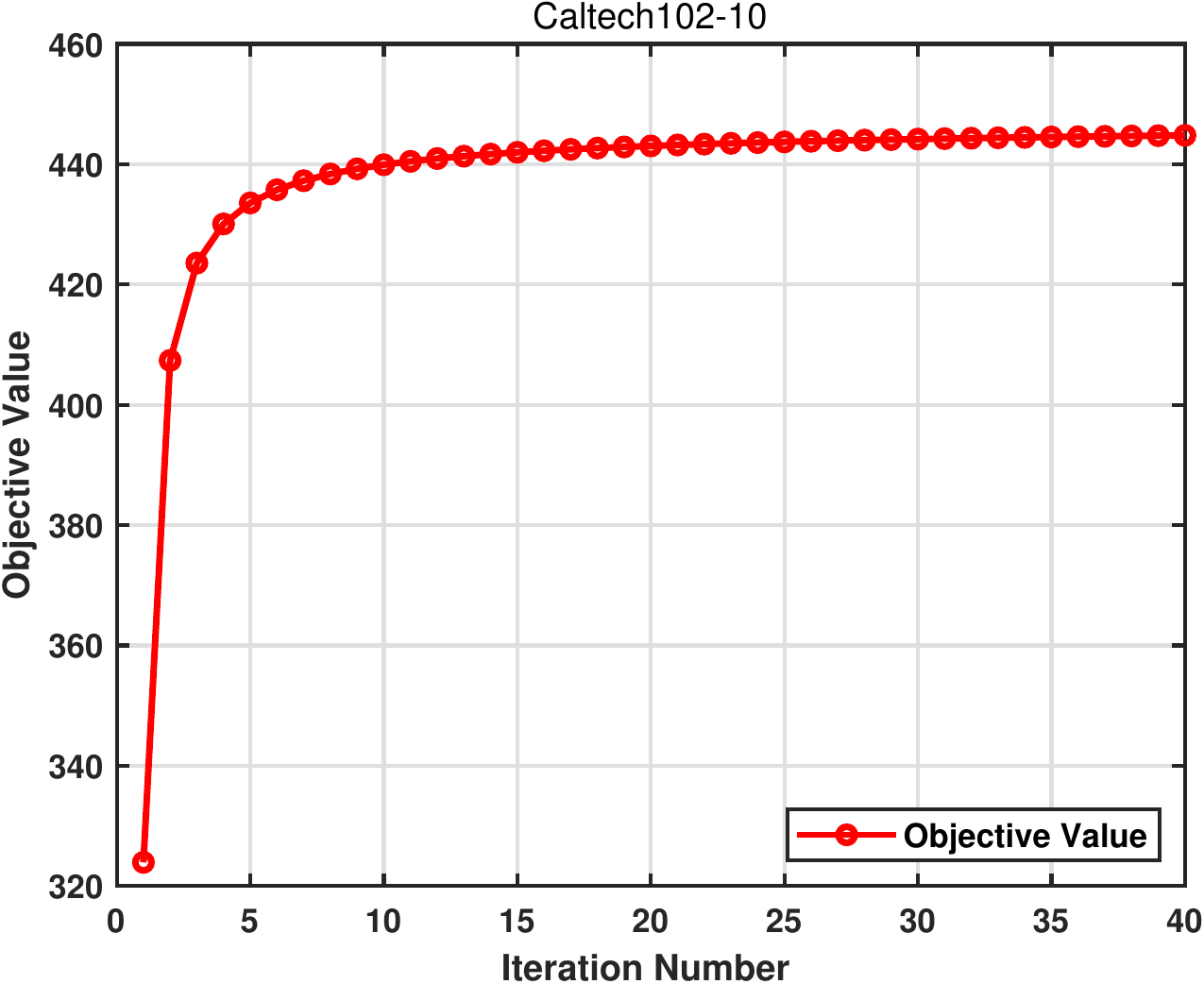}\label{flower17_paraMRLambdaACC}}
\subfloat[Caltech102-15]{\includegraphics[width = 0.225\textwidth]{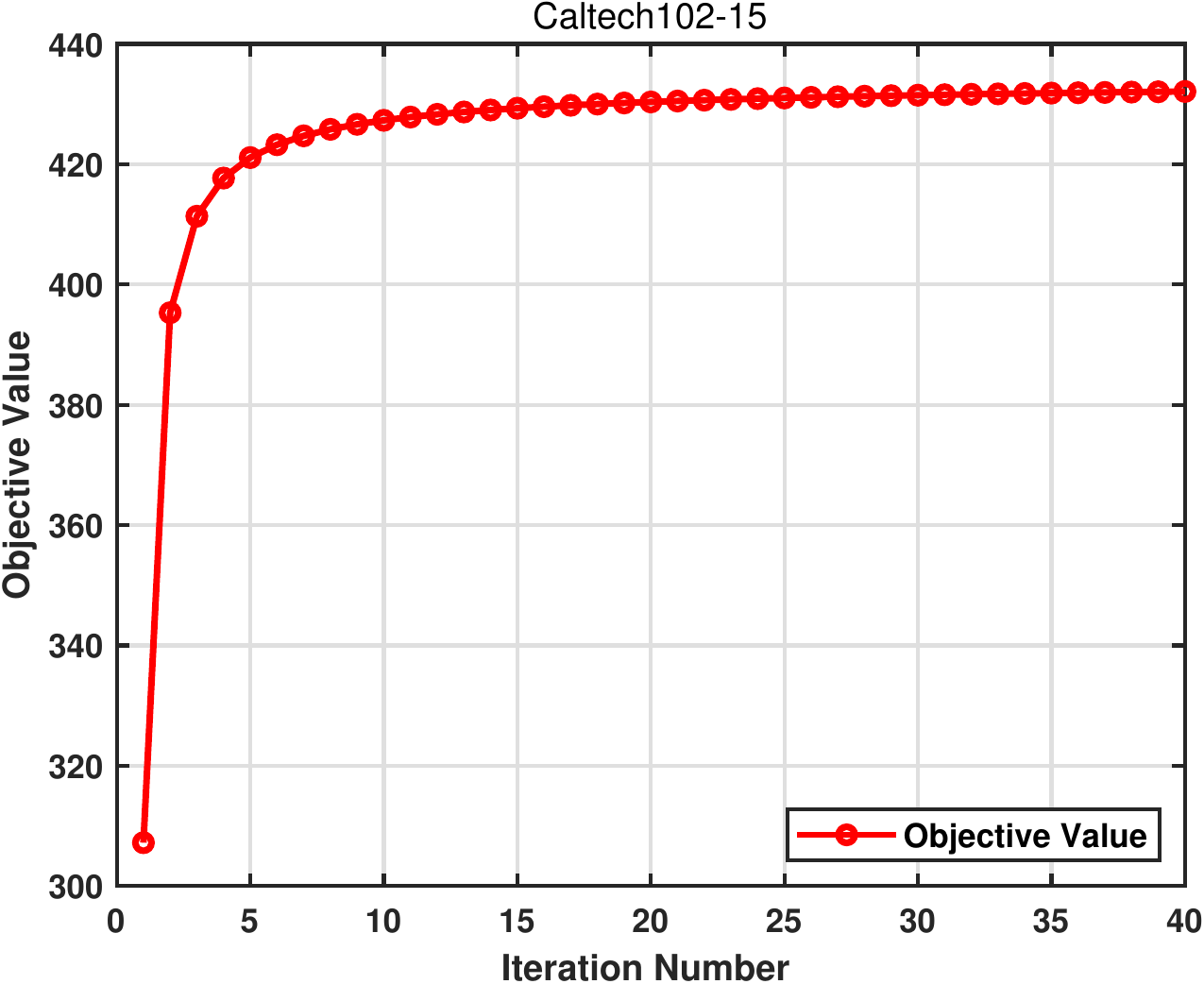}\label{flower17_paraMRLambdaNMI}}
\subfloat[Caltech10-20]{\includegraphics[width = 0.225\textwidth]{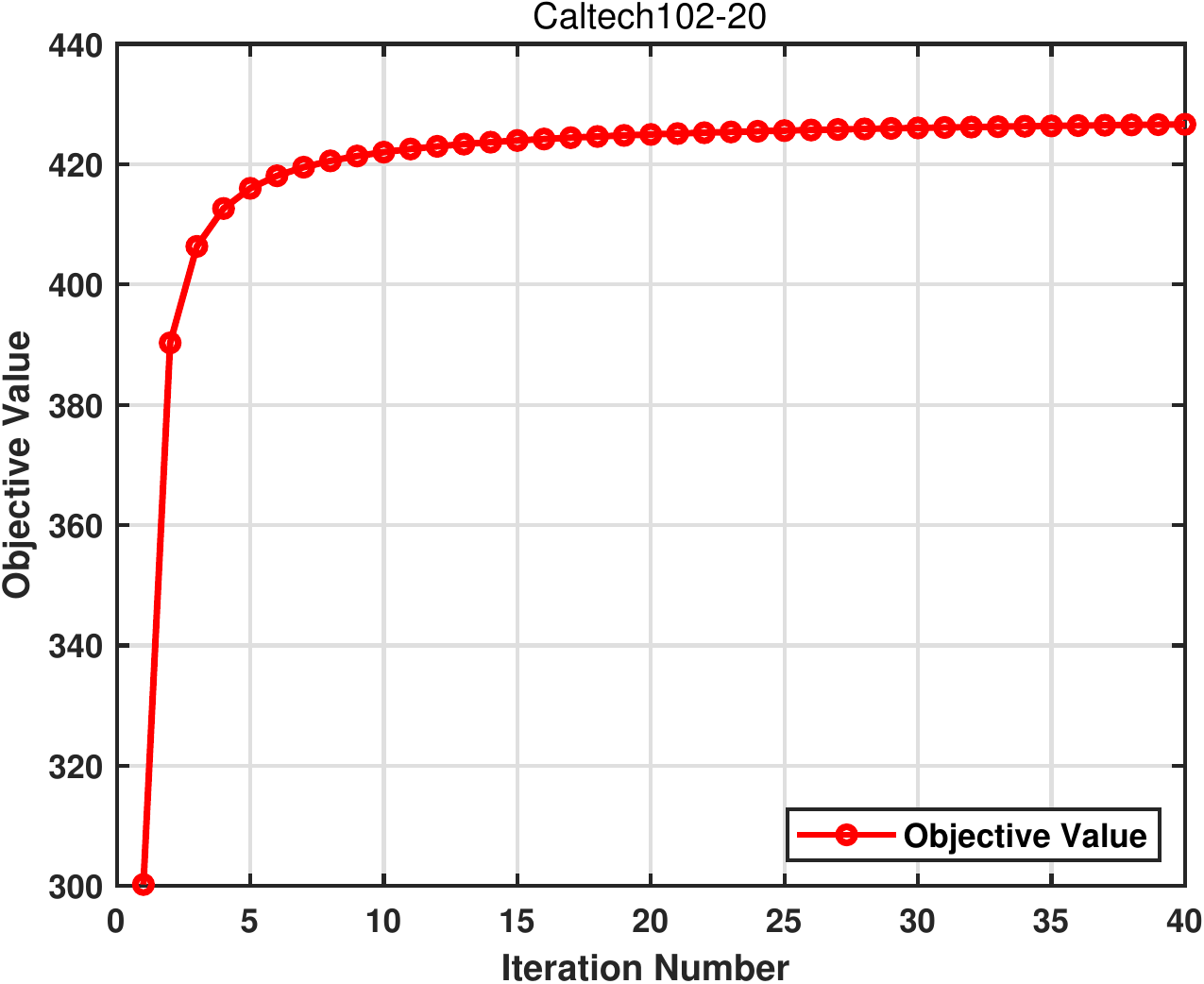}\label{flower17_paraMRLambdaNMI}}\\
\caption{{The convergence of the proposed LF-MVC-GAM on four Caltech102 datasets. Other datasets' results are omitted due to space limit and are provided in supplementary materials.}}\label{MVC-LFA convergence}
\end{figure*}

\begin{figure*}[!htbp]
\centering
\subfloat[Caltech102-5]{\includegraphics[width = 0.225\textwidth]{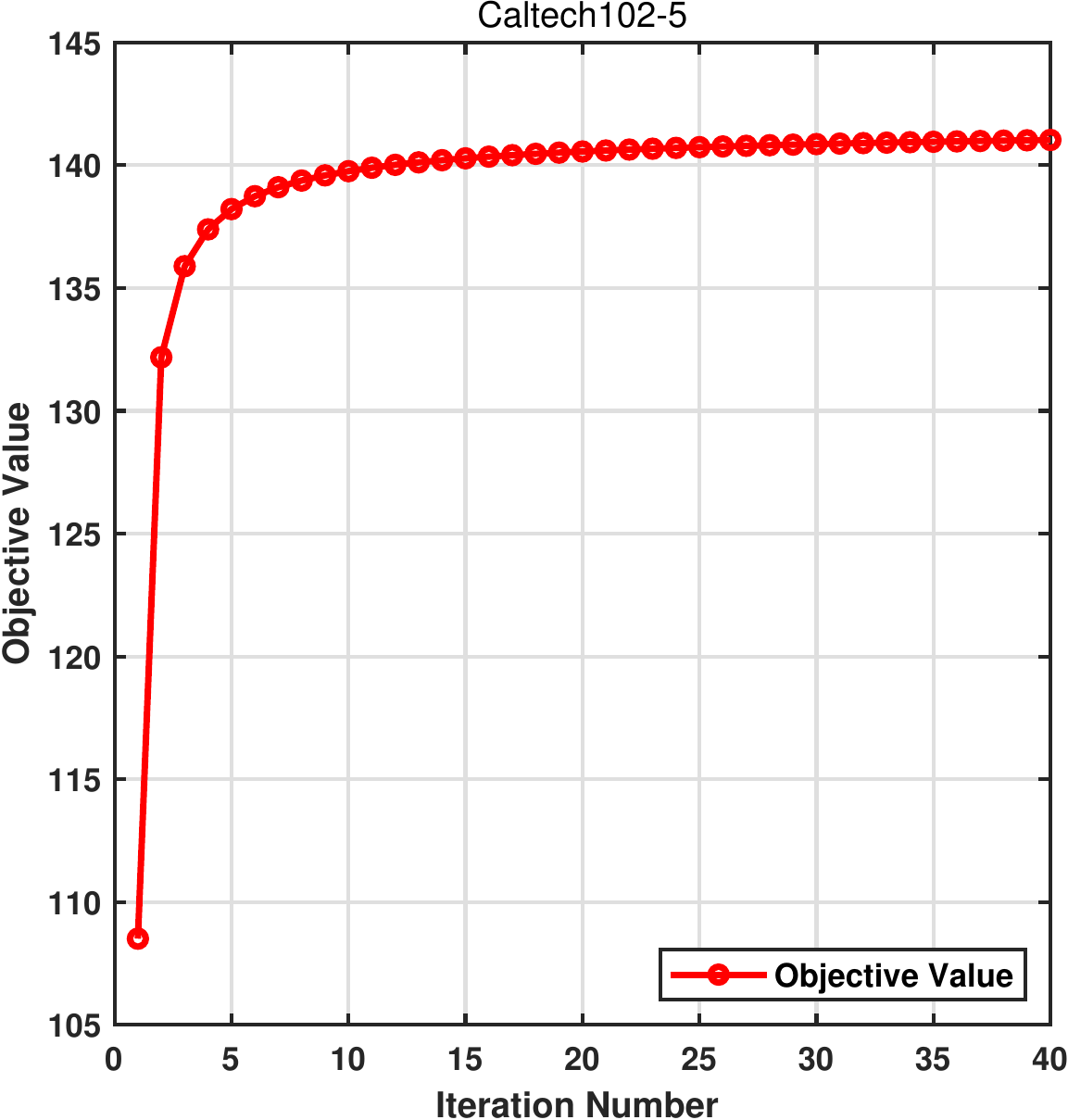}\label{flower17_paraMRLambdaACC}}
\subfloat[Caltech102-10]{\includegraphics[width = 0.225\textwidth]{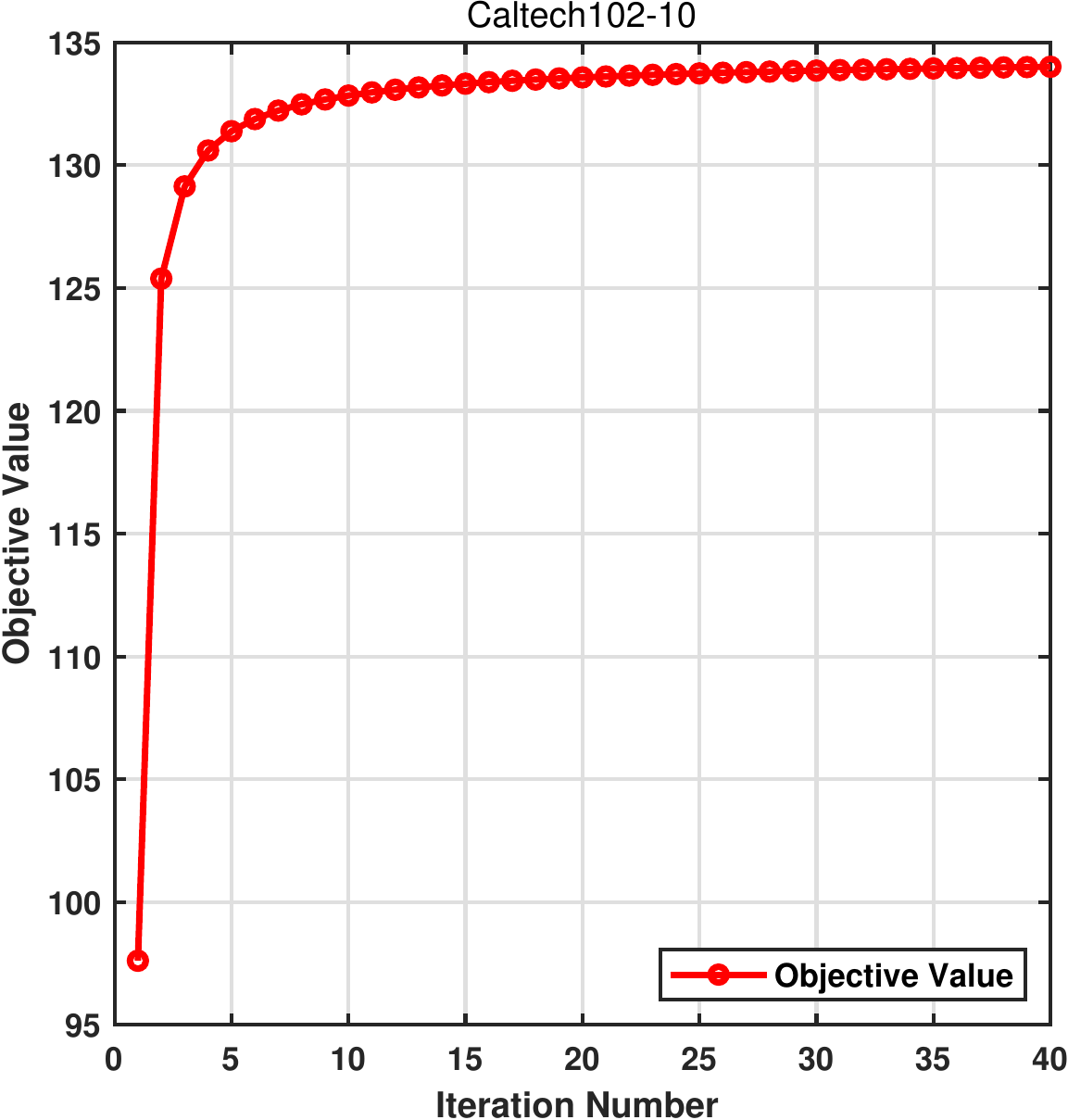}\label{flower17_paraMRLambdaACC}}
\subfloat[Caltech102-15]{\includegraphics[width = 0.225\textwidth]{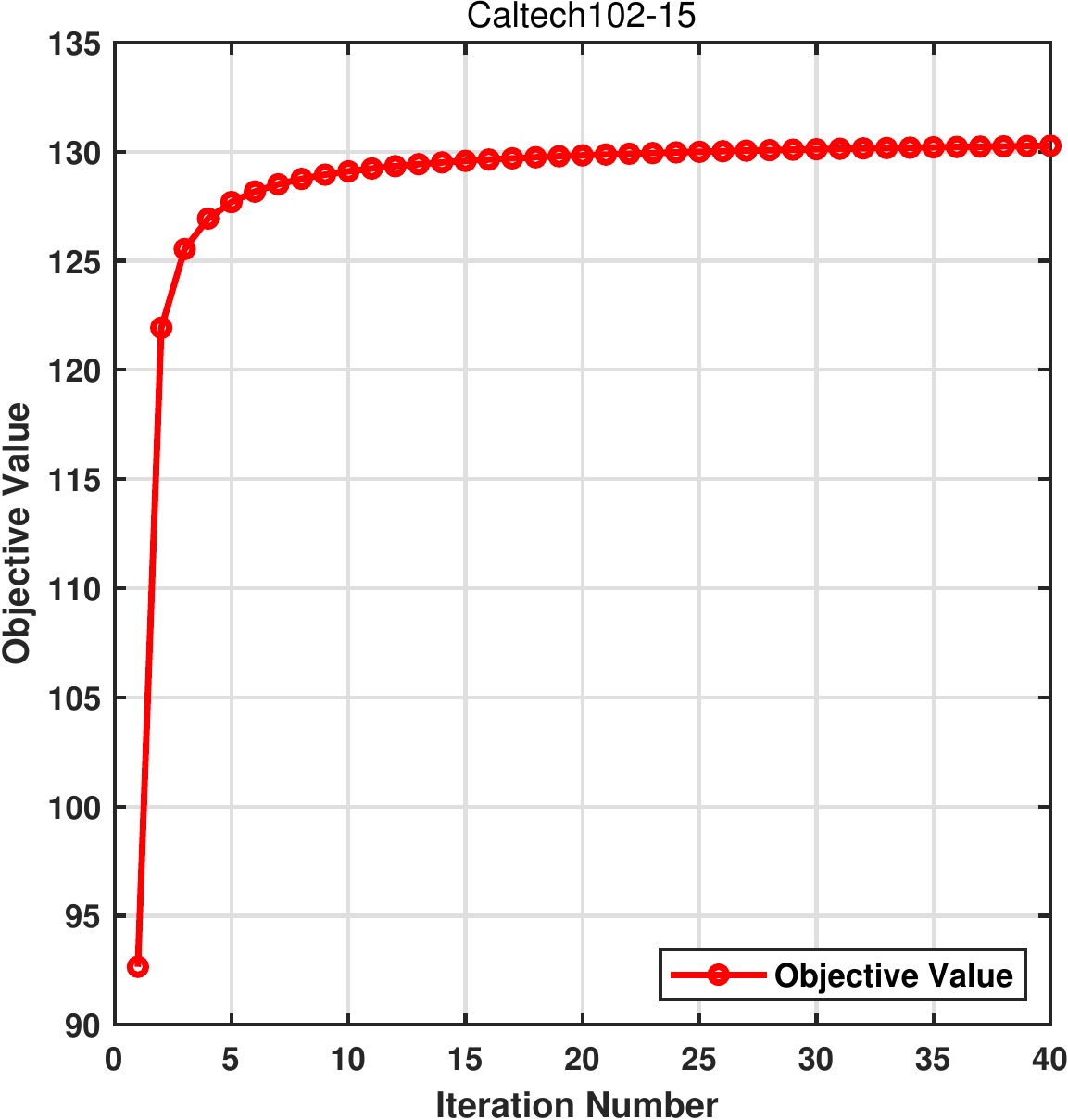}\label{flower17_paraMRLambdaNMI}}
\subfloat[Caltech10-20]{\includegraphics[width = 0.225\textwidth]{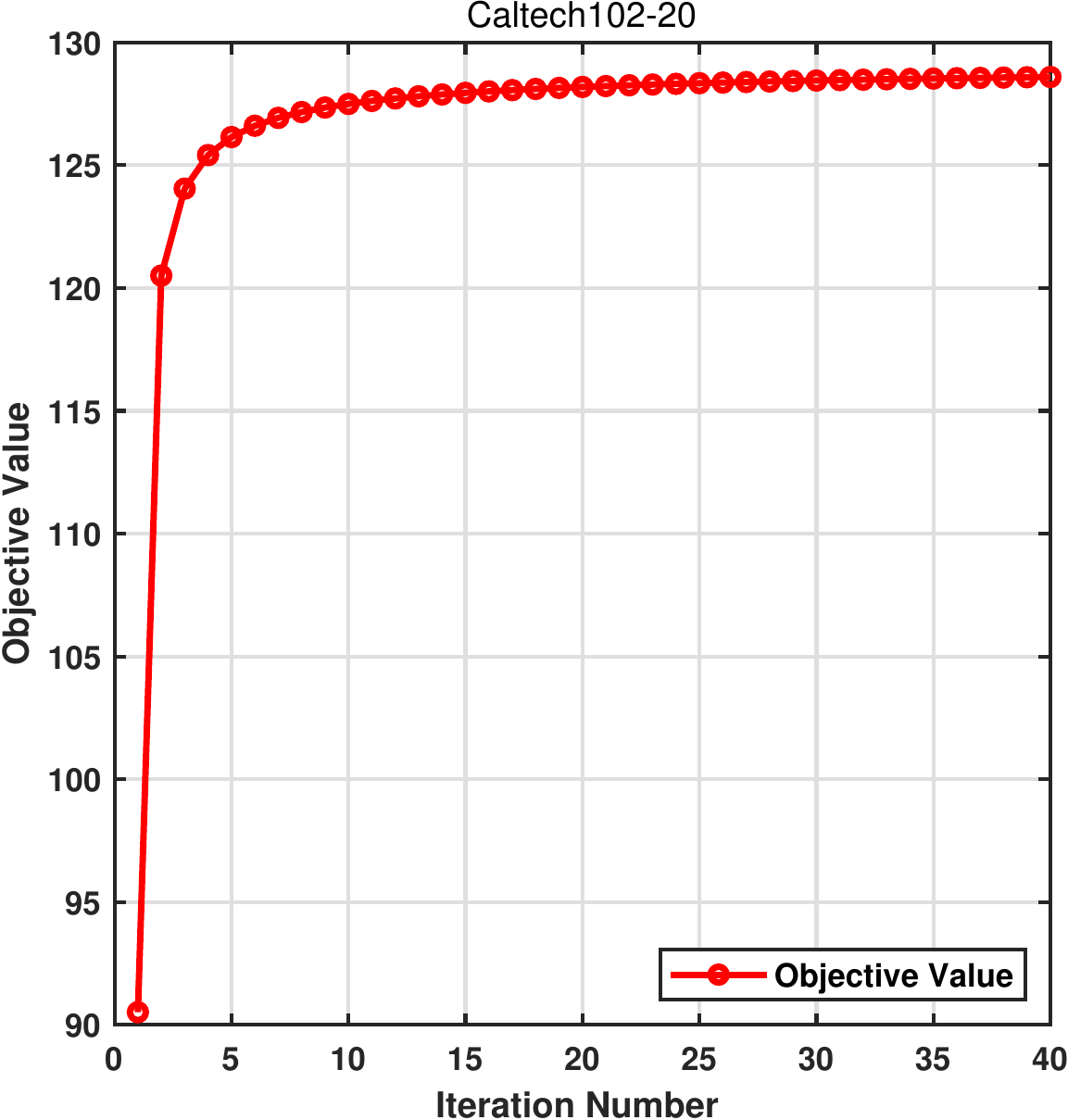}\label{flower17_paraMRLambdaNMI}}\\
\caption{{The convergence of the proposed LF-MVC-LAM on four Caltech102 datasets. Other datasets' results are omitted due to space limit and are provided in supplementary materials.}}\label{LMVC-LFA convergence}
\end{figure*}

\subsection{Convergence of the Proposed Algorithms}

Our algorithms are theoretically guaranteed to converge to a local minimum according to \cite{Bezdek2003}.  The examples of the evolution of the objective value on the experimental results are shown in Figures \ref{MVC-LFA convergence} and \ref{LMVC-LFA convergence}. In the above experiments, we observe that the objective value of our algorithms do monotonically increase at each iteration and that they usually converge in less than $15$ iterations. These results clearly verifies our proposed algorithms' convergences.

\section{Conclusion}\label{conclusion}

In this article, we propose two simple but effective methods LF-MVC-GAM and LF-MVC-LAM, which maximize the alignment between the consensus partition matrix and the weighted base partition matrices with orthogonal transformation and its local version. We theoretically uncover the connection between the existing $k$-means and the proposed alignment. Based on this, we derive a simple novel optimization framework for multi-view clustering, which significantly reduces the computational complexity and simplifies the optimization procedure. Two iterative algorithms are developed to efficiently solve the resultant optimization. By the virtue of it, the proposed algorithms show clearly superior clustering performance with significant reduction in computational cost on benchmark datasets, in comparison with state-of-the-art methods. In the future, we will explore to capture the noises or low-quality partition existing in basic partitions and further improve the fusion method.

\section*{Acknowledgments}
This work was supported by the National Key R\&D Program of China 2020AAA0107100, the Natural Science Foundation of China (project no. 61922088, 61773392 and 61976196), and Outstanding Talents of  ``Ten Thousand Talents Plan'' in Zhejiang Province (project no. 2018R51001).

\bibliographystyle{unsrt}  
\bibliography{references}  

\appendix
\section{Proof of Lemma 1}
\begin{lemma}\label{trace form1}
For any given matrix $\mathbf{P}\in \mathbb{R}^{k\times k}$, the inequality $\mathrm{Tr}^2 (\mathbf{P}) \leq k \cdot \mathrm{Tr}({\mathbf{P}}^{\top}\mathbf{P})$ always holds.
\end{lemma}

\begin{proof}
For a given matrix $\mathbf{P}\in \mathbb{R}^{k\times k}$, $\mathrm{Tr}^2 (\mathbf{P}) = (\sum_{i=1}^{k} \sigma_{i})^2,$ where $\sigma_{i}$ denotes the $i$-th singular value of $\mathbf{P}$. Further, $\mathrm{Tr}({\mathbf{P}}^{\top}\mathbf{P}) = \sum_{i=1}^{k} \sigma_{i}^2 $. According to Cauchy-Schwarz Inequality, we can obtain that $(\sum_{i=1}^{k} \sigma_{i})^2  \leq k\sum_{i=1}^{k} \sigma_{i}^2$. Therefore Lemma \ref{trace form1} is proved.
\end{proof}

\section{Proof of Lemma 2}

\begin{lemma}\label{trace form2}
$\mathrm{Tr}(\mathbf{B} \mathbf{B}^{\top}) \leq m^2 k$.
\end{lemma}
\begin{proof}
\begin{equation}
\begin{split}
\mathrm{Tr}(\mathbf{B} \mathbf{B}^{\top}) & = \mathrm{Tr}( {\sum\nolimits_{p=1}^m {\boldsymbol{\beta}}_p \mathbf{H}_p \mathbf{W}_p)}^{\top}( \sum\nolimits_{q=1}^m {\boldsymbol{\beta}}_q \mathbf{H}_q \mathbf{W}_q)\leq \mathrm{Tr}[{(\sum\nolimits_{p=1}^m {\mathbf{H}_p}\mathbf{W}_p)}^{\top}(\sum\nolimits_{p=1}^m {\mathbf{H}_q} \mathbf{W}_q)] \\
& = \mathrm{Tr} (\sum\nolimits_{p,q=1}^m {(\mathbf{H}_p\mathbf{W}_p)}^{\top} (\mathbf{H}_q\mathbf{W}_q))   \leq \frac{m^2}{2} (\mathrm{Tr} [{(\mathbf{H}_p\mathbf{W}_p)}^{\top} {(\mathbf{H}_p\mathbf{W}_p)}]  + \mathrm{Tr} [{(\mathbf{H}_q\mathbf{W}_q)}^{\top} {(\mathbf{H}_q\mathbf{W}_q)}]) = m^2 k,
\end{split}
\end{equation}
 Therefore Lemma \ref{trace form2} is proved.
\end{proof}

\section{Gap between various optimization functions}

According to Theorem 1, $\mathrm{Tr}^2 ({\mathbf{F}}^{\top}\mathbf{B}) = (\sum_{i=1}^{k} \sigma_{i})^2  \leq k\sum_{i=1}^{k} \sigma_{i}^2 = \mathrm{Tr}({\mathbf{F}}^{\top} \mathbf{B} \mathbf{B}^{\top} \mathbf{F})$ where $\sigma_{i}$ denotes the $i$-th singular value of ${\mathbf{F}}^{\top}\mathbf{B}$. The equation holds when all the $\sigma_i$ are equal. Then, it is not difficult to verify that
	\begin{equation}
	    \underbrace{\mathrm{Tr}(\mathbf{B} \mathbf{B}^{\top}) - \mathrm{Tr} ( {\mathbf{F}}^{\top} \mathbf{B} \mathbf{B}^{\top} \mathbf{F})}_{\Large{\textcircled{\small{1}}}}  \leq  \underbrace{\mathrm{Tr}(\mathbf{B} \mathbf{B}^{\top}) - \frac{1}{k} \mathrm{Tr}^2 ({\mathbf{F}}^{\top}\mathbf{B})}_{\Large{\textcircled{\small{2}}}} \leq \underbrace{m^2k - \frac{1}{k} \mathrm{Tr}^2 ({\mathbf{F}}^{\top}\mathbf{B})}_{\Large{\textcircled{\small{3}}}}
	\end{equation}
\\
This optimization strategy has been widely applied in machine learning community \cite{DBLP:journals/ml/ChapelleVBM02,DBLP:books/daglib/0097035}. For example, instead of directly minimizing the infeasible number of errors of the leave-one-out procedure, the authors suggest to minimize the upper bound of relative errors \cite{DBLP:journals/ml/ChapelleVBM02}, leading to the widely known "radius-margin bound".
	\begin{equation}
	\label{bound}
	    \underbrace{\mathrm{Tr}(\mathbf{B} \mathbf{B}^{\top}) - \mathrm{Tr} ( {\mathbf{F}}^{\top} \mathbf{B} \mathbf{B}^{\top} \mathbf{F})}_{\Large{\textcircled{\small{1}}}}  \leq  \underbrace{\mathrm{Tr}(\mathbf{B} \mathbf{B}^{\top}) - \frac{1}{k} \mathrm{Tr}^2 ({\mathbf{F}}^{\top}\mathbf{B})}_{\Large{\textcircled{\small{2}}}} \leq \underbrace{m^2k - \frac{1}{k} \mathrm{Tr}^2 ({\mathbf{F}}^{\top}\mathbf{B})}_{\Large{\textcircled{\small{3}}}}
	\end{equation}
	
	Further, in our case,  $\Large{\textcircled{\small{1}}}$ is upper-bounded by $\Large{\textcircled{\small{3}}}$. We theoretically prove that the solution of $\Large{\textcircled{\small{3}}}$ is a global optimum for problem $\Large{\textcircled{\small{1}}}$, as stated in Lemma \ref{lemma1}. Also, it is not difficult to verify that the solutions of $\Large{\textcircled{\small{3}}}$ and $\Large{\textcircled{\small{1}}}$ share the same clustering structure, leading to the same clustering results.\\
	
	In addition, solving $\Large{\textcircled{\small{3}}}$ is proved with $\mathcal{O}(n)$ time complexity while $\Large{\textcircled{\small{1}}}$ is $\mathcal{O}(n^3)$, making it more efficient in large-scale data clustering tasks. Please check the last paragraph of left column in page four in the revised version.\\

	\begin{lemma}\label{lemma1}
	The optimal solution of $\mathbf{F}$ to minimize $\Large{\textcircled{\small{3}}}$ is also a solution of the global optimum  of $\Large{\textcircled{\small{1}}}$. 
	\end{lemma}
	
	\begin{proof}
	For $\Large{\textcircled{\small{3}}}$, we can obtain the optimum solution with SVD. Suppose that the matrix $\mathbf{B}$ in $\Large{\textcircled{\small{3}}}$ has the  rank-$k$ truncated singular value decomposition form as $\mathbf{B} = {\mathbf{S} \mathbf{\Sigma} \mathbf{V}^{\top}}$, where ${\mathbf{S}_{k}} \in \mathbb{R}^{n \times n},{\mathbf{\Sigma}_{k}} \in \mathbb{R}^{n \times k},{\mathbf{V}_{k}} \in \mathbb{R}^{k \times k}$. The optimization problem in $\Large{\textcircled{\small{3}}}$ has a closed-form solution that $\mathbf{F}_{3} = \mathbf{S}_{k} \mathbf{V}^{\top}$, where $\mathbf{S}_{k}$ denotes the $k$ left singular vector corresponding to $k$-largest singular values. The detailed deviation can found in our Theorem 2 in the manuscript. Further, we can easily obtain that the solution of $\Large{\textcircled{\small{1}}}$ is the product of the $k$-largest eigenvectors of $\mathbf{B} \mathbf{B}^{\top}$ and arbitrary rank-$k$ orthogonal square matrix. The the $k$-largest eigenvectors of $\mathbf{B} \mathbf{B}^{\top}$ is also the rank-$k$ left singular vector of $\mathbf{B}$. Therefore, it is not difficult to verify that the solution of $\Large{\textcircled{\small{3}}}$ can also reach the global optimum  of $\Large{\textcircled{\small{1}}}$. This completes the proof.\\
	\end{proof}

	To see this point in depth, we have also plotted the relative gap of \Large{\textcircled{\small{1}}} and \Large{\textcircled{\small{2}}} on the tested datasets in Figure. \ref{gap}. The red line is $\mathrm{Tr}(\mathbf{B} \mathbf{B}^{\top}) - \mathrm{Tr} ( {\mathbf{F}}^{\top} \mathbf{B} \mathbf{B}^{\top} \mathbf{F})$ (\Large{\textcircled{\small{1}}}) while the blue line represents $\mathrm{Tr}(\mathbf{B} \mathbf{B}^{\top}) - \frac{1}{k} \mathrm{Tr}^2 ({\mathbf{F}}^{\top}\mathbf{B})$(\Large{\textcircled{\small{2}}}). Form Figure. \ref{gap}, it is clearly observed that \Large{\textcircled{\small{2}}} is indeed the upper bound for the red line \Large{\textcircled{\small{1}}} and the gap is tight.\\
 	
\begin{figure*}[!htbp]
\centering
\subfloat[Flower17]{\includegraphics[width = 0.25\textwidth]{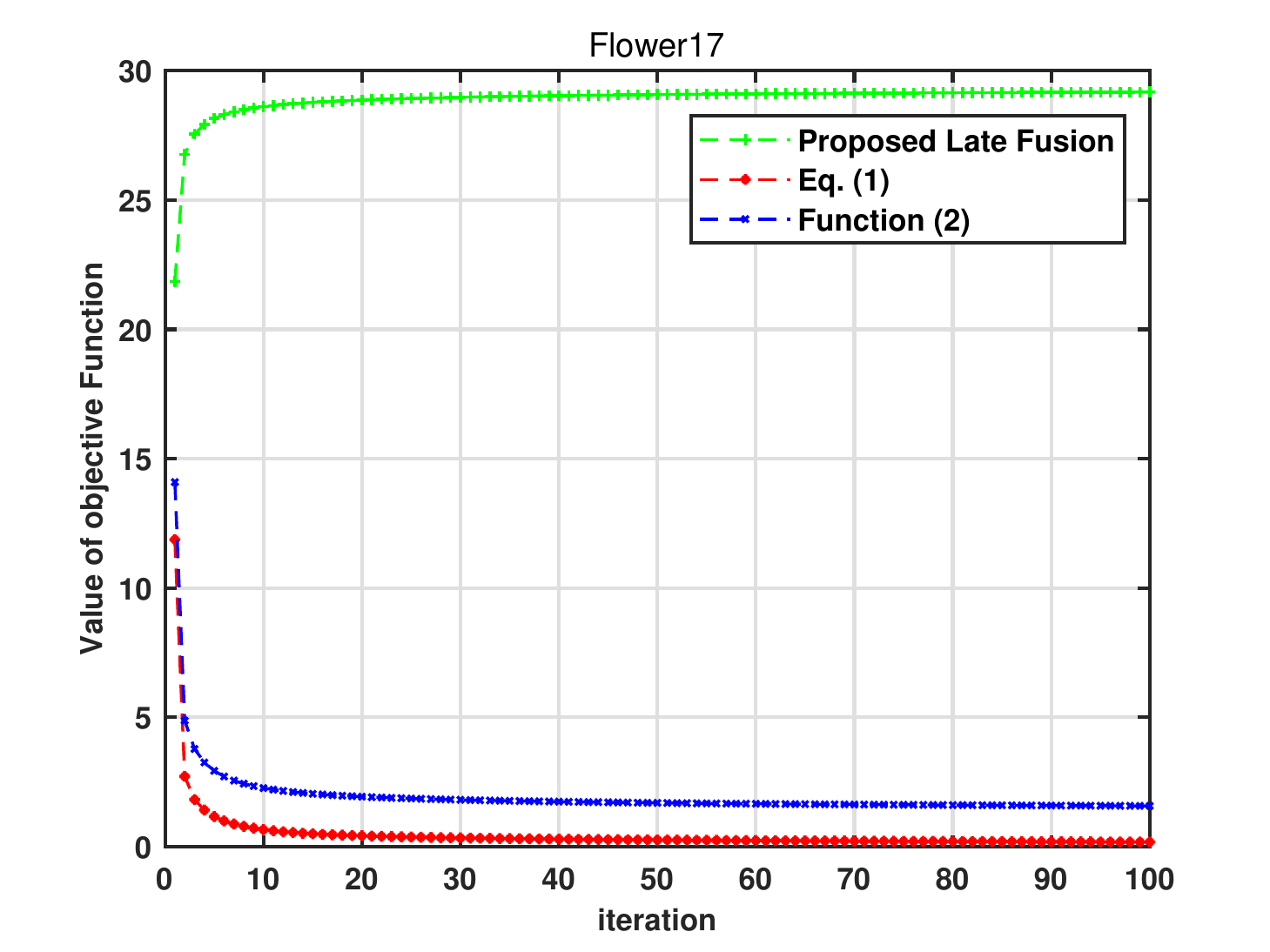}}
\subfloat[YALE]{\includegraphics[width =0.25\textwidth]{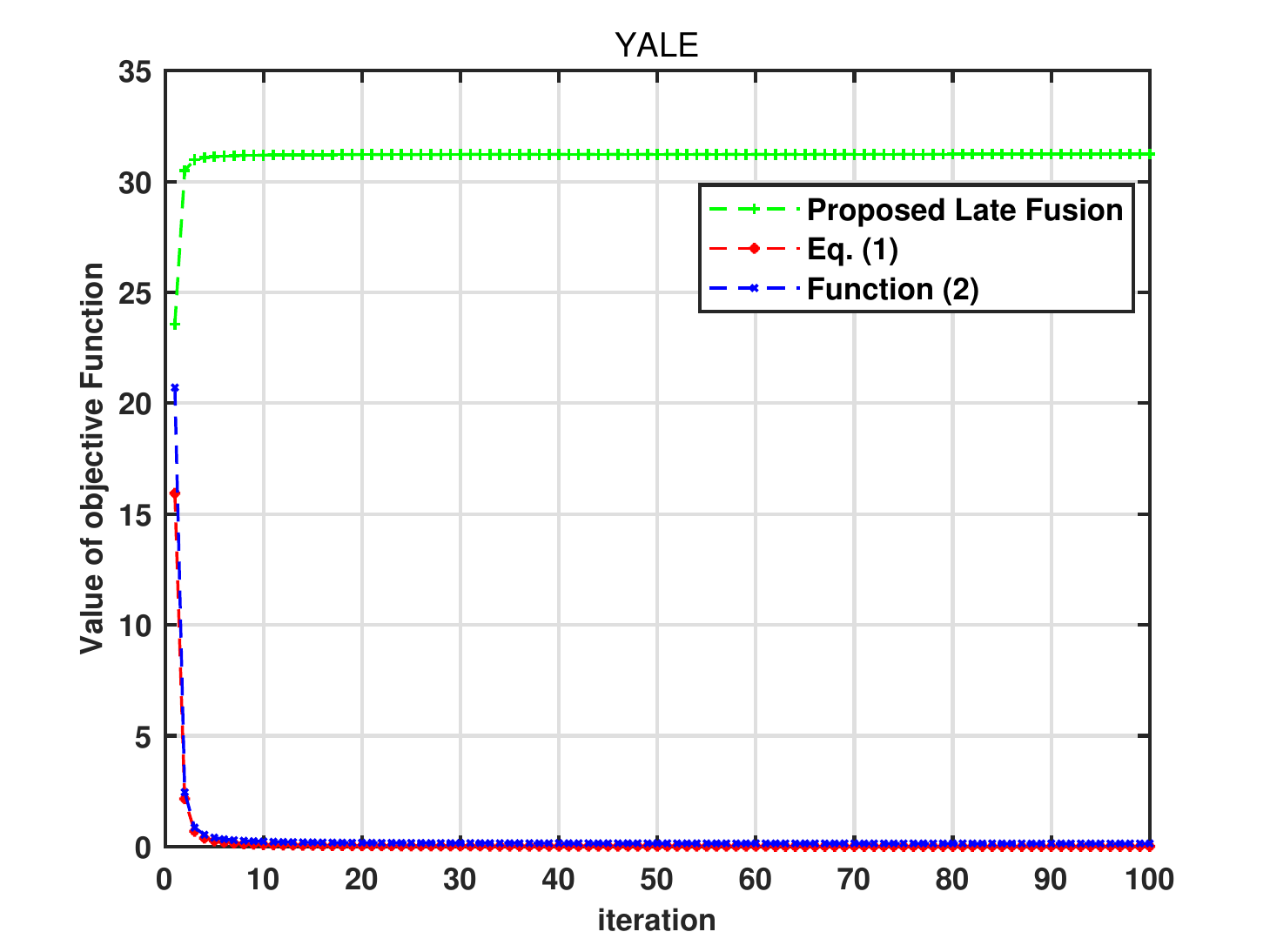}}
\subfloat[AR10P]{\includegraphics[width=0.25\textwidth]{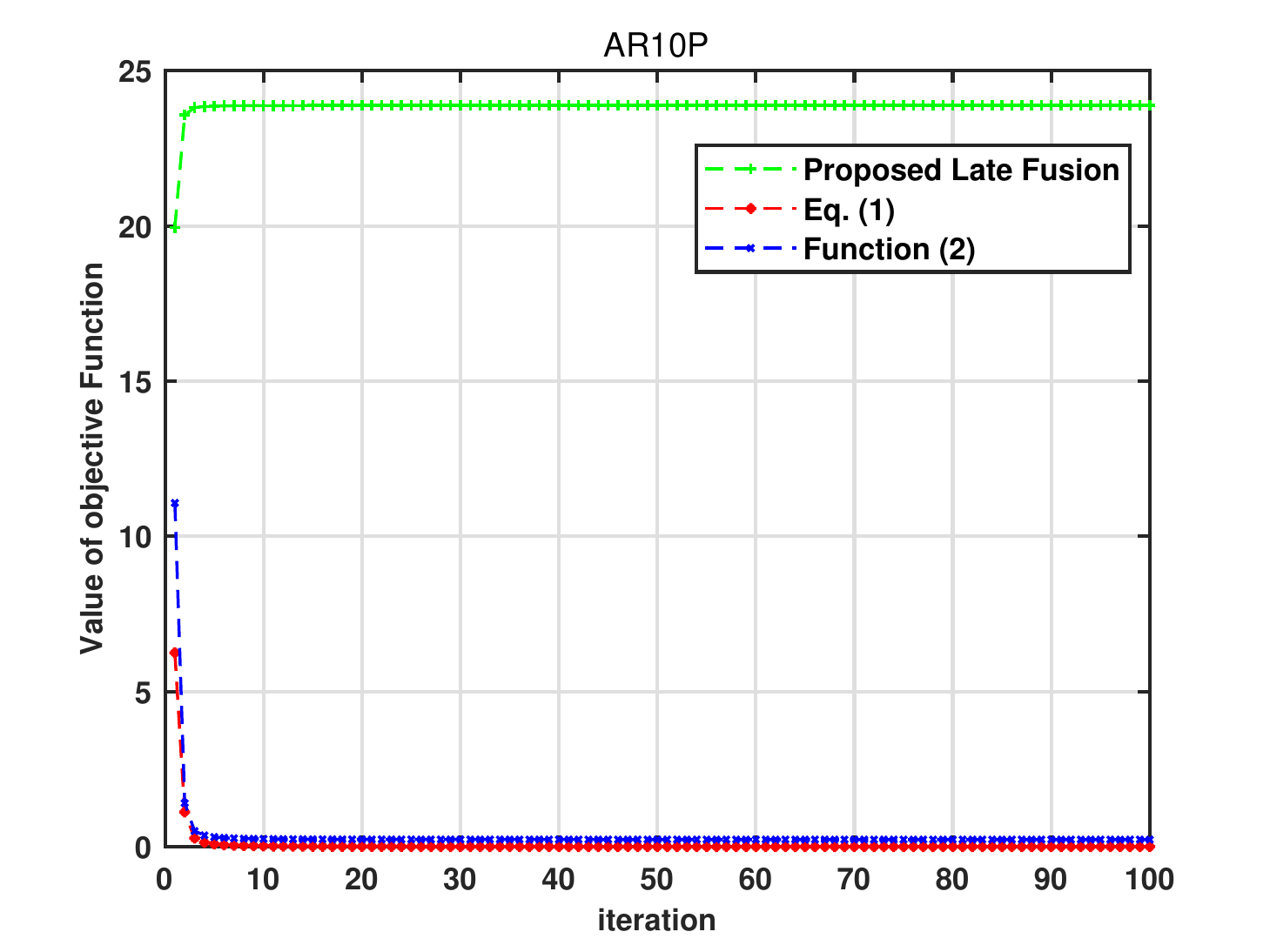}}
\subfloat[Mfeat]{\includegraphics[width =0.25\textwidth]{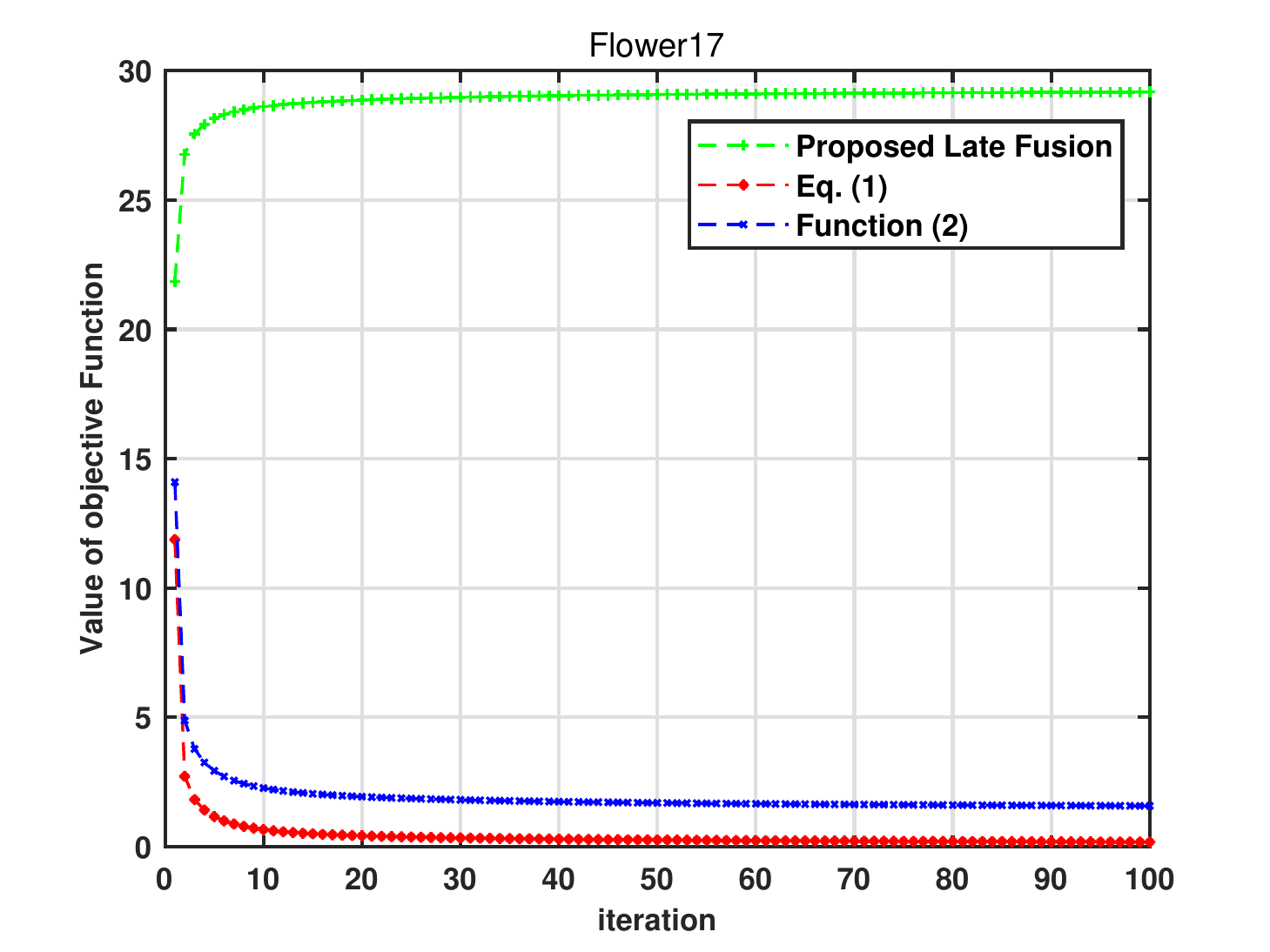}}
\caption{The evolution of two function objectives with the variation of iterations.}\label{gap}
\end{figure*}

\section{GENERALIZATION ANALYSIS}

In this section,we are going to give more detailed proof for our proposed algorithms' generalization ability.
Firstly, we present some notations or definitions in the beginning.

Given an input space $\mathcal{X}$, $n$ samples $\left\{\mathbf{x}_i\right\}_{i=1}^n$ drawn i.i.d. from an unknown sampling distribution $\mu$ that $\mathcal{X} \sim \mu$. Then the label space is $\mathcal{Y}$ and $\mathcal{F}$ is a group of mapping functions which each function $g: \mathcal{X} \rightarrow \mathcal{Y}$. Based on these assumptions, we will define the Rademacher complexity and Gaussian complexity.

For a given function class $\mathcal{F}$ and $n$ samples,suppose $\sigma_1,\sigma_2,\cdots,\sigma_n$ and $z_1,z_2,\cdots, z_n$ are independent Rademacher variables and standard normal variables ($\sigma  \in {\left\{-1,1\right\}}, z \sim\mathcal{N}(0,1)$). The empirical Rademacher complexity and Gaussian complexity can be expressed as,
\begin{equation}
\label{empirical}
\mathfrak{R}_{n}(\mathcal{F})=\mathop{\mathbb{E}}\nolimits_{\boldsymbol{\sigma}} \sup _{g \in \mathcal{F}} \frac{1}{n} \sum_{i=1}^{n} \sigma_{i} g\left(\mathbf{x}_{i}\right),\qquad\mathcal{Q}_{n}(\mathcal{F})=\mathop{\mathbb{E}}_{z} \sup _{g \in \mathcal{F}} \frac{1}{n} \sum_{i=1}^{n} z_{i} g\left(\mathbf{x}_{i}\right),\\
\end{equation}
where $\mathfrak{R}_{n}(\mathcal{F})$ and $\mathcal{Q}_{n}(\mathcal{F})$ define the Rademacher complexity and Gaussian complexity respectively. Furthermore, the expected Rademacher complexity and Gaussian complexity can be expressed by taking expectations of Eq. (\ref{empirical}) over the sampling distribution $\mu$,
\begin{equation}
\label{empirical}
\mathfrak{R}(\mathcal{F})=\mathop{\mathbb{E}}\limits_{\mathbf{x} \sim \mu} (\mathfrak{R}_{n}(\mathcal{F})),\qquad
\mathcal{Q}(\mathcal{F})=\mathop{\mathbb{E}}\limits_{\mathbf{x}\sim \mu} (\mathcal{Q}_{n}(\mathcal{F})),\\
\end{equation}

Our proposed formulation for LF-MVC-GAM and LF-MVC-LAM is given as follows,
\begin{equation}\label{LFA 3}
\begin{split}
\; &\; \max\limits_{\mathbf{F},\left\{\mathbf{W}_p\right\}_{p=1}^m,\boldsymbol{\beta}} \mathrm{Tr} ({\mathbf{F}}^{\top} \sum\nolimits_{p=1}^m{\boldsymbol{\beta}}_p \mathbf{H}_p \mathbf{W}_p) +  \lambda \mathrm{Tr} ({\mathbf{F}}^{ \top} \mathbf{M}),\\
\;&\; s.t.\;{\mathbf{F}}^{ \top}\mathbf{F} = \mathbf{I}_k, \mathbf{{W}}_p^{ \top}\mathbf{W}_p = \mathbf{I}_k,{\Vert {\boldsymbol{\beta}} \Vert}_2 = 1,{\boldsymbol{\beta}}_p \geq 0,
\end{split}
\end{equation}
where $\left\{\mathbf{H}_p \right\}_{p=1}^m$ are a set of base partitions, $\left\{\mathbf{W}_p \right\}_{p=1}^m$ are the respective transformed matrices and $\boldsymbol{\beta} \in \mathbb{R}^{m}$ is the view-specific coefficient. Let $\mathbf{C} = [\mathbf{c}_1,\cdots,\mathbf{c}_k]$ be the learned center matrix composed of the $k$ centroids, and $\boldsymbol{\beta},\,\{\mathbf{W}_{p}\}_{p=1}^{m}$ the learned kernel weights and permutation matrices by the proposed algorithms. The proposed algorithms should make the following empirical error small \cite{maurer2008generalization,liu2016dimensionality},
\begin{equation}\label{eq:Generalization1}
{
1-\frac{1}{n} \sum_{i=1}^{n}\left[\max\limits_{ \mathbf{y}\in\{\mathbf{e}_1,\cdots,\mathbf{e}_k\}}
\langle h(\mathbf{x}_i),\mathbf{C}\mathbf{y}\rangle\right],
}
\end{equation}
where $i$ denotes the index of the sample, $h(\mathbf{x}_i)=\sum_{p=1}^{m}{\boldsymbol{\beta}}_p\mathbf{W}_{p}^{\top}h_p(\mathbf{x}_{i}^{(p)}) +  \lambda \mathbf{M}_{i,:}^{\top}$, and $\mathbf{e}_1,\ldots,\mathbf{e}_k$ form the orthogonal bases of $\mathbb{R}^k$. $h_p(\mathbf{x}_{i})$ denotes the $i$-th row of $\mathbf{H}_p$ representing the $i$-th sample's representation in $p$-th view and $\mathbf{M}_{i,:}$ denotes the $i$-th row of the regularization matrix $\mathbf{M}$. The operator $\langle \mathbf{A},\mathbf{B}\rangle$ denotes  $\mathrm{Tr} (\mathbf{A}^{\top} \mathbf{B})$.

We define the function class $\mathcal{F}$ for LF-MVC-GAM first:
\begin{equation}\label{eq:Generalization2}
{
\begin{split}
\mathcal{F}= &\Big\{g:\;\mathbf{x}\mapsto 1 - \max\nolimits_{\mathbf{y}\in\{\mathbf{e}_1,\cdots,\mathbf{e}_k\}}
\left\langle \sum\nolimits_{p=1}^{m}{\boldsymbol{\beta}}_p\mathbf{W}_{p}^{\top}h_p(\mathbf{x}^{(p)}),\,\mathbf{C}\mathbf{y}\right\rangle {\Big|}{\mathbf{F}}^{\top}\mathbf{F} = \mathbf{I}_k,\,\mathbf{{W}}_p^{\top}\mathbf{W}_p = \mathbf{I}_k,\,{\Vert {\boldsymbol{\beta}} \Vert}_2 = 1,\,{\boldsymbol{\beta}}_p \geq 0,\\
&\qquad\forall p,\forall \mathbf{x}\in\mathcal{X}\Big\},
\end{split}
}
\end{equation}

where $\mathbf{x} $ is a multi-view data sample feature containing $m$ views $\left[\mathbf{x}^{(1)},\mathbf{x}^{(2)}, \cdots, \mathbf{x}^{(m)} \right]$ and $\mathbf{C}$ is the obtained cluster centroids by performing Floyd algorithm on $\mathbf{F}$\cite{DBLP:journals/prl/Jain10}.
According to \cite{mohri2018foundations}, an essential theorem about the generalization error bound is gives as follows,
\begin{theorem}
\label{1}
Let $\mathcal{F}$ be a family of functions class mapping on $\mathcal{X}$ with value range $\left[ a,b \right]$.
For any $\delta>0$, with probability at least $1-\delta$, the following holds for all $g\in \mathcal{F}$:
\begin{equation}
\mathbb{E}\left[g({\mathbf{x}})\right]-\frac{1}{n}\sum\nolimits_{i=1}^{n}g({\mathbf{x}}_i)\leq 2\mathfrak{R}(\mathcal{F})+ 	\left( b-a \right) \sqrt{\frac{\log{1/\delta}}{2n}},
\end{equation}
\end{theorem}

\begin{theorem}
\label{2}
For each loss function $g \in \mathcal{F}$ learned by LF-MVC-GAM and LF-MVC-LAM, $1-4m \leq g(\x_i) \leq 1+4m$ always holds.
\end{theorem}

\begin{proof}
From Cauchy-Schwarz inequality $m = \sum\nolimits_{p=1}^{m} 1 \sum\nolimits_{p=1}^{m} {\boldsymbol{\beta}}_p^2 \geq (\sum\nolimits_{p=1}^{m} {\boldsymbol{\beta}}_p)^2 $, the $ \sum\nolimits_{p=1}^{m}{\boldsymbol{\beta}}_p \leq \sqrt{m}$. Moreover, we can easily obtain that for each $ i,j \in\{1,\cdots,n\},-1 \leq{h}_{p}^{\top}(\mathbf{x}^{(p)}_i) {h}_{q}(\mathbf{x}^{(q)}_j) \leq 1$ because of for each $ p \in \{1,\cdots,m\},\mathbf{H}_p \mathbf{H}_p^{\top} = \mathbf{I}_k $ holds.

We can prove that, for $i,j \in\{1,\cdots,n\}$,
\begin{equation}
\begin{split}
{{h}^{\top}(\mathbf{x}_i)} {h}(\mathbf{x}_j) &\;= \left(\sum\nolimits_{p=1}^{m}{\boldsymbol{\beta}}_p \mathbf{W}_{p}^{\top}h_p(\mathbf{x}_{i}^{(p)}) + \lambda \mathbf{M}_{i,:}^{\top}\right)^{\top} \left(\sum\nolimits_{p=1}^{m}{\boldsymbol{\beta}}_p\mathbf{W}_{p}^{\top}h_p(\mathbf{x}_{j}^{(p)}) + \lambda \mathbf{M}_{j,:}^{\top}\right),\\
&\; = \left(\sum\nolimits_{p=1}^{m}({\boldsymbol{\beta}}_p+\frac{1}{\sqrt{m}}) \mathbf{W}_{p}^{\top}h_p(\mathbf{x}_{i}^{(p)})\right)^{\top} \left(\sum\nolimits_{p=1}^{m}({\boldsymbol{\beta}}_p+\frac{1}{\sqrt{m}})\mathbf{W}_{p}^{\top}h_p(\mathbf{x}_{j}^{(p)})\right),\\
&\; = \sum\nolimits_{p,q=1}^{m} ({\boldsymbol{\beta}}_p+\frac{1}{\sqrt{m}}) ({\boldsymbol{\beta}}_p+\frac{1}{\sqrt{m}})
{h}_{p}^{\top}(\mathbf{x}^{(p)}_i)\mathbf{W}_{p}\mathbf{W}_{q}^{\top}{h}_{q}(\mathbf{x}^{(q)}_j),\\
&\; \leq \sum\nolimits_{p,q=1}^{m} ({\boldsymbol{\beta}}_p+\frac{1}{\sqrt{m}}) ({\boldsymbol{\beta}}_q+\frac{1}{\sqrt{m}})
{h}_{p}^{\top}(\mathbf{x}^{(p)}_i){h}_{q}(\mathbf{x}^{(q)}_j),\\
&\; \leq \sum\nolimits_{p,q=1}^{m} ({\boldsymbol{\beta}}_p+\frac{1}{\sqrt{m}}) ({\boldsymbol{\beta}}_q+\frac{1}{\sqrt{m}}),\\
&\; =\sum\nolimits_{p,q=1}^{m} ({\boldsymbol{\beta}}_p{\boldsymbol{\beta}}_q)+ m * \frac{2}{\sqrt{m}} \sum\nolimits_{p=1}^{m} {\boldsymbol{\beta}}_p + m ,\\
&\; \leq \sum\nolimits_{p,q=1}^{m} \frac{1}{2}({\boldsymbol{\beta}}_p^2 + {\boldsymbol{\beta}}_q^2)+ m *\frac{2}{\sqrt{m}} \sum\nolimits_{p=1}^{m} {\boldsymbol{\beta}}_p + m ,\\
&\; \leq m + 2m +m = 4m.
\end{split}
\end{equation}
In the same way, it is easy to prove that $-4m \leq {h}(\mathbf{x}_i) {h}^{\top}(\mathbf{x}_j) \leq 4m$.

Therefore for $v\in\{1,\cdots,k\}$
\begin{equation}
\begin{split}
& \langle h(\mathbf{x}_i),\mathbf{C}\mathbf{e}_{v}\rangle  \\
& = h^{\top}(\mathbf{x}_i)\left(\frac{1}{|\mathbf{C}_{v}|}\sum\nolimits_{ j\in\mathbf{C}_{v}} {h(\mathbf{x}_{j})}\right)\\
& =  \left(\frac{1}{|\mathbf{C}_{v}|}\sum\nolimits_{ j\in\mathbf{C}_{v}}h^{\top}(\mathbf{x}_i)h(\mathbf{x}_{j})\right) \geq -4m.
\end{split}
\end{equation}
Therefore we can obtain that $ 1-4m \leq g(\mathbf{x}_i) = 1 - \max \langle h(\mathbf{x}_i),\mathbf{C}\mathbf{y}\rangle \leq 1+4m$.
This completes the proof.
\end{proof}

Based on Theorem \ref{1} and \ref{2}, we have the following Theorem,
\begin{theorem}\label{3}
Let $\mathcal{F}$ be a family of functions class mapping on $\mathcal{X}$ learned from LF-MVC-GAM and LF-MVC-LAM. For any $\delta>0$, with probability at least $1-\delta$, the following holds for all $g\in \mathcal{F}$:
\begin{equation}
\mathbb{E}\left[g({\mathbf{x}})\right]-\frac{1}{n}\sum\nolimits_{i=1}^{n}g({\mathbf{x}}_i)\leq 2\mathfrak{R}(\mathcal{F})+ 	\left( 8m \right) \sqrt{\frac{\log{1/\delta}}{2n}},
\end{equation}
\end{theorem}

In the next step, we are going to find the upper bound for $\mathfrak{R}(\mathcal{F})$.The Gaussian complexity is related to the Rademacher complexity by the following lemma,
\begin{lemma}
\begin{equation}
\mathfrak{R}(\mathcal{F}) \leq \sqrt{\frac{\pi}{2}} \mathcal{Q}(\mathcal{F})
\end{equation}
\end{lemma}

Therefore, we can upper bound the Rademacher complexity by finding a proper Gaussian process.
\begin{lemma}
(Slepians Lemma) Let $\Omega$ and $\Xi$ be mean zero, separable Gaussian processes indexed by a common set
$\mathcal{F},$ such that
$$
\begin{array}{c}
\mathbb{E} \left(\Omega_{g_{1}}-\Omega_{g_{2}}\right)^{2} \leq \mathbb{E}\left(\Xi_{g_{1}}-\Xi_{g_{2}}\right)^{2}, \forall g_{1}, g_{2} \in \mathcal{F} \\
\end{array}
$$
Then $ \mathbb{E} \sup _{g \in \mathcal{F}} \Omega_{g} \leq \mathbb{E} \sup_{g \in \mathcal{F}} \Xi_{g}$
\end{lemma}

This indicates that we need to find the proper Gaussian process for the bound. Suppose  $z_1,z_2,\cdots, z_n$ is standard Gaussian variables ($ z_i \sim \mathcal{N}(0,1), i \in \{1,\cdots,m\}$). In our case, let
\begin{equation}\label{eqtl1}
\Omega_{\mathbf{F},\left\{\mathbf{W}_p\right\}_{p=1}^m,\boldsymbol{\beta}}=\sum_{i=1}^{n}z_i \big(1 - \max\limits_{ \mathbf{y}\in\{\mathbf{e}_1,\cdots,\mathbf{e}_k\}}\langle h(\mathbf{x}_i),\mathbf{C}\mathbf{y}\rangle \big)
\end{equation}

Then we define that $\Delta=\{\mathbf{e}_1,\cdots,\mathbf{e}_k\}$ and $z_{il}$ are standard Gaussian variables ($ z_{il} \sim \mathcal{N}(0,1), i \in \{1,\cdots,n\}, l \in \{1,\cdots,k\} $).

\begin{equation}\label{eqtl2}
\begin{split}
&{\Xi_{\mathbf{F},\left\{\mathbf{W}_p\right\}_{p=1}^m,\boldsymbol{\beta}}} = \sum\limits_{i=1}^n \sum\limits_{l=1}^k  z_{il} f^{\top}(\mathbf{x}_i)\mathbf{C}\mathbf{e}_l.
\end{split}
\end{equation}
where $f(\mathbf{x}_i)=\sum_{p=1}^{m}{\boldsymbol{\beta}}_p\mathbf{W}_{p}^{\top}h_p(\mathbf{x}_{i}^{(p)})$.

	Firstly, assume that $\max_{\mathbf{y}\in\Delta}\left\langle g_{1}(\mathbf{x}_i),\mathbf{C}\mathbf{y}\right\rangle  = \left\langle g_{1}(\mathbf{x}_i),\mathbf{C}\mathbf{y}_1\right\rangle$ and $ \max_{\mathbf{y}\in\Delta}\left\langle g_{2}(\mathbf{x}_i),\mathbf{C}\mathbf{y}\right\rangle  = \left\langle g_{2}(\mathbf{x}_i),\mathbf{C}\mathbf{y}_2\right\rangle$, we can easily obtain that,
	\begin{equation}
	    \big | \max_{\mathbf{y}\in\Delta}\left\langle g_{1}(\mathbf{x}_i),\mathbf{C}\mathbf{y}\right\rangle - \max_{\mathbf{y}\in \Delta}\left\langle g_{2}(\mathbf{x}_i),\mathbf{C}^{\prime}\mathbf{y}\right\rangle  \big | = \big | \left\langle g_{1}(\mathbf{x}_i),\mathbf{C}\mathbf{y}_1\right\rangle - \left\langle g_{2}(\mathbf{x}_i),\mathbf{C}\mathbf{y}_2\right\rangle \big |.
	\end{equation}
	
	If $\left\langle g_{1}(\mathbf{x}_i),\mathbf{C}\mathbf{y}_1\right\rangle - \left\langle g_{2}(\mathbf{x}_i),\mathbf{C}\mathbf{y}_2\right\rangle \geq 0$, then it is easy to show that, 
	
	$ 0 \leq \left\langle g_{1}(\mathbf{x}_i),\mathbf{C}\mathbf{y}_1\right\rangle - \left\langle g_{2}(\mathbf{x}_i),\mathbf{C}\mathbf{y}_2\right\rangle \leq \left\langle g_{1}(\mathbf{x}_i),\mathbf{C}\mathbf{y}_1\right\rangle - \left\langle g_{2}(\mathbf{x}_i),\mathbf{C}\mathbf{y}_1\right\rangle \leq \max_{\mathbf{y}\in \Delta} \big | \left\langle g_{1}(\mathbf{x}_i),\mathbf{C}\mathbf{y}\right\rangle - \left\langle g_{2}(\mathbf{x}_i),\mathbf{C}\mathbf{y}\right\rangle \big |.$ 
	
	Else if $\left\langle g_{1}(\mathbf{x}_i),\mathbf{C}\mathbf{y}_1\right\rangle - \left\langle g_{2}(\mathbf{x}_i),\mathbf{C}\mathbf{y}_2\right\rangle \le 0$, we can still obtain that 
	
	$\big | \left\langle g_{1}(\mathbf{x}_i),\mathbf{C}\mathbf{y}_1\right\rangle - \left\langle g_{2}(\mathbf{x}_i),\mathbf{C}\mathbf{y}_2\right\rangle \big | = \big | \left\langle g_{2}(\mathbf{x}_i),\mathbf{C}\mathbf{y}_2\right\rangle - \left\langle g_{1}(\mathbf{x}_i),\mathbf{C}\mathbf{y}_1\right\rangle  \big | \leq \big | \left\langle g_{2}(\mathbf{x}_i),\mathbf{C}\mathbf{y}_2\right\rangle - \left\langle g_{1}(\mathbf{x}_i),\mathbf{C}\mathbf{y}_2\right\rangle  \big | \\ \leq \max_{\mathbf{y}\in \Delta} \big | \left\langle g_{1}(\mathbf{x}_i),\mathbf{C}\mathbf{y}\right\rangle - \left\langle g_{2}(\mathbf{x}_i),\mathbf{C}\mathbf{y}\right\rangle \big |$. \\
	
	Combine the above two circumstances, we can obtain that 
	$ \big | \max_{\mathbf{y}\in\Delta}\left\langle g_{1}(\mathbf{x}_i),\mathbf{C}\mathbf{y}\right\rangle - \max_{\mathbf{y}\in \Delta}\left\langle g_{2}(\mathbf{x}_i),\mathbf{C}^{\prime}\mathbf{y}\right\rangle  \big | \\ \leq \max_{\mathbf{y}\in \Delta} \big | \left\langle g_{1}(\mathbf{x}_i),\mathbf{C}\mathbf{y}\right\rangle - \left\langle g_{2}(\mathbf{x}_i),\mathbf{C}\mathbf{y}\right\rangle \big |$.

For any $g_1$ and $g_2$,the learned centroid matrix $\mathbf{C}$ and $\mathbf{C}^{\prime}$, we can obtain that,

\begin{equation}
\begin{split}
\mathbb{E} \left(\Omega_{g_{1}}-\Omega_{g_{2}}\right)^{2}  & = \mathbb{E}\sum_{i=1}^{n}z_i^2 \left( \max_{\mathbf{y}\in\Delta}\left\langle g_{1}(\mathbf{x}_i),\mathbf{C}\mathbf{y}\right\rangle - \max_{\mathbf{y}\in \Delta}\left\langle g_{2}(\mathbf{x}_i),\mathbf{C}^{\prime}\mathbf{y}\right\rangle \right)^2 \\
& \leq \mathbb{E}\sum_{i=1}^{n}z_i^2 \max_{\mathbf{y}\in \Delta} \left( \sum\limits_{p=1}^{m}{\boldsymbol{\beta}}_p h^{\top}_p(\mathbf{x}^{(p)}_i)\mathbf{W}_p\mathbf{C}\mathbf{y} -\sum\limits_{p=1}^{m}{\boldsymbol{\beta}}_p^{\prime}h_p^{\top}(\mathbf{x}^{(p)}_i){\mathbf{W}_{p}^{\prime}}\mathbf{C}^{\prime}\mathbf{y}\right)^2 \\
& = \mathbb{E}\sum_{i=1}^{n}z_i^2 \max_{\mathbf{y}\in \Delta}\left( \sum\limits_{l=1}^{k}y_{l}\left(\sum\limits_{p=1}^{m}{\boldsymbol{\beta}}_p h^{\top}_p(\mathbf{x}^{(p)}_i)\mathbf{W}_p\mathbf{C} -\sum\limits_{p=1}^{m}{\boldsymbol{\beta}}_p^{\prime}h_p^{\top}(\mathbf{x}^{(p)}_i){\mathbf{W}_{p}^{\prime}}\mathbf{C}^{
\prime}\right)\mathbf{e}_l \right)^2\\
& \leq \mathbb{E}\sum_{i=1}^{n}\sum\limits_{l=1}^{k}   z_{il}^2 
\left(\left(\sum\limits_{p=1}^{m}{\boldsymbol{\beta}}_p h^{\top}_p(\mathbf{x}^{(p)}_i)\mathbf{W}_p\mathbf{C} -\sum\limits_{p=1}^{m}{\boldsymbol{\beta}}_p^{\prime}h_p^{\top}(\mathbf{x}^{(p)}_i){\mathbf{W}_{p}^{\prime}}\mathbf{C}^{
\prime}\right)\mathbf{e}_l \right)^2,\\
&\; = \mathbb{E}\left(\Xi_{g_{1}}-\Xi_{g_{2}}\right)^{2}.
\end{split}
\end{equation}
where the last inequality (second to last step) holds because $\sum_{l=1}^k {y}_{l} =1$ (\cite{liu2017spectral,liu2020efficient}).

We further obtain that,
\begin{equation}\label{eq:guassianbound}
\small{
\begin{split}
\mathbb{E}\left[\sup\nolimits_{g\in \mathcal{F}}{\Xi}_g\right] & = \mathbb{E}_z\left[\sup_{\mathbf{C},\left\{\mathbf{W}_p\right\}_{p=1}^m,\boldsymbol{\beta}}\sum\limits_{i=1}^n \sum\limits_{l=1}^k  z_{il} \left(\sum\nolimits_{p=1}^{m}{\boldsymbol{\beta}}_p\mathbf{W}_{p}^{\top}h_p(\mathbf{x}^{(p)}_i)\right)^{\top}\mathbf{C}\mathbf{e}_l\right] \\
&\leq \mathbb{E}_z \left[\sum_{l=1}^k\left|\sum\limits_{i=1}^n z_{il}\right|\right]\\
& \leq k\sqrt{n}.
\end{split}
}
\end{equation}
The inequality holds for the Jensen's inequality and $\mathop{\mathbb{E}}\left( {z}_{il}^2  \right)=1$.

Combing the above discussion,we now give the generalization bound for our proposed algorithms.

For any $\delta>0$, with probability at least $1-\delta$, the following holds for all $g\in \mathcal{G}$:
\begin{equation}
\mathbb{E}\left[g({\mathbf{x}})\right]-\frac{1}{n}\sum\nolimits_{i=1}^{n}g({\mathbf{x}}_i)\leq \frac{\sqrt{\pi/2}k}{\sqrt{n}} + \left( 8m \right) \sqrt{\frac{\log{1/\delta}}{2n}},
\end{equation}

This completes the proof.

\section{Experimental Results}
\subsection{Clustering Performance}
We also report the purity in Table \ref{purity result} and \ref{mgc purity result}. Again, we observe that the proposed algorithm significantly outperforms other early-fusion multi-kernel algorithms. These results are consistent with our observations in the metric of ACC and NMI.

We have also report the time consuming on large-scale multiple-kernel benchmark datasets in Table \ref{large time result}.
\begin{table*}[!htbp]
\begin{center}
{
\centering
\caption{Purity comparison of different multiple kernel clustering algorithms on twelve benchmark data sets.The best result is highlighted with underlines. Boldface means no statistical difference from the best one and '-' means the out-of-memory failure.}
\label{purity result}
\resizebox{\textwidth}{!}{
\begin{tabular}{|c|c|c|c|c|c|c|c|c|c|c|c|}
\hline
\multirow{2}{*}{Datasets} & \multirow{2}{*}{A-MKKM} & \multirow{2}{*}{SB-KKM} & {MKKM} & {CRSC} & {RMKKM} & {RMSC} & {LMKKM} & {MKKM-MR} & {LKAM} & LF-MVC-GAM & LF-MVC-LAM \\
\cline{4-10,11-12}
 &  &  &\cite{huang2012multiple}  &\cite{kumar2011co}  &\cite{du2015robust}  &\cite{xia2014robust}  &\cite{gonen2014localized}  &\cite{Liu2016Multiple}  &\cite{Li2016Multiple}  & \multicolumn{2}{c|}{Proposed}  \\
\midrule
\multicolumn{12}{|c|}{Purity$(\%)$}    \\ \hline

AR10P & 39.23  & 43.08  & 40.00  & 32.44  & 32.31  & 33.08  & 40.77  & 39.23  & 28.46  & 43.85  & \underline{$\mathbf{53.08}$}  \\ \hline
YALE & 53.94  & 57.58  & 52.73  & 53.82  & 58.18  & 57.24  & 53.94  & 60.00  & 49.09  & 55.76  & \underline{$\mathbf{61.82}$}  \\ \hline
Plant & 60.21  & 56.38  & 56.38  & 60.21  & 55.00  & 59.47  & - & 58.72  & 56.60  & $\mathbf{62.66}$  & \underline{$\mathbf{65.00}$}  \\ \hline
Mfeat & 94.20  & 86.00  & 66.75  & 77.93  & 77.45  & 94.60  & 94.90  & 92.55  & \underline{$\mathbf{96.65}$}  & $\mathbf{95.80}$  & $\mathbf{95.90}$  \\ \hline
Flower17 & 51.99  & 44.63  & 46.84  & 47.55  & 55.07  & 54.12  & 49.63  & 60.51  & 59.26  & $\mathbf{62.13}$ & \underline{$\mathbf{63.97}$}  \\ \hline
CCV & 23.56  & 23.62  & 22.41  & 23.56  & 20.79  & 20.08  & 24.36  & 24.88  & 22.90  & \underline{$\mathbf{28.51}$}  & \underline{$\mathbf{30.12}$}  \\ \hline
Caltech102-5 & 36.24  & 38.04  & 29.80  & 33.57  & 33.92  & 34.90  & 39.41 & 39.02  & 33.92  & $\mathbf{40.59}$  & \underline{$\mathbf{41.96}$}  \\ \hline
Caltech102-10 & 31.24  & 34.47  & 24.22  & 30.83  & 28.82  & 31.47  & 34.02  & 35.88  & 30.39  & $\mathbf{37.06}$  & \underline{$\mathbf{37.75}$}  \\ \hline
Caltech102-15 & 31.81  & 32.54  & 21.63  & 29.56  & 26.21  & 27.12  & - & 34.25  & 28.89  & $\mathbf{35.62}$  & \underline{$\mathbf{38.17}$}  \\ \hline
Caltech102-20 & 31.91  & 32.78  & 20.39  & 29.31  & 26.13  & 25.59  & - & 34.66  & 27.84  & $\mathbf{36.47}$  & \underline{$\mathbf{37.06}$}  \\ \hline
Caltech102-25 & 30.41  & 31.76  & 18.00  & 29.00 & 23.45  & 25.80  & - & 32.20  & 28.75  & 34.43  & \underline{$\mathbf{36.20}$}  \\ \hline
Caltech102-30 & 28.71  & 31.35  & 18.04  & 28.63 & 23.50  & 24.15  & - & 33.20  & 26.76  & $\mathbf{34.77}$  & \underline{$\mathbf{35.16}$}  \\ \hline

\end{tabular}}
}
\end{center}
\end{table*}

\begin{table*}[]
\caption{Purity comparison of different Graph-based multi-view clustering algorithms on twelve benchmark data sets.The best result is highlighted with format mean(std).}
\resizebox{\textwidth}{!}{
\label{mgc purity result}
\begin{tabular}{c|cccccccc}
\midrule
Dataset       & SC            & Co-reg        & LMSC          & MVGL          & MvWCMD        & GMC           & LF-MVC-GAM & LF-MVC-LAM      \\
\hline
MSRC-v1       & 0.699 (0.007) & 0.559 (0.052) & 0.694 (0.053) & 0.757 (0.000) & 0.766 (0.034) & 0.781 (0.000) & 0.810 (0.000) & \textbf{0.843 (0.000)} \\
Caltech101-7  & 0.799 (0.000) & 0.825 (0.014) & 0.855 (0.012) & 0.858 (0.000) & 0.855 (0.003) & 0.897 (0.000) & 0.888 (0.000) & \textbf{0.914 (0.000)} \\
Caltech101-20 & 0.744 (0.007) & 0.746 (0.006) & 0.783 (0.003) & 0.756 (0.000) & 0.664 (0.033) & 0.750 (0.000) & 0.785 (0.000) & \textbf{0.810 (0.000)} \\
ALOI          & 0.651 (0.015) & 0.661 (0.016) & 0.716 (0.003) & 0.663 (0.000) & 0.537 (0.004) & 0.682 (0.000) & 0.685 (0.000) & \textbf{0.736 (0.000)} \\
\hline
\end{tabular}}
\end{table*}

\begin{table*}[]
\caption{Time-consuming comparison (in seconds) of kernel-based multi-view clustering algorithms on large-scale benchmarks and and '-' means the out-of-memory failure.}
\label{large time result}
\resizebox{\textwidth}{!}{
\begin{tabular}{c|cccccccccc}
\midrule
        & \multicolumn{2}{c}{CSRC} & \multicolumn{2}{c}{MKAM} & \multicolumn{2}{c}{MKKM-MR} & \multicolumn{2}{c}{LFA-GAM} & \multicolumn{2}{c}{LFA-LAM}      \\
\hline
Dataset & Time       & Speed Up    & Time        & Speed Up   & Time        & Speed Up      & Time           & Speed Up   & Time             & Speed Up      \\
\hline
Reuters & -    & -        & 22147.81    & 0.01$\times$       & 280.48      & 1$\times$& \textbf{76.43} & \textbf{4$\times$} & \textbf{118.318} & \textbf{2.37$\times$} \\
MNIST   & -           &  -           &   -          &  -          &  -           &  -             & 246.01         & 1.14$\times$       & 280.34           & 1 $\times$\\
\hline           
\end{tabular}}
\end{table*}

\section{More experimental results}
In \cite{DBLP:conf/ijcai/NieLL17}, a self-weighted strategy is proposed to adaptively update the view coefficients and output the ideal low-rank graph similarity matrix for spectral clustering (SWMC). Moreover, \cite{2021Fast} proposes a unified framework to jointly conduct prototype graph fusion and directly output the clustering labels. The main difference between the referenced papers and our proposed methods is that \cite{DBLP:conf/ijcai/NieLL17,2021Fast} follow the similarity fusion while ignore the partition matrix information. Both of them optimize the optimal similarity matrix by assuming the linear combination of multiple static or dynamic graphs. Our experimental results in Table \ref{compareresult} clearly show the advantages of partition level information fusion comparing to similarity fusion.\\
\begin{table*}[htbp]
\begin{center}
\caption{Experimental results on benchmark datasets with SWMC and Ours.}
\label{compareresult}
\begin{tabular}{|c|c|c|c|c|c|c|}
\hline
Datasets   & SWMC  & Ours           & SWMC   & Ours            & SWMC     & Ours             \\ \hline
\multicolumn{3}{|c|}{ACC}           & \multicolumn{2}{c|}{NMI} & \multicolumn{2}{c|}{Purity} \\ \hline
Caltech-5  & 12.16 & \textbf{41.37} & 22.07  & \textbf{72.85}  & 21.37    & \textbf{41.96}   \\ \hline
Caltech-10 & 8.24  & \textbf{35.59} & 18.6   & \textbf{64.34}  & 13.53    & \textbf{37.75}   \\ \hline
Caltech-15 & 6.73  & \textbf{36.61} & 13     & \textbf{61.05}  & 10.46    & \textbf{38.17}   \\ \hline
Caltech-20 & 6.67  & \textbf{35.9}  & 12.05  & \textbf{57.47}  & 9.46     & \textbf{37.06}   \\ \hline
CCV        & 10.59 & \textbf{29.79} & 0.35   & \textbf{22.1}   & 10.75    & \textbf{30.12}   \\ \hline
Flower17   & 12.72 & \textbf{62.35} & 19.02  & \textbf{59.39}  & 13.31    & \textbf{63.97}   \\ \hline
mfeat      & 42.1  & \textbf{95.9}  & 43.26  & \textbf{91.25}  & 99.75    & \textbf{95.9}    \\ \hline
AR10P      & 13.85 & \textbf{53.08} & 5.42   & \textbf{53.11}  & 15.48    & \textbf{53.08}   \\ \hline
\end{tabular}
\end{center}
\end{table*}

\subsection{Parameter Sensitivity of LF-MVC-LAM}

Notice that LF-MVC-LAM introduces two hyper-meters, i.e., the trade-off coefficient $\lambda$ and the neighborhood number $\tau$ which denotes $\tau$-nearest neighbors for the sample. To test the sensitivity of the proposed algorithm against these two parameters, we fix one parameter and tune the other in a large range and the results are reported in Figure. \ref{LMVC-LFA SensitivityFig}. The $\lambda$ is chosen ranging from $\left[2^{-5}, 2^{-4}, \cdots, 2^{5}\right]$ and $\tau$ is carefully searched in the range of $[0.1,0.2,\cdots,0.9,1]*n$. The Figure \ref{LMVC-LFA SensitivityFig} shows the sensitivity experimental results on AR10P, CCV, Flower17, Mfeat, Plant and YALE.

From these figures, we observe that: i) all the two hyper-parameters are effective in improving the clustering performance; ii) LF-MVC-LAM is practically stable against the these parameters that it achieves competitive performance in a wide range of parameter settings; iii) the acc first increases to a high value and generally maintains it up to slight variation with the increasing value of $\lambda$. LF-MVC-LAM demonstrates stable performance across a wide range of $\lambda$. iiii) the proposed algorithm is relatively sensitive to the neighbor numbers $\tau$ which relatively reflects the inner data structure of data themselves. However, it still out-performances the second-best algorithm in most of the benchmark datasets.

\begin{figure*}[!htbp]
\centering
\subfloat[AR10P($\lambda$,$\tau$)]{\includegraphics[width = 0.3\textwidth]{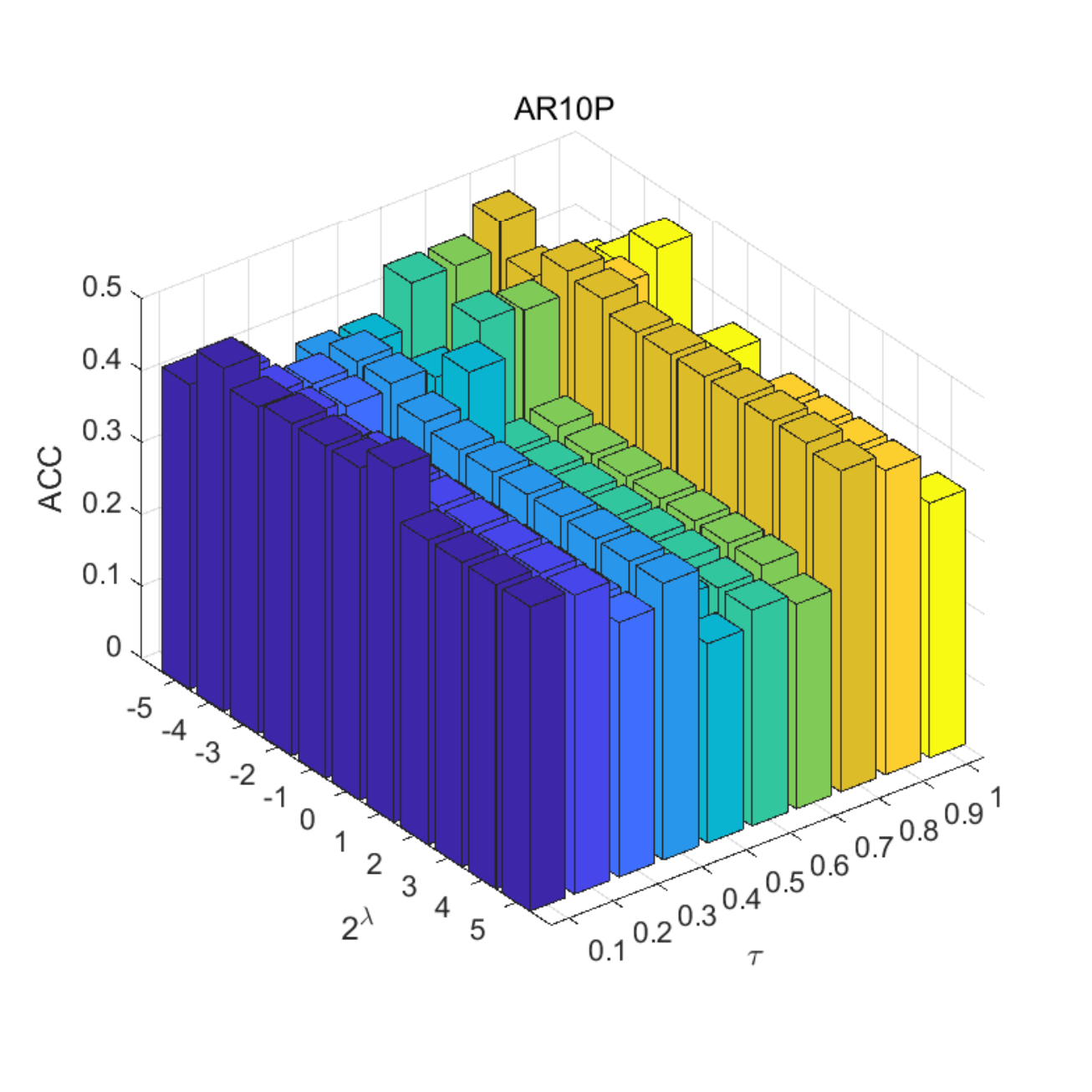}\label{flower17_paraMRTauACC}}
\subfloat[CCV($\lambda$,$\tau$)]{\includegraphics[width = 0.3\textwidth]{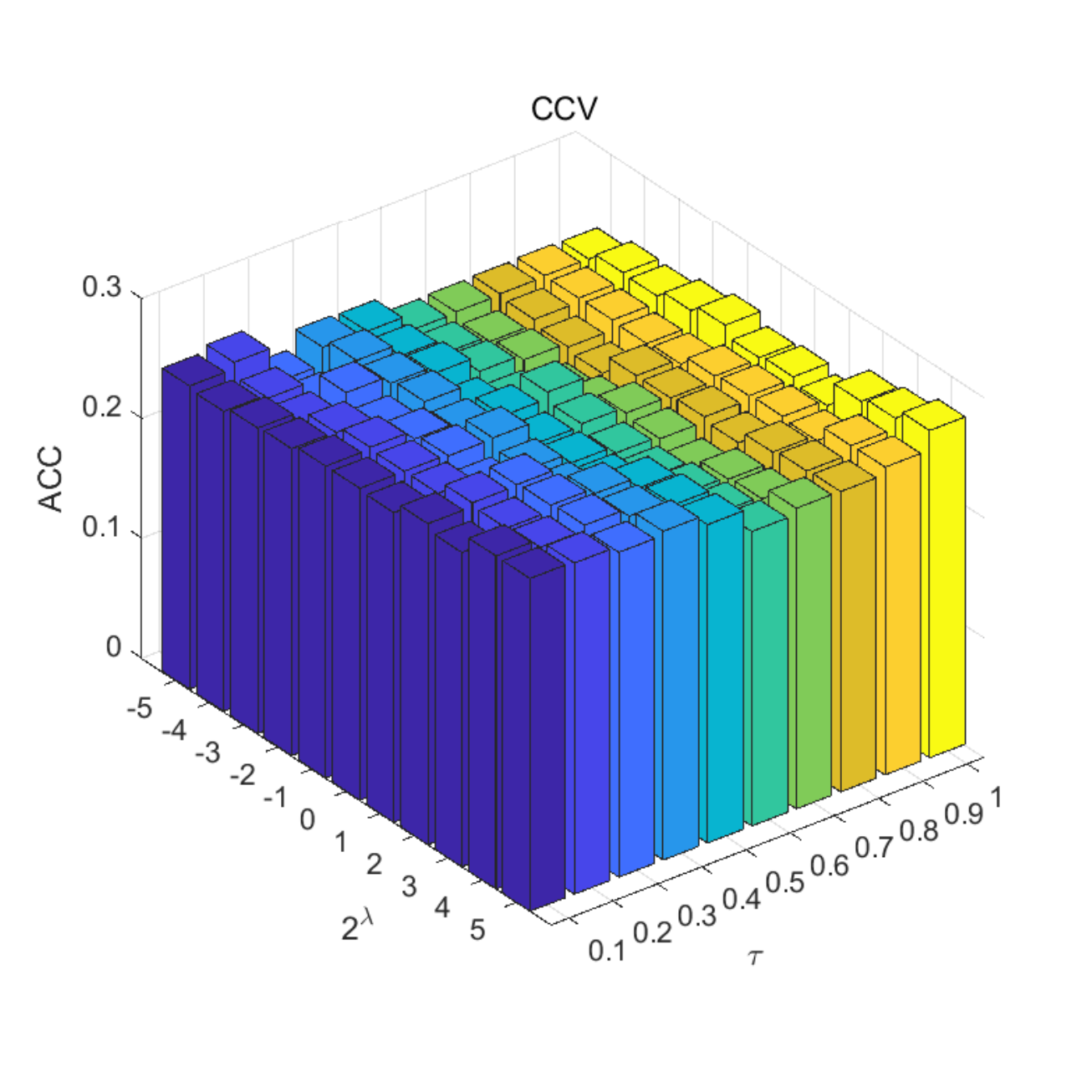}\label{flower17_paraMRTauNMI}}
\subfloat[Flower17($\lambda$,$\tau$)]{\includegraphics[width = 0.3\textwidth]{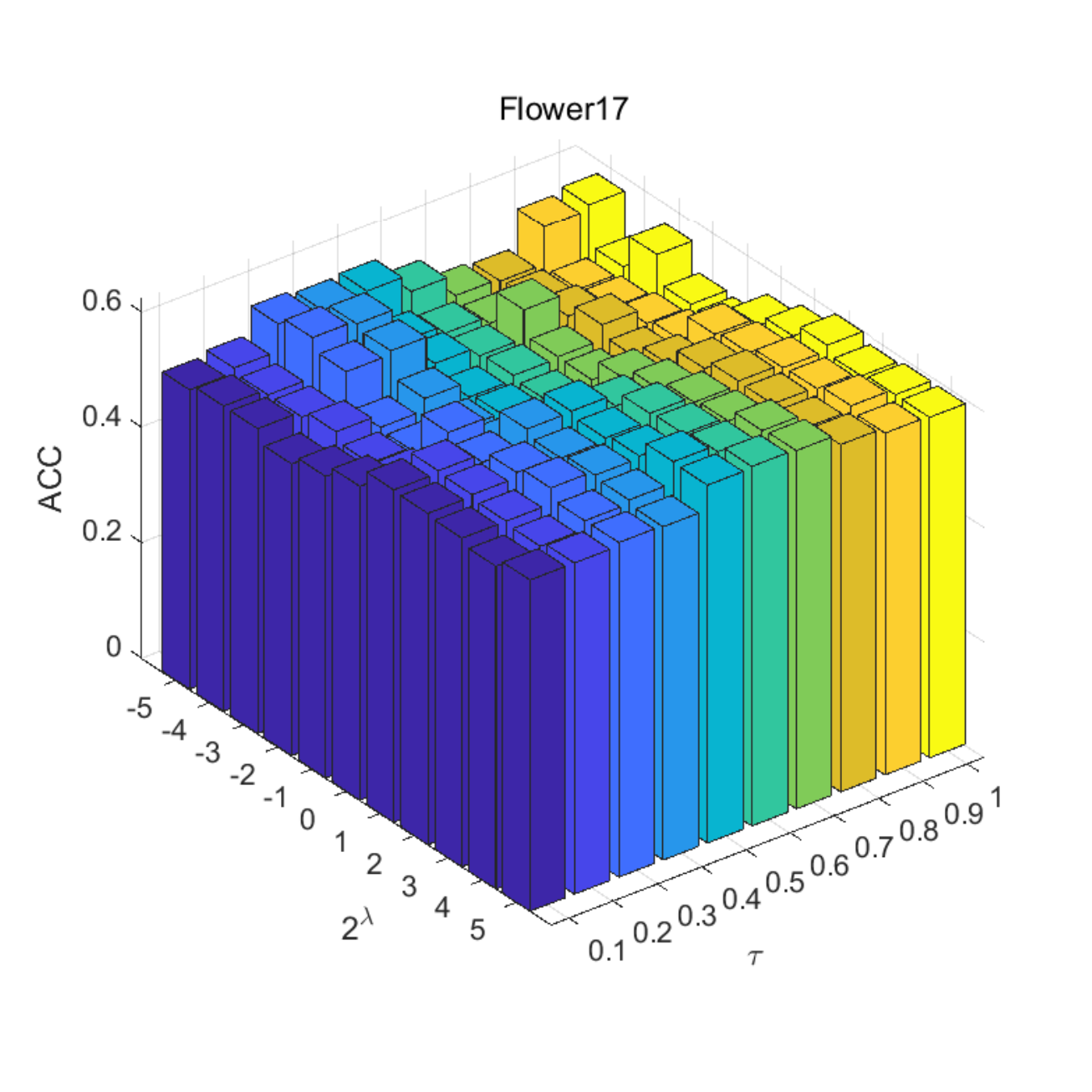}\label{flower17_paraMRLambdaACC}}\\[-10pt]
\subfloat[Mfeat($\lambda$,$\tau$)]{\includegraphics[width = 0.3\textwidth]{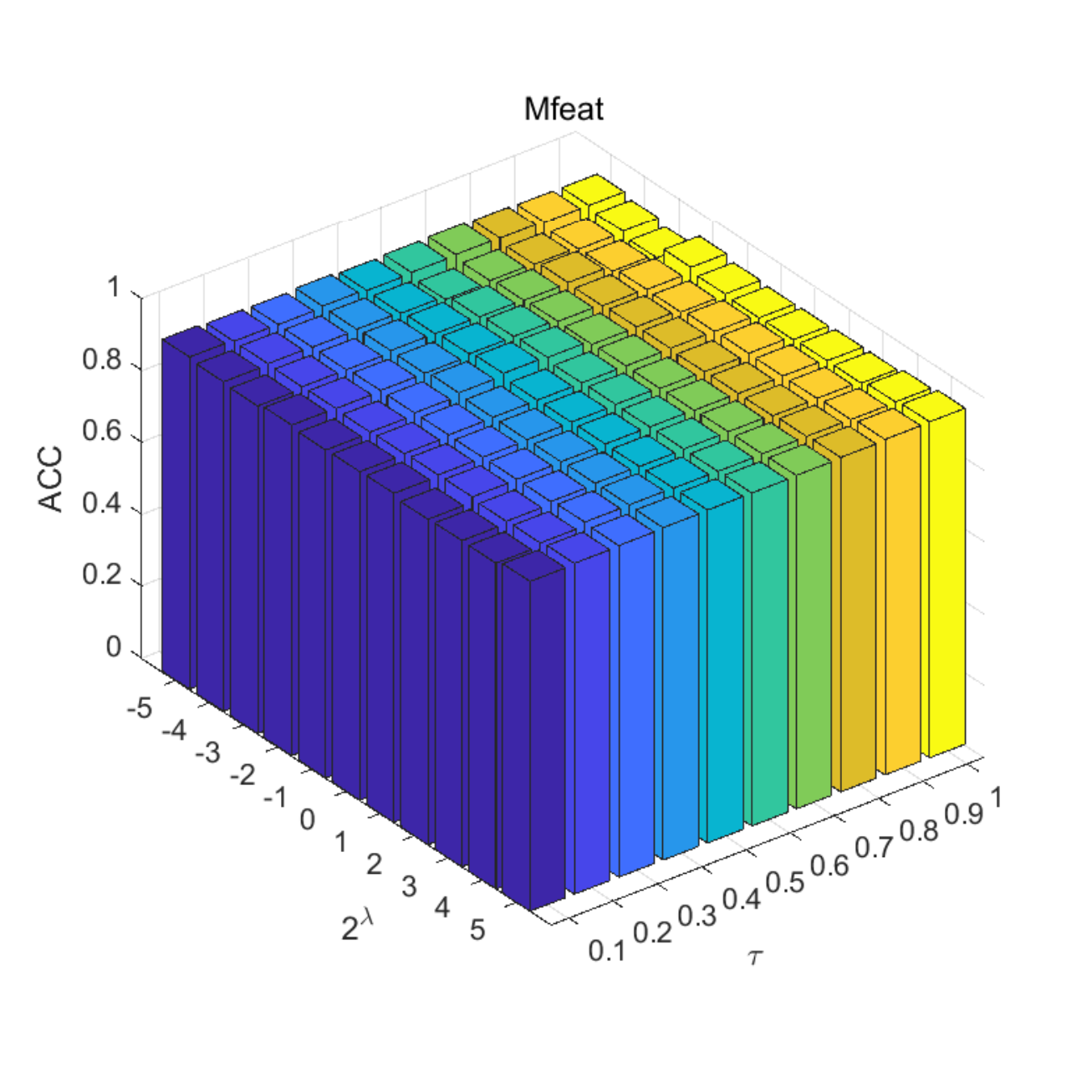}\label{flower17_paraMRLambdaNMI}}
\subfloat[Plant($\lambda$,$\tau$)]{\includegraphics[width = 0.3\textwidth]{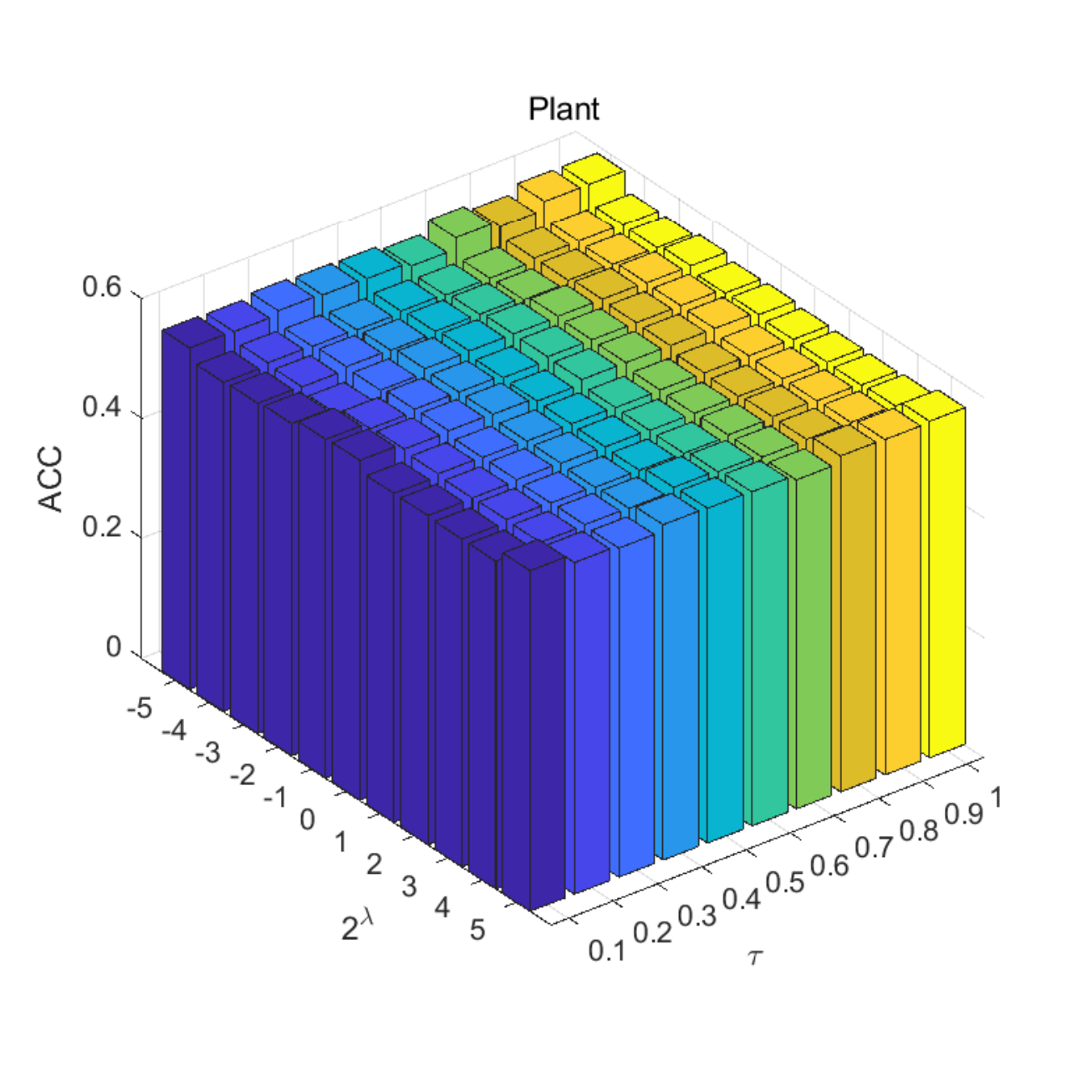}\label{flower17_paraMRLambdaNMI}}
\subfloat[YALE($\lambda$,$\tau$)]{\includegraphics[width = 0.3\textwidth]{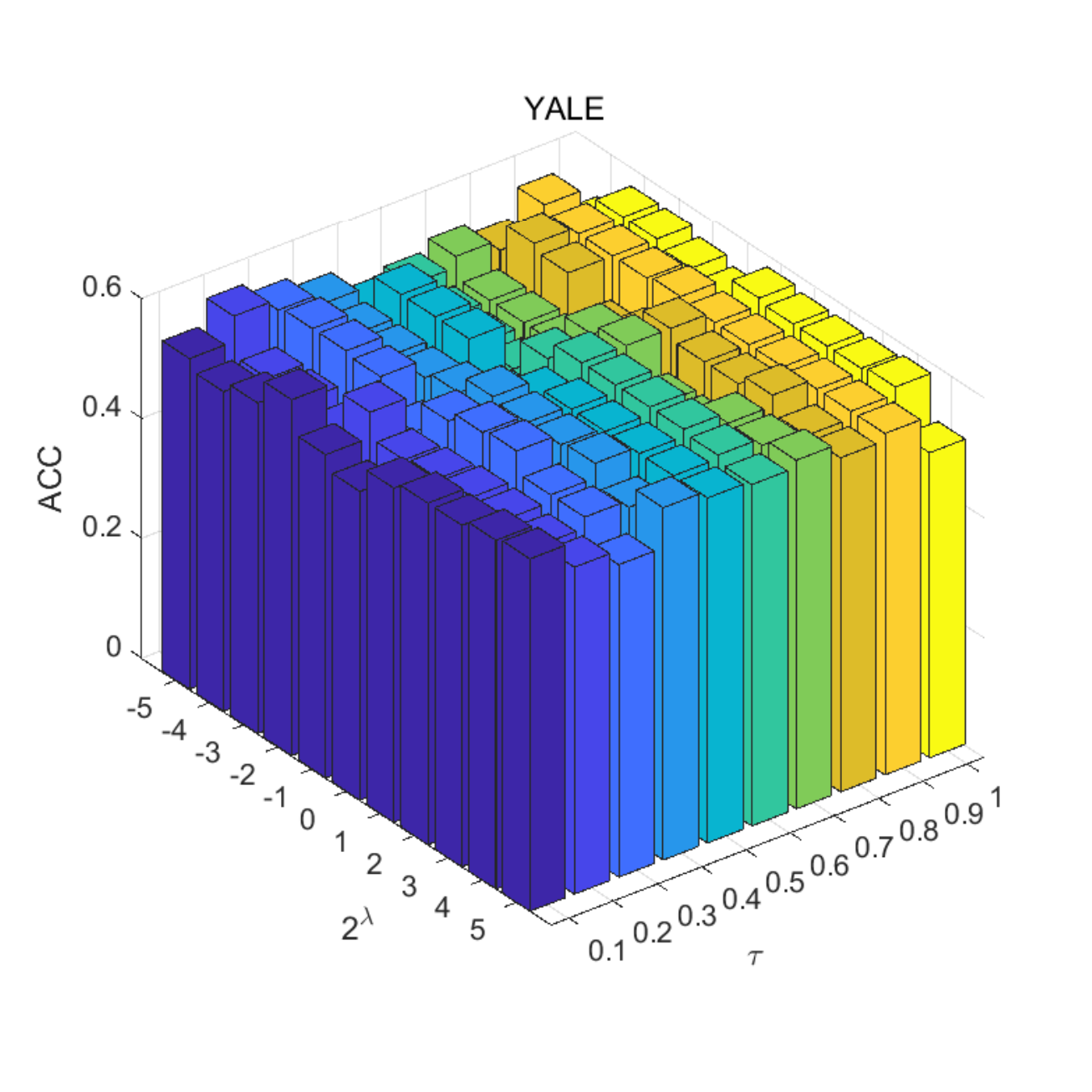}\label{flower17_paraMRLambdaNMI}}
\caption{{The sensitivity of the proposed LF-MVC-LAM with the variation of $\lambda$ and $\tau$ on six benchmark datasets. }}\label{LMVC-LFA SensitivityFig}
\end{figure*}

\section{convergence}
The examples of the evolution of the objective value on the experimental results are shown in Figure \ref{MVC-LFA convergence1} and \ref{LMVC-LFA convergence1}.

\begin{figure*}[!htbp]
\centering
\subfloat[Caltech102-25]{\includegraphics[width = 0.25\textwidth]{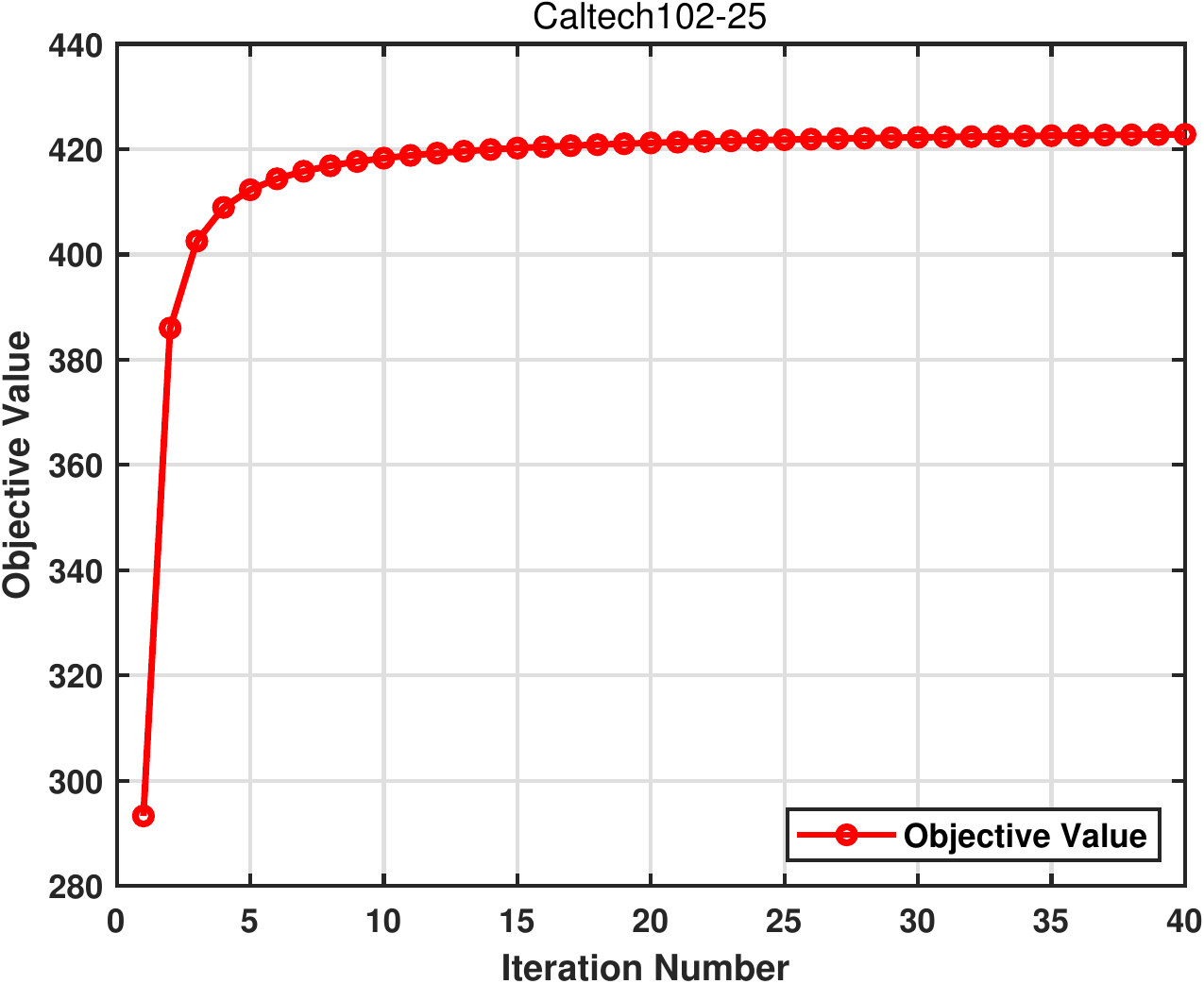}\label{flower17_paraMRLambdaACC}}
\subfloat[Caltech102-30]{\includegraphics[width = 0.25\textwidth]{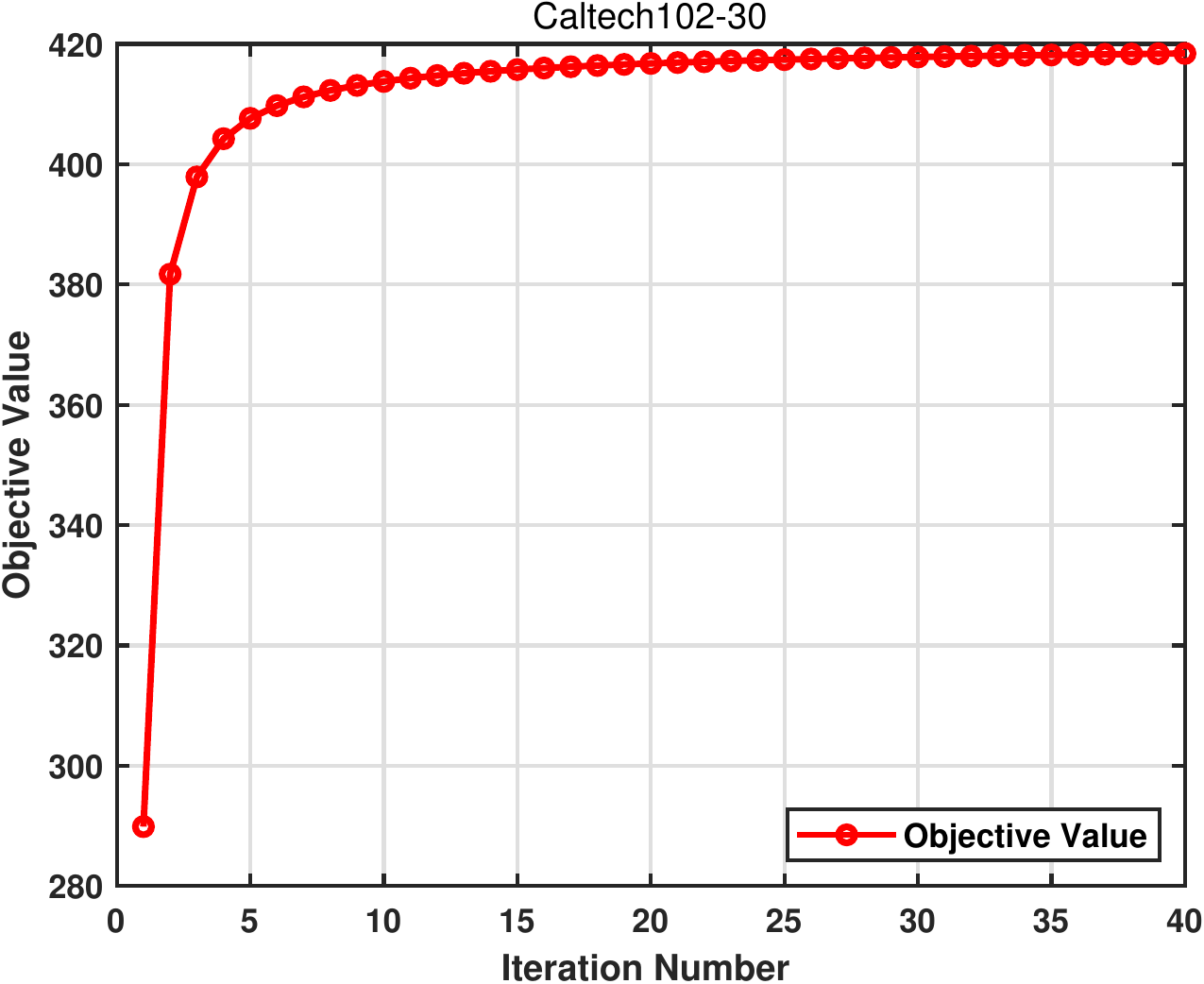}\label{flower17_paraMRLambdaNMI}}
\subfloat[AR10P]{\includegraphics[width = 0.25\textwidth]{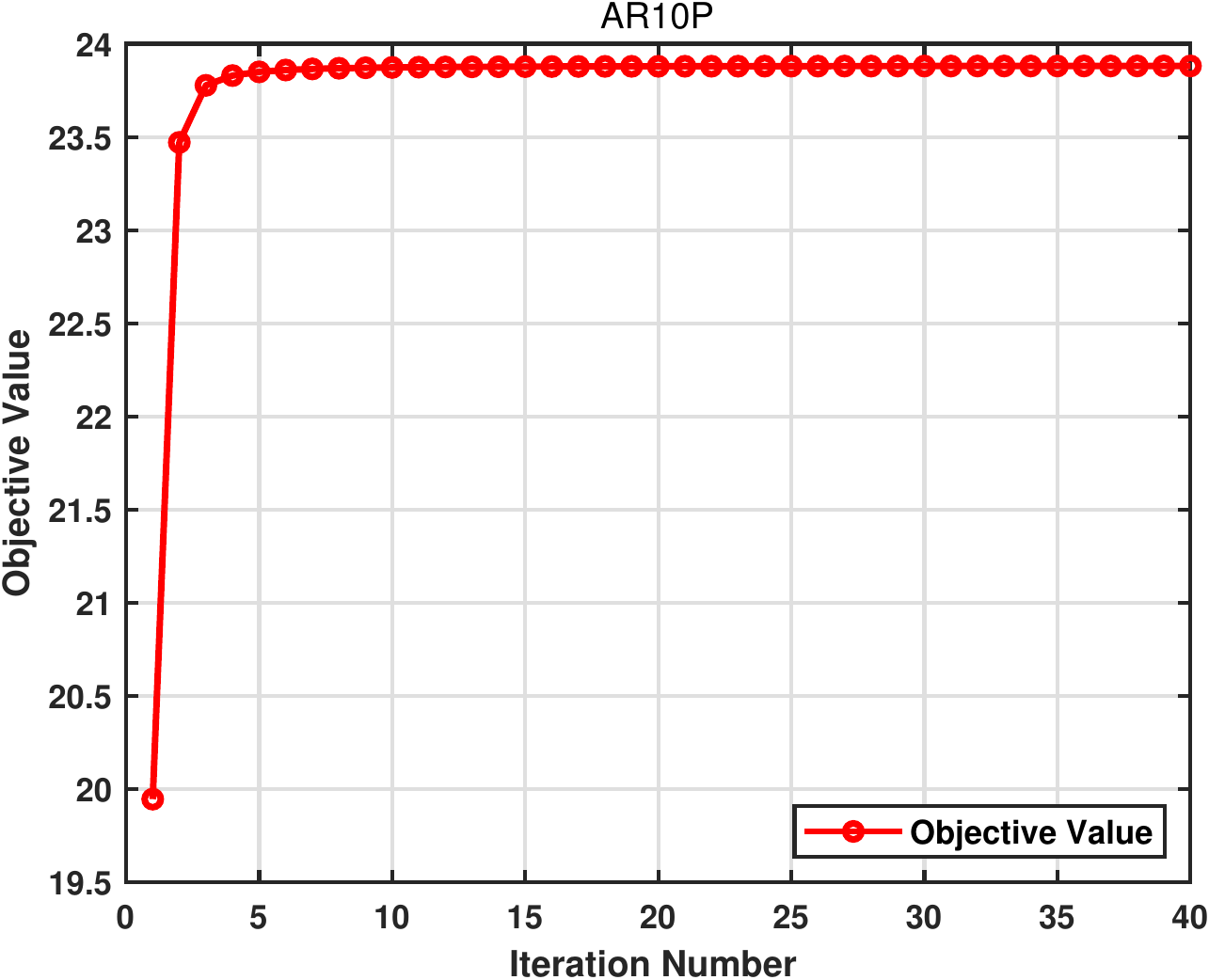}\label{flower17_paraMRLambdaNMI}}
\subfloat[Mfeat]{\includegraphics[width = 0.25\textwidth]{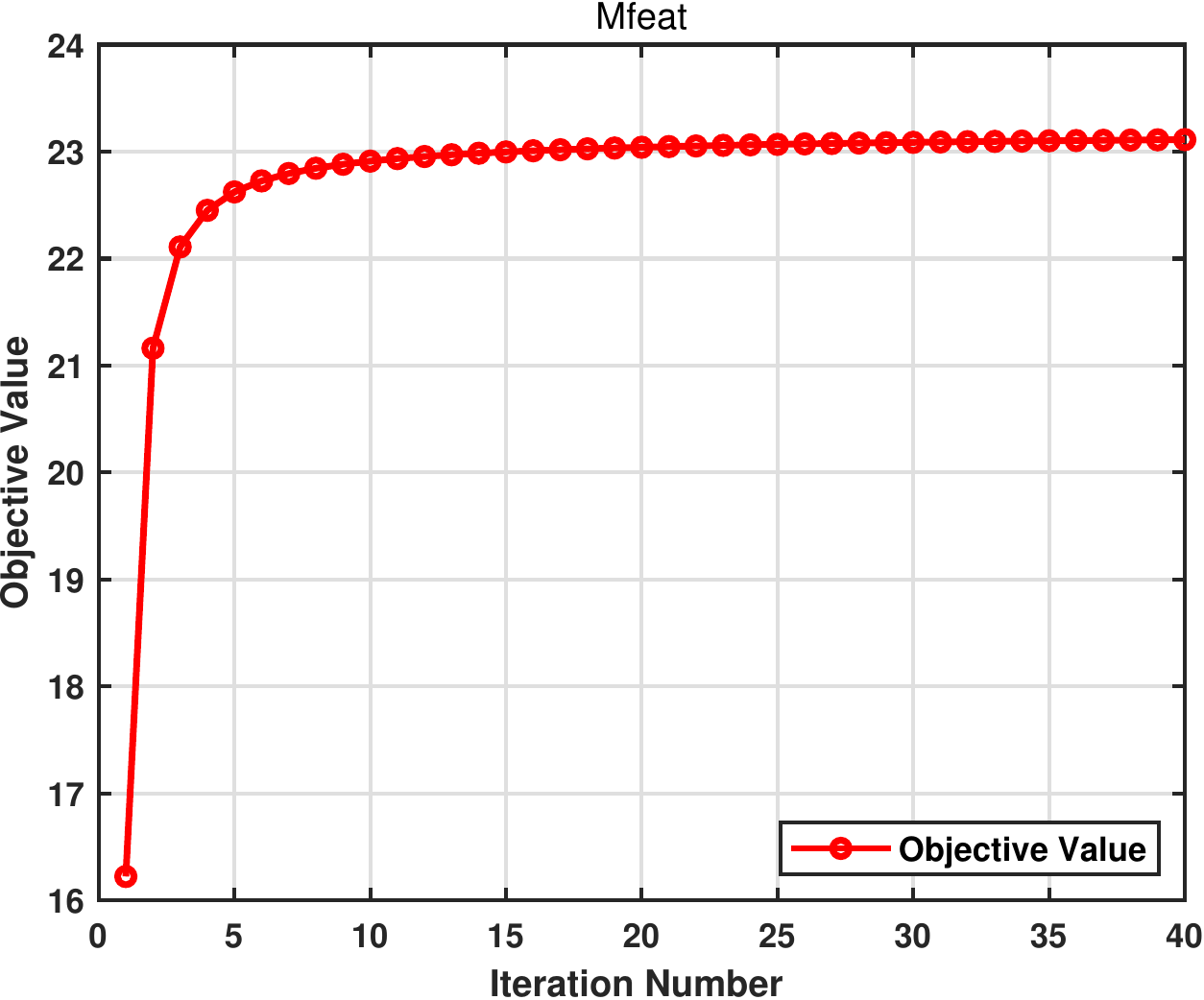}\label{flower17_paraMRLambdaNMI}}\\
\subfloat[YALE]{\includegraphics[width = 0.25\textwidth]{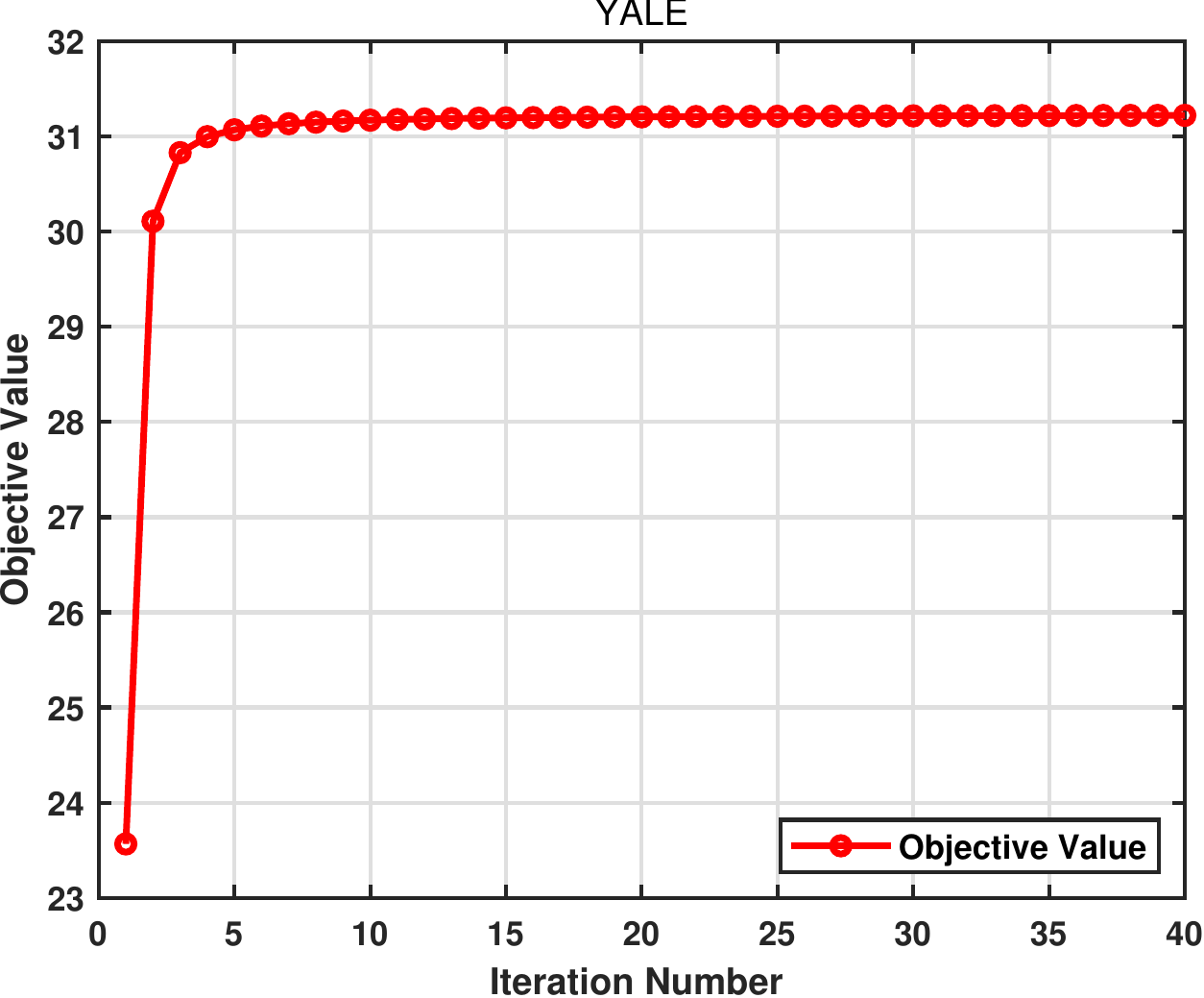}\label{flower17_paraMRLambdaNMI}}
\subfloat[Plant]{\includegraphics[width = 0.25\textwidth]{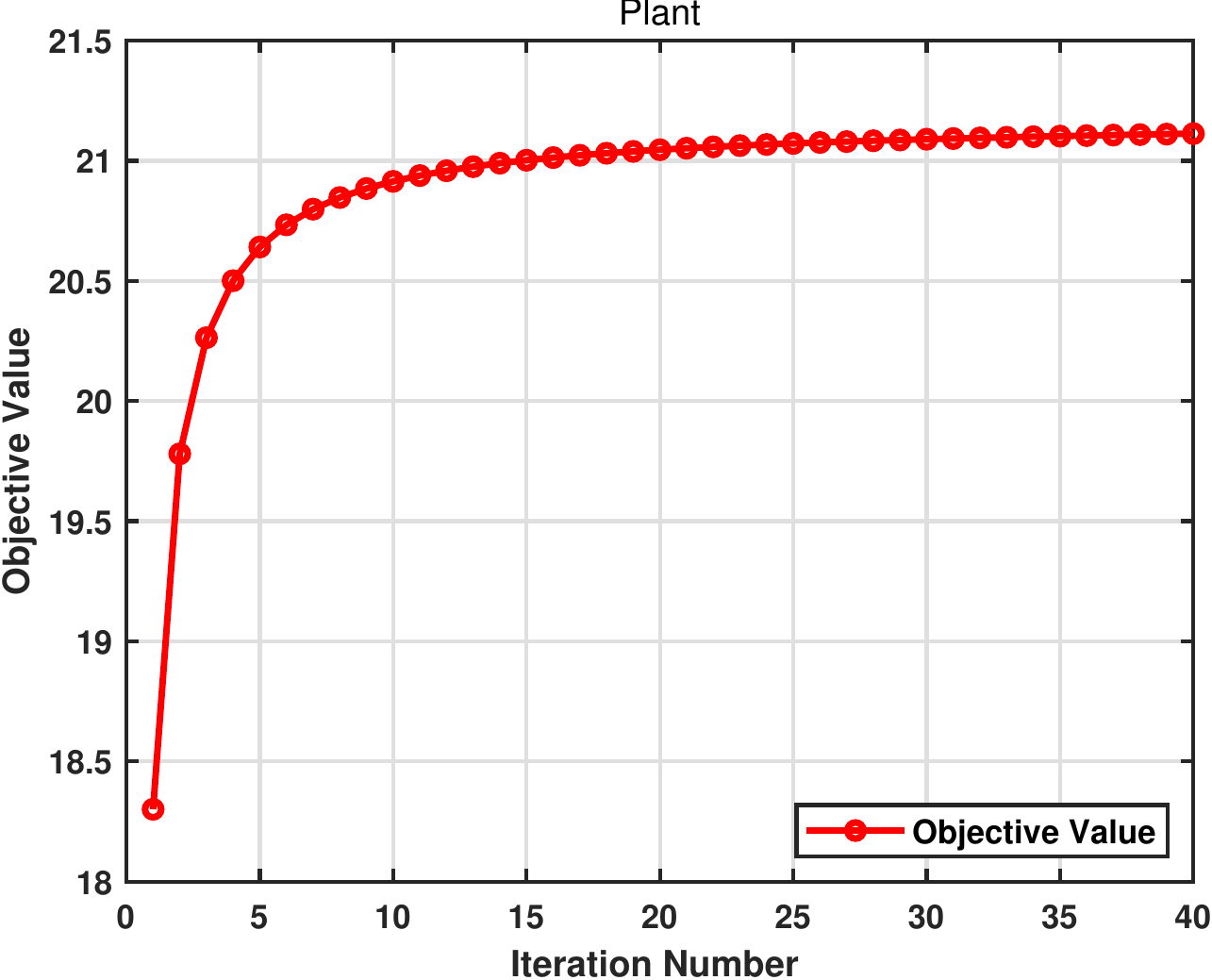}\label{flower17_paraMRLambdaNMI}}
\subfloat[CCV]{\includegraphics[width = 0.25\textwidth]{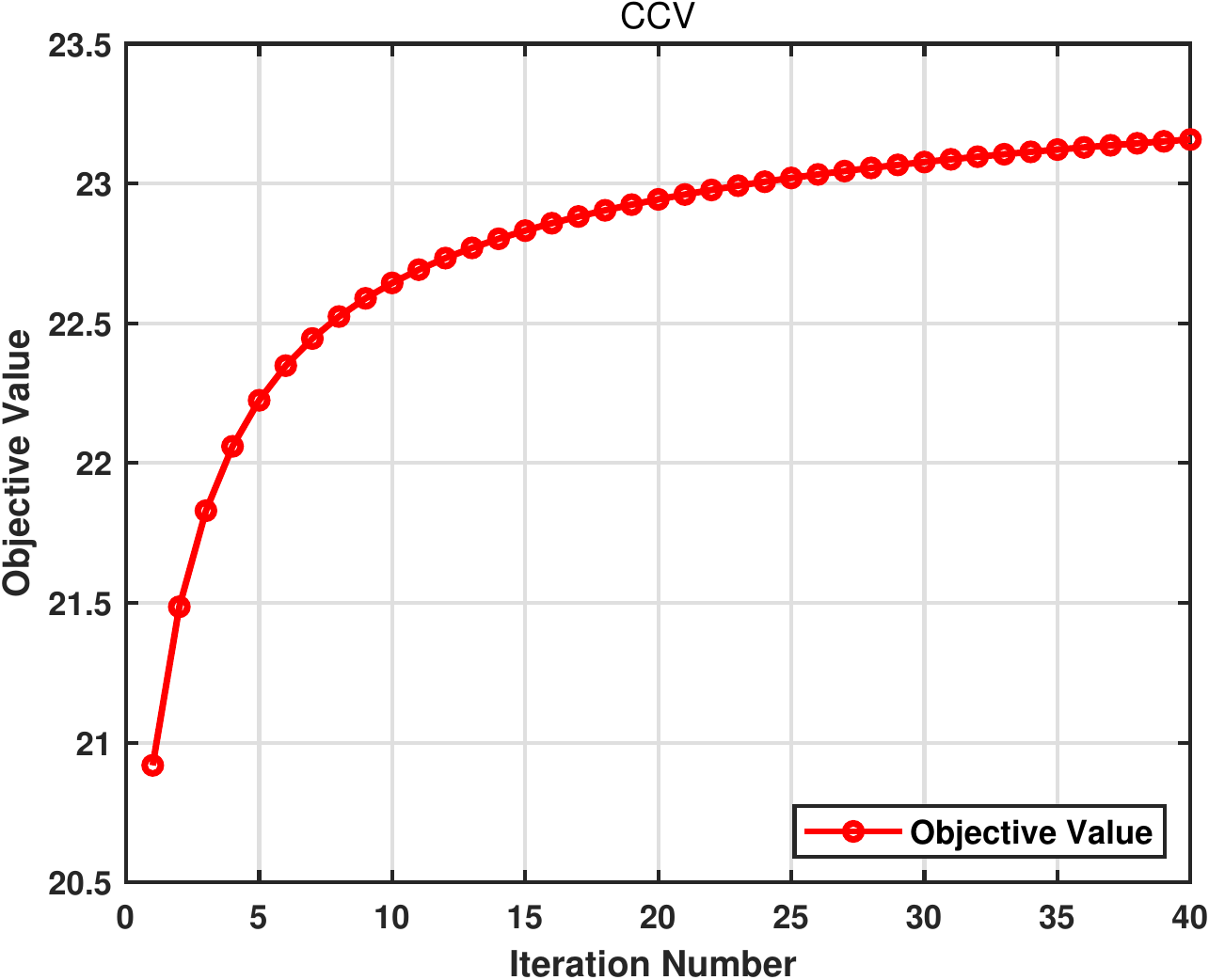}\label{flower17_paraMRLambdaNMI}}
\subfloat[Flower17]{\includegraphics[width = 0.25\textwidth]{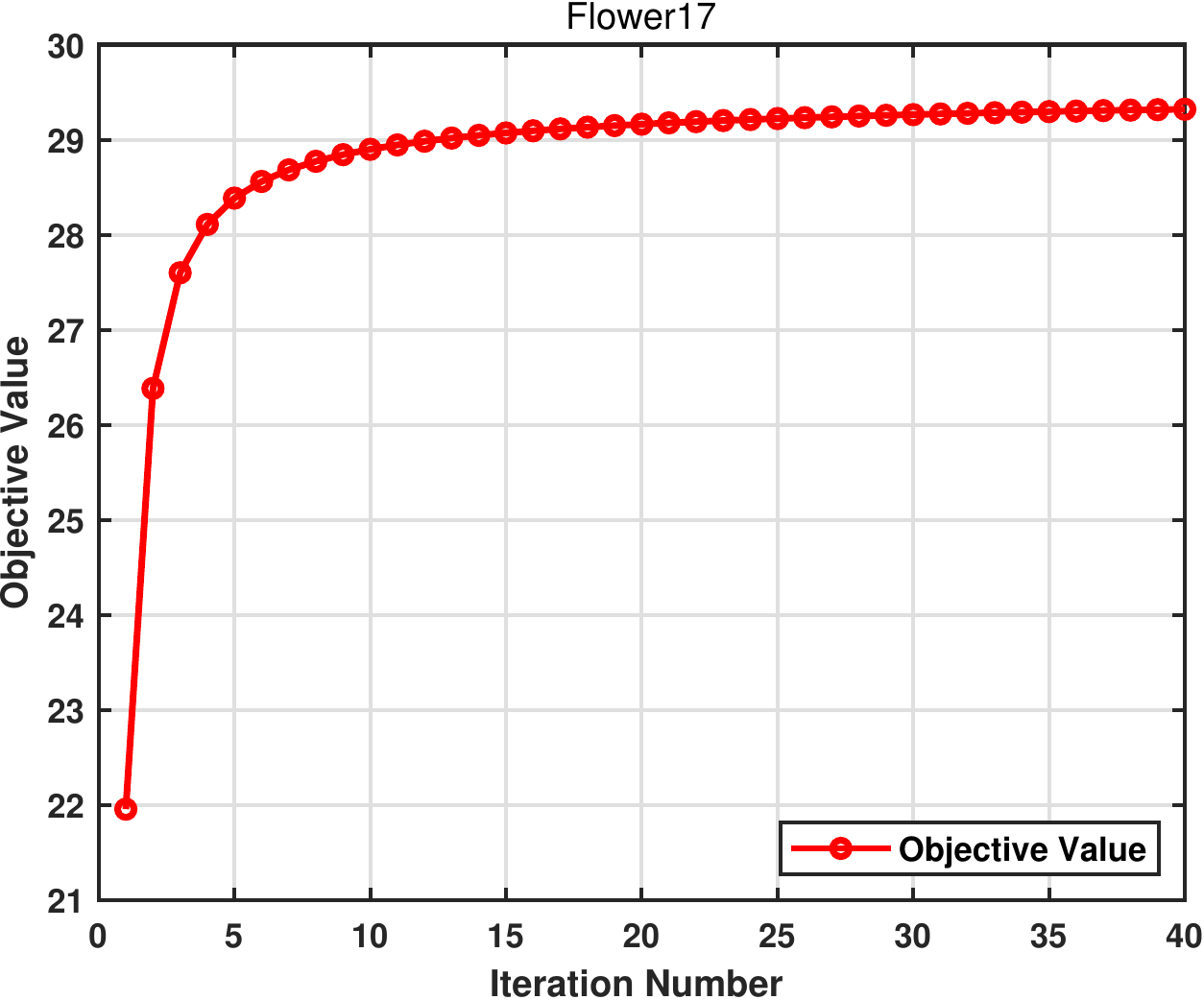}\label{flower17_paraMRLambdaNMI}}
\caption{{The convergence of the proposed LFA-GAM other eight datasets. }}\label{MVC-LFA convergence1}
\end{figure*}

\begin{figure*}[!htbp]
\centering
\subfloat[Caltech102-25]{\includegraphics[width = 0.25\textwidth]{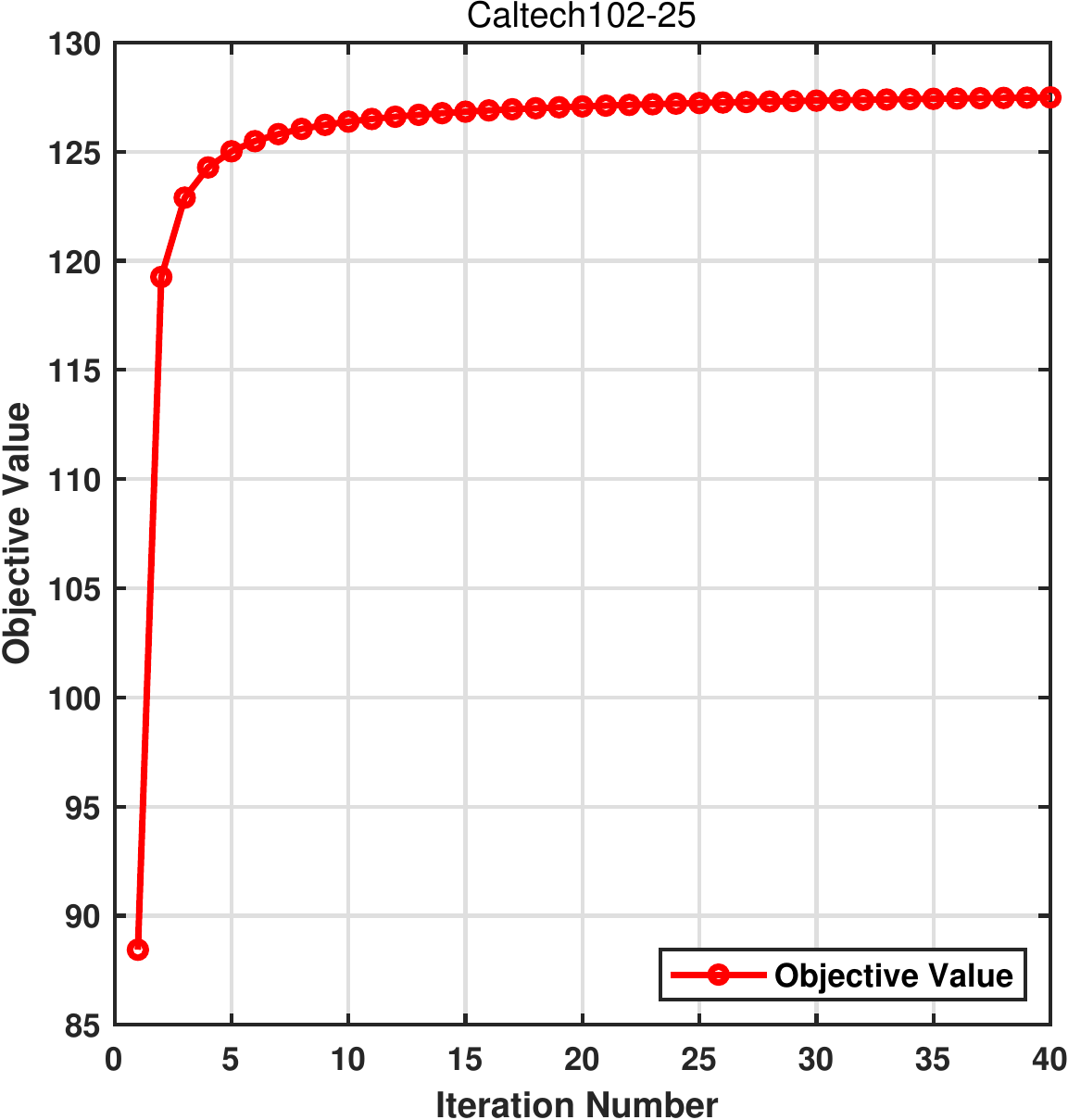}\label{flower17_paraMRLambdaACC}}
\subfloat[Caltech102-30]{\includegraphics[width = 0.25\textwidth]{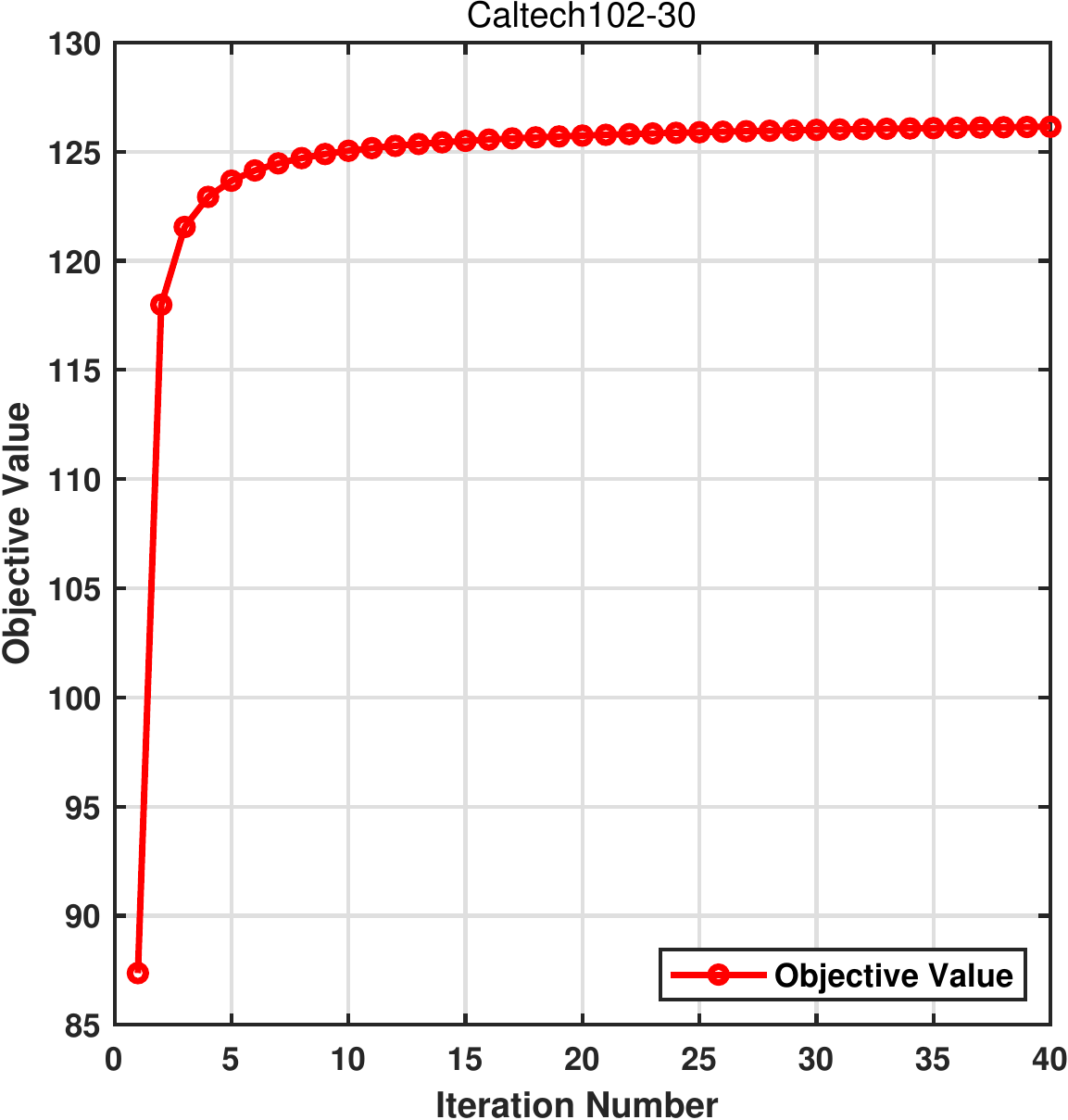}\label{flower17_paraMRLambdaNMI}}
\subfloat[AR10P]{\includegraphics[width = 0.25\textwidth]{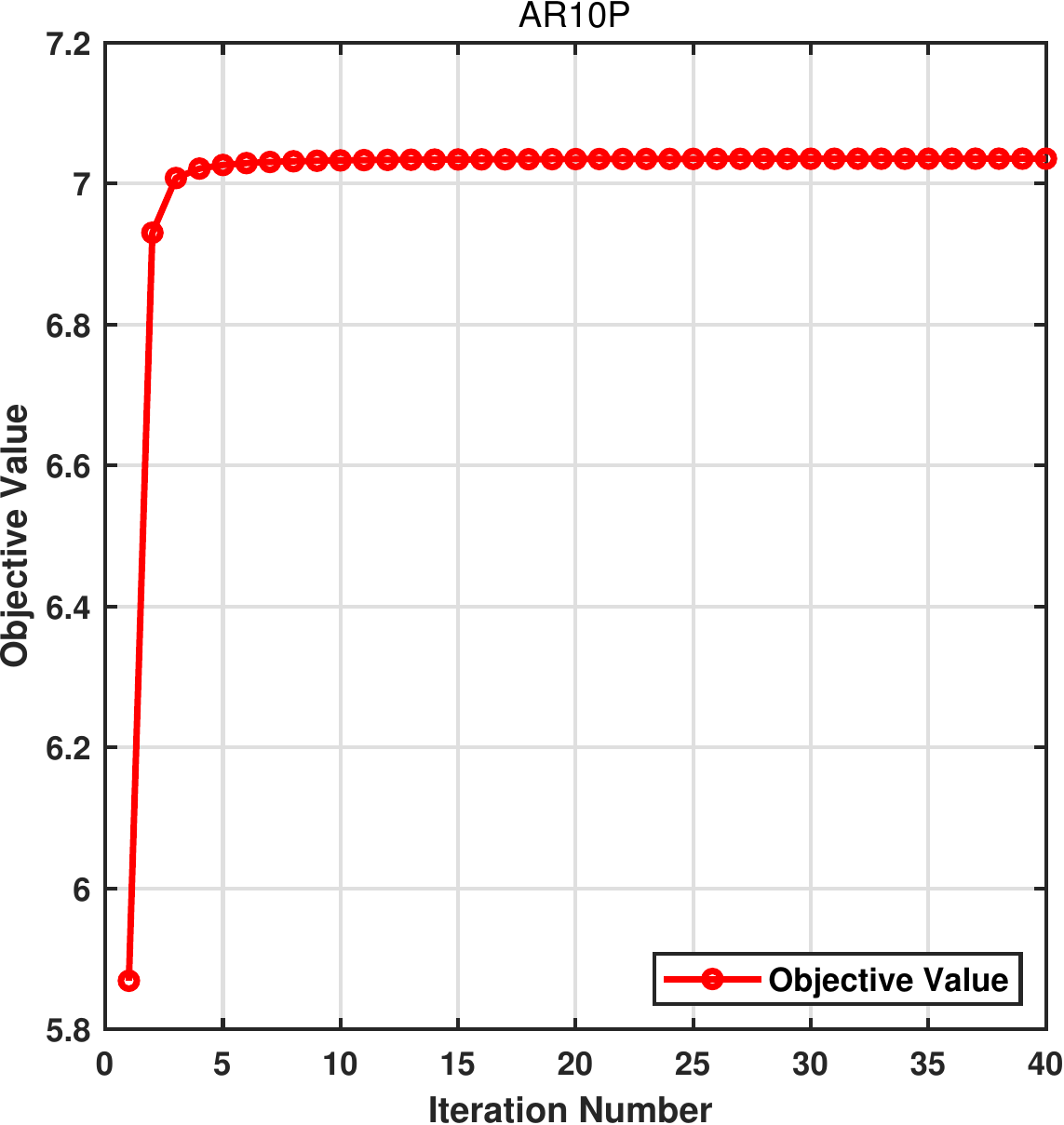}\label{flower17_paraMRLambdaNMI}}
\subfloat[Mfeat]{\includegraphics[width = 0.25\textwidth]{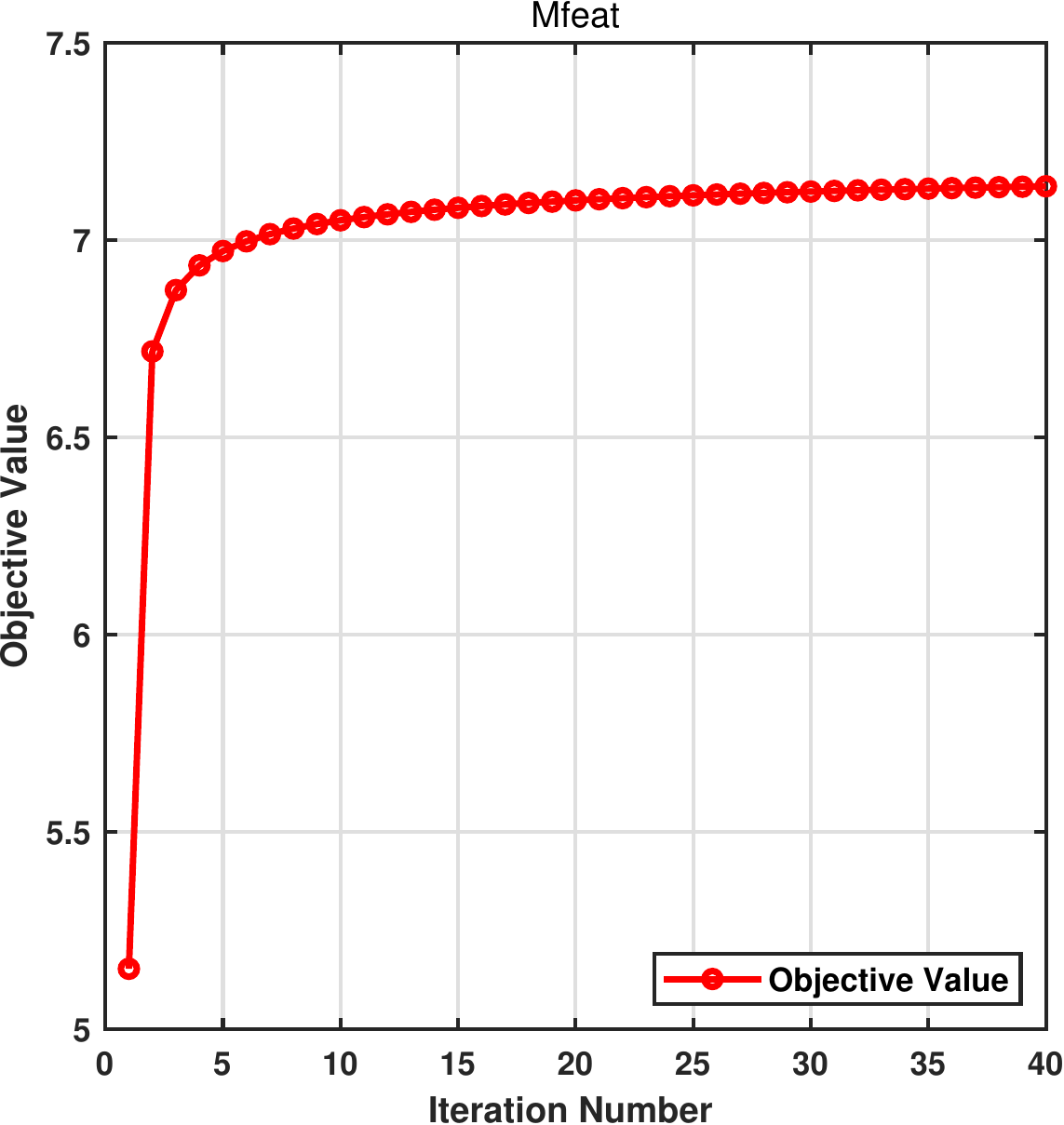}\label{flower17_paraMRLambdaNMI}}\\
\subfloat[YALE]{\includegraphics[width = 0.25\textwidth]{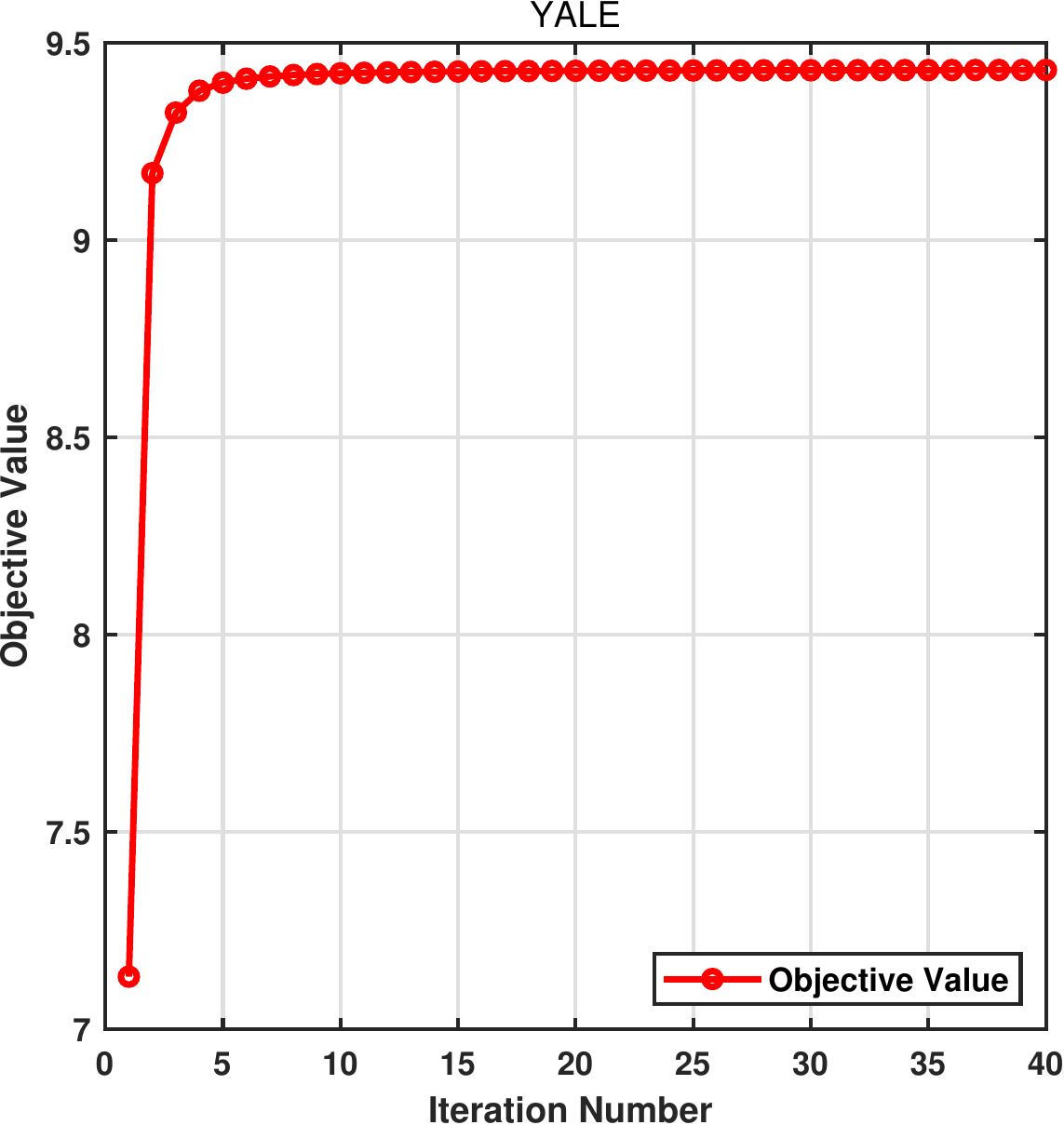}\label{flower17_paraMRLambdaNMI}}
\subfloat[Plant]{\includegraphics[width = 0.25\textwidth]{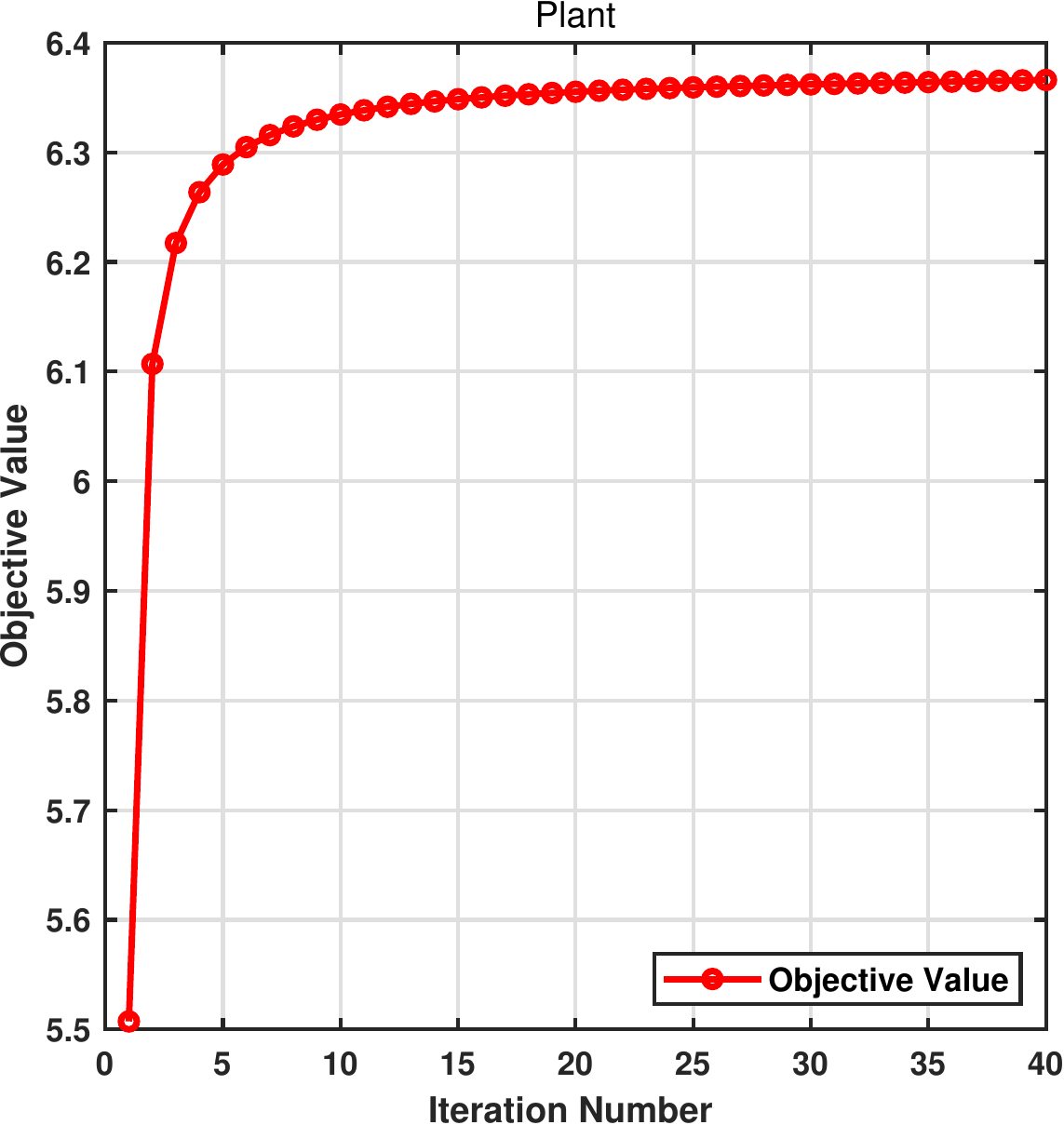}\label{flower17_paraMRLambdaNMI}}
\subfloat[CCV]{\includegraphics[width = 0.25\textwidth]{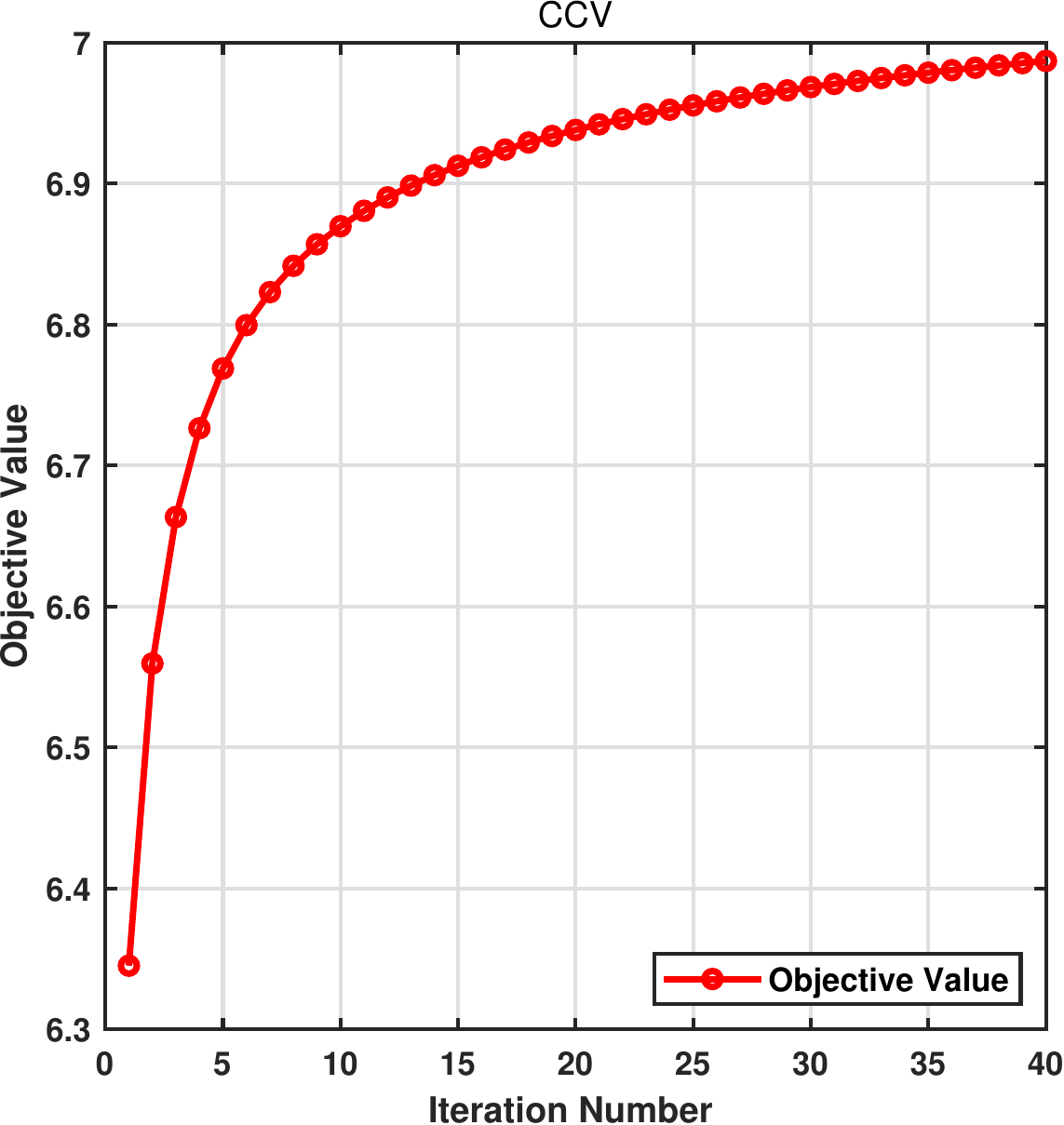}\label{flower17_paraMRLambdaNMI}}
\subfloat[Flower17]{\includegraphics[width = 0.25\textwidth]{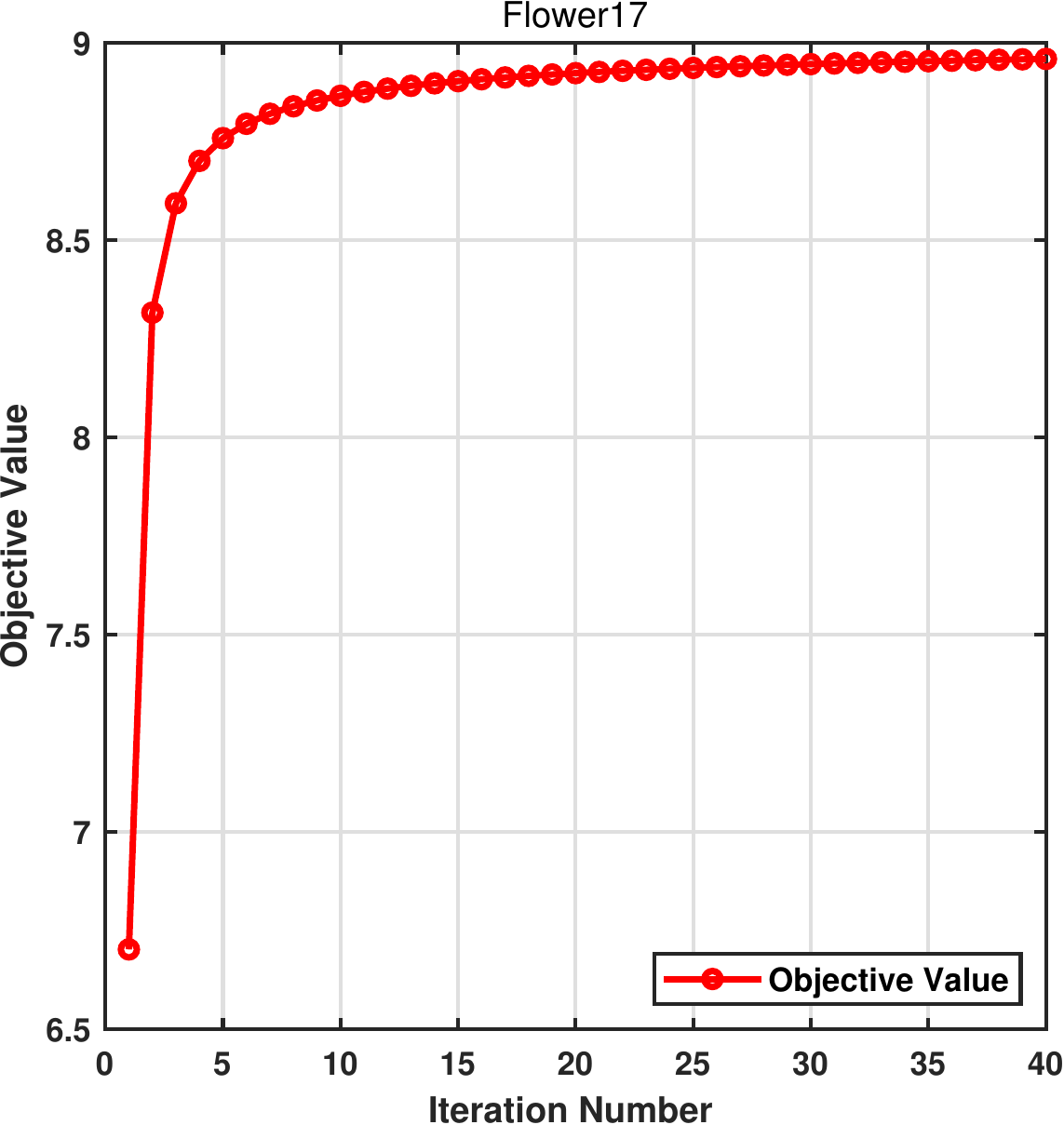}\label{flower17_paraMRLambdaNMI}}
\caption{{The convergence of the proposed LFA-LAM other eight datasets. }}\label{LMVC-LFA convergence1}
\end{figure*}

\section{Evolution of the learned affinity matrices}
It can be observed form Figure \ref{evolution_s}, the learned affinity matrices show clearer block clustering structure with the variation of iterations. 
	
\begin{figure*}[!htbp]
\centering
\subfloat[Mfeat $1^{st}$ iteration]{\includegraphics[width = 0.25\textwidth]{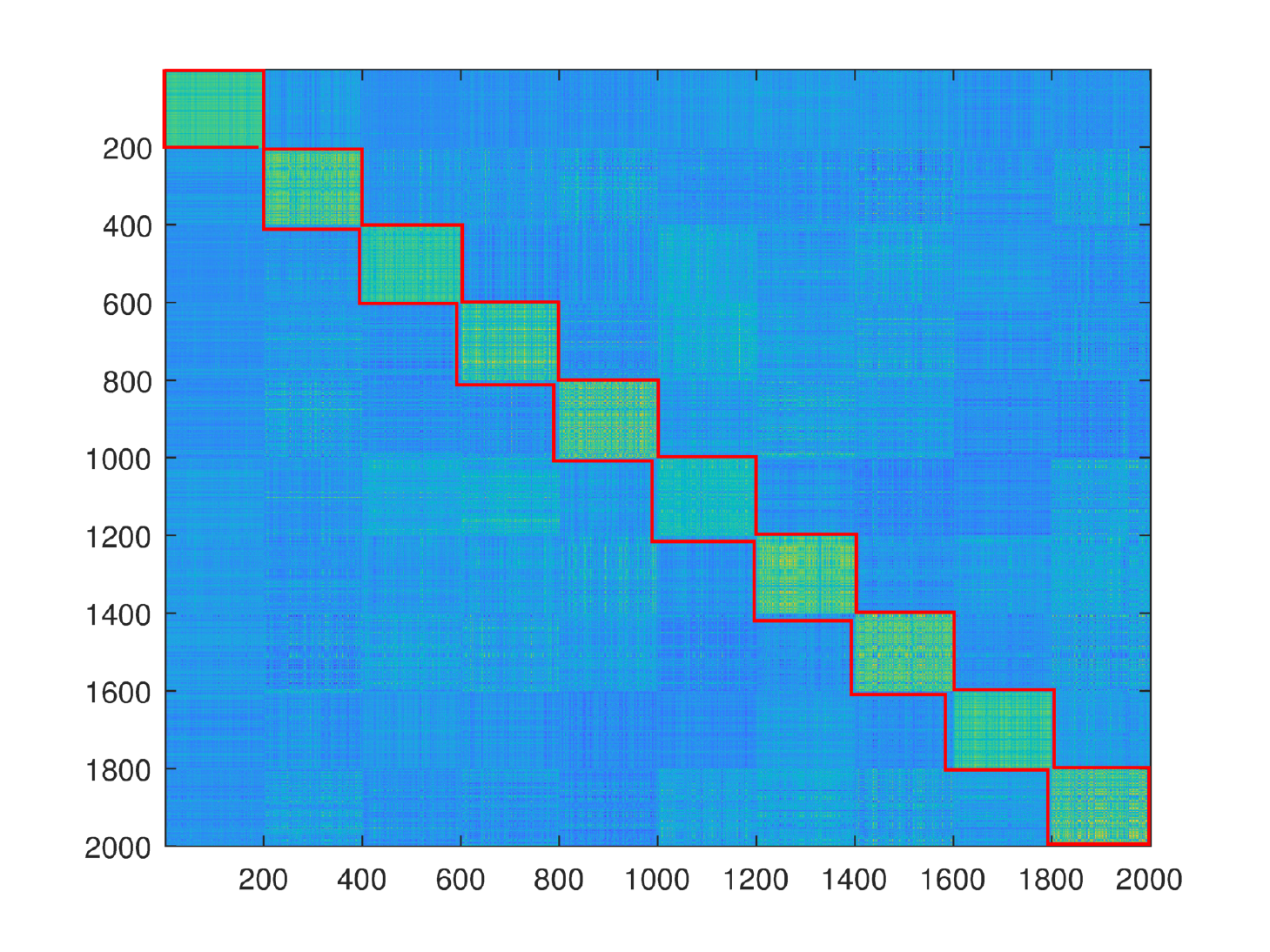}\label{evolution_1_s}}
\subfloat[Mfeat $2^{nd}$ iteration]{\includegraphics[width = 0.25\textwidth]{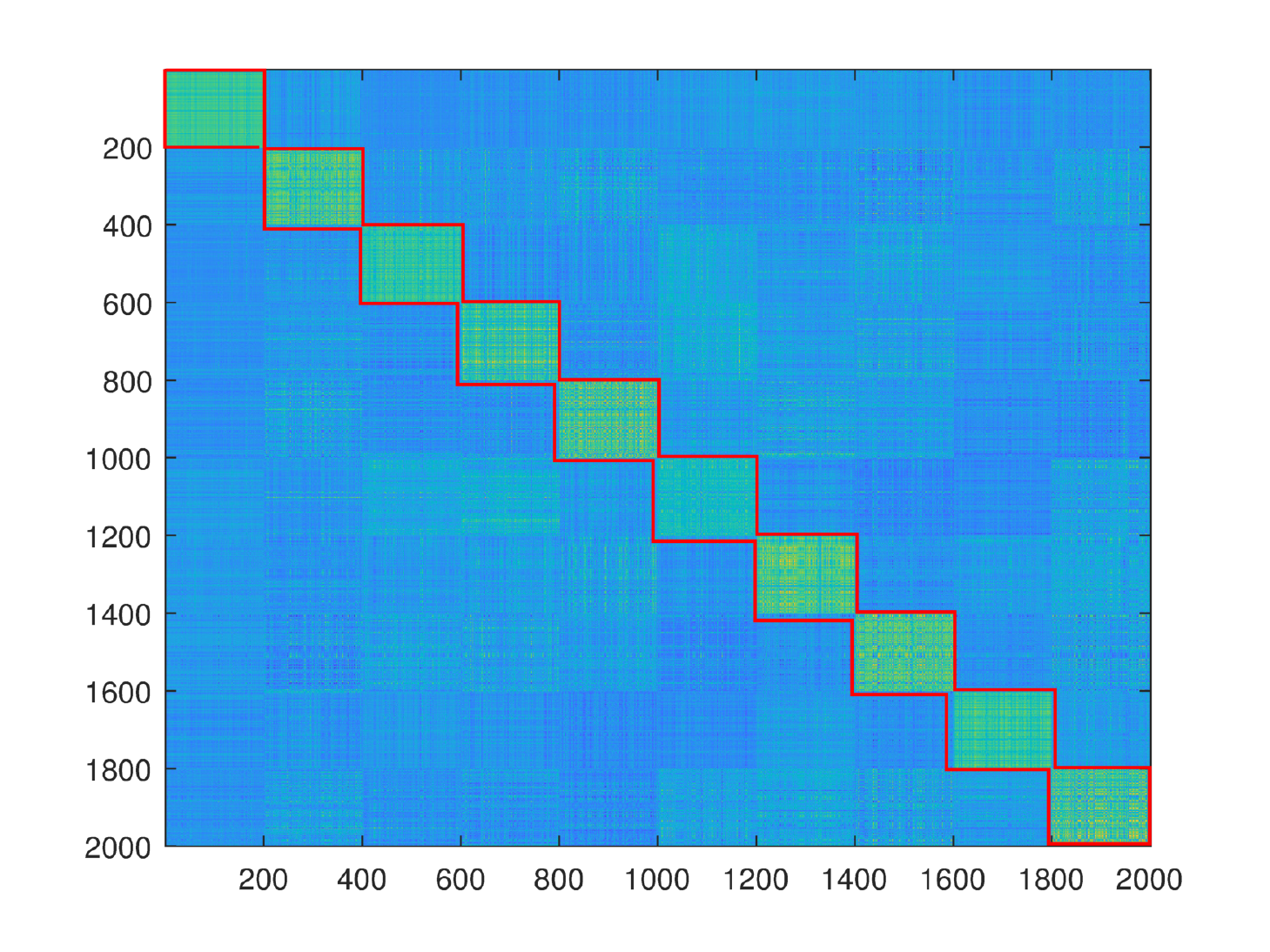}\label{evolution_5_s}}
\subfloat[Mfeat $5^{th}$ iteration]{\includegraphics[width = 0.25\textwidth]{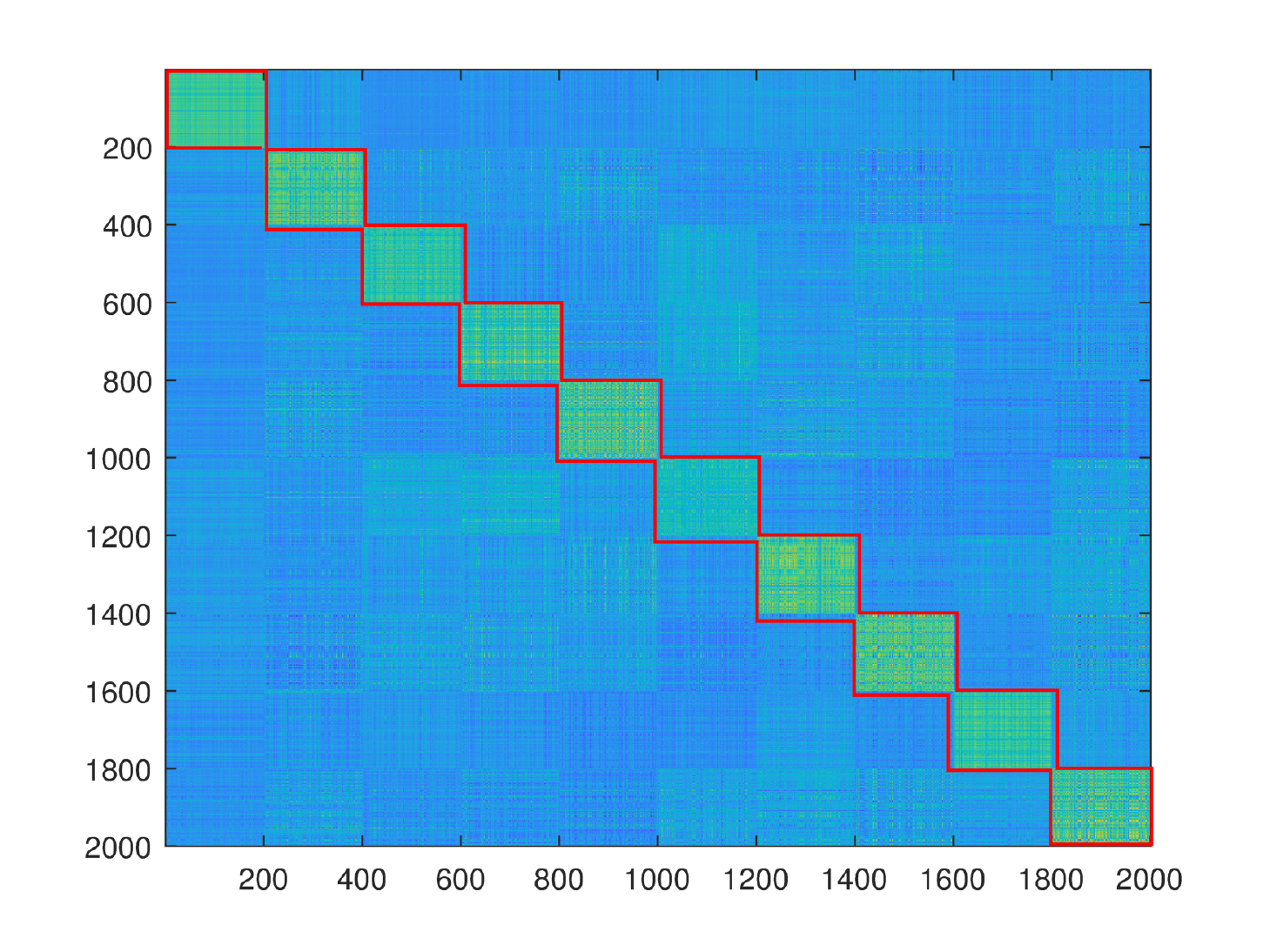}\label{evolution_10_s}}
\subfloat[Mfeat $10^{th}$ iteration]{\includegraphics[width = 0.25\textwidth]{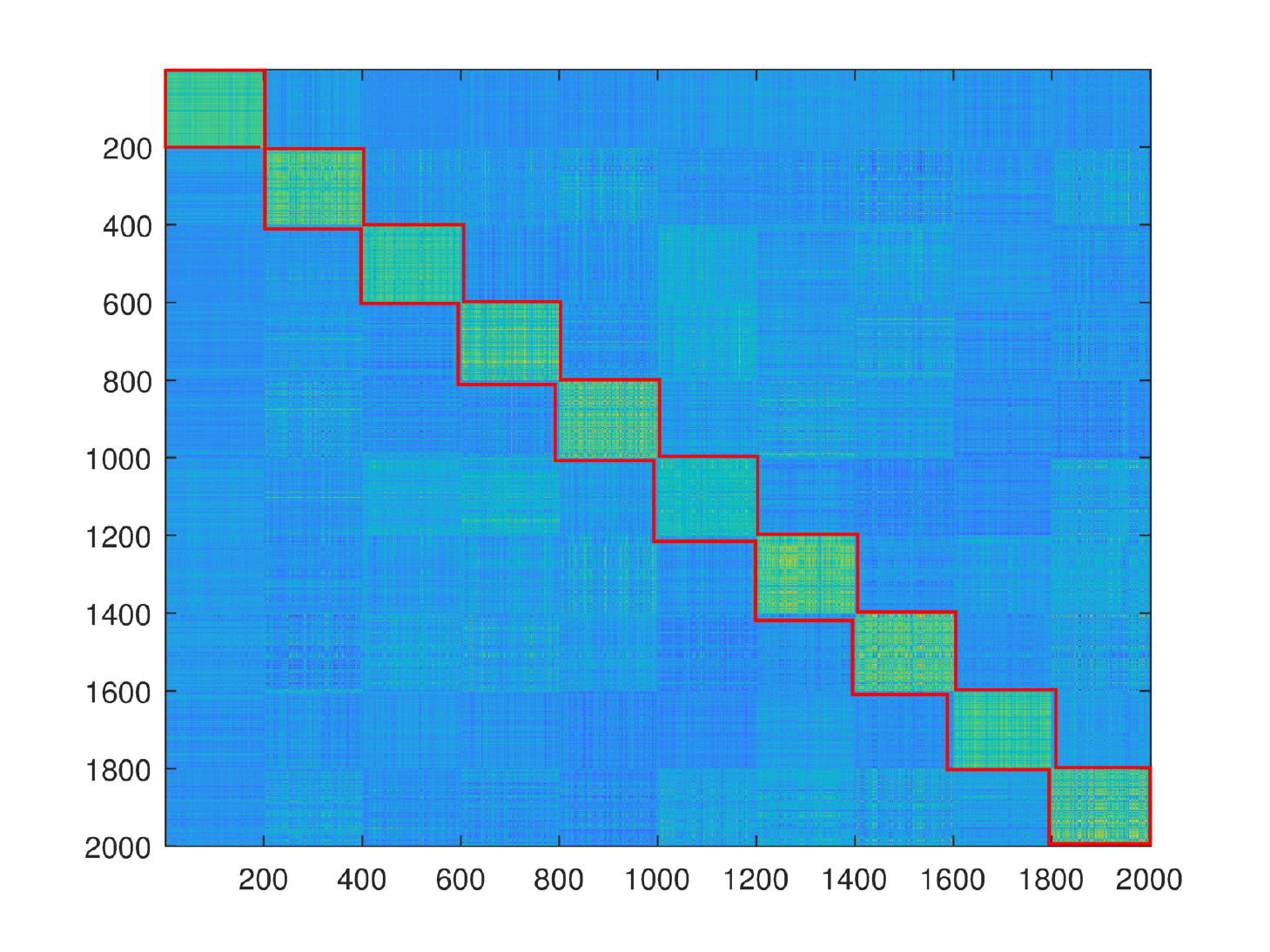}\label{evolution_20_s}}
\caption{The learned affinity matrices  of the $1^{st}$,$2^{nd}$,$5^{th}$ and $10^{th}$ iteration on Mfeat.}\label{evolution_s}
\end{figure*}

\end{document}